\pdfoutput=1

\documentclass[twocolumn]{autart}    

\usepackage{times} 
\usepackage{changes}

\usepackage{tabularx,hhline}   
\usepackage{tablefootnote}
\usepackage{multirow}
\usepackage{booktabs}
\usepackage{makecell}

\usepackage{array}
\newcolumntype{L}[1]{>{\raggedright\let\newline\\\arraybackslash\hspace{0pt}}m{#1}}
\newcolumntype{C}[1]{>{\centering\let\newline\\\arraybackslash\hspace{0pt}}m{#1}}
\newcolumntype{R}[1]{>{\raggedleft\let\newline\\\arraybackslash\hspace{0pt}}m{#1}}

\newcommand{\hdrule}{\midrule[\heavyrulewidth]}

\usepackage[shortlabels]{enumitem}
\setlist[enumerate, 1]{leftmargin=0.7cm, itemsep=0.25em, topsep=-0.65em}

\usepackage{amsmath,amssymb}
\usepackage{cuted} 
\usepackage{mathrsfs} 
\usepackage{mathtools} 
\input{widebar.tex}  
\usepackage{ushort}
\usepackage{empheq}
\newenvironment{tagcases}[1][]{\empheq[left={#1\empheqlbrace}]{alignat=2}}{\endempheq}

\newcommand{\secref}[1]{\S\ref{#1}}
\newcommand{\secsref}[1]{\S\S\ref{#1}}

\newcommand{\Int}[1]{{\kern0pt#1}^{\mathrm{o}}}
\newcommand{\Bdy}[1]{\partial{\kern0pt#1}}
\newcommand{\MatNorm}[1]{{\left\vert\kern-0.25ex\left\vert\kern-0.25ex\left\vert #1 
    \right\vert\kern-0.25ex\right\vert\kern-0.25ex\right\vert}}

\newcommand{\diag}[1]{\operatorname{\rm diag}\{#1\}}
\newcommand{\rank}[1]{\operatorname{\rm rank}\left(#1\right)}
\DeclareMathOperator*{\Argmax}{arg\,max}

\newcommand{\eqnum}  
{\leavevmode\hfill\refstepcounter{equation}\textup{(\theequation)}}

\usepackage[norelsize,algo2e,vlined,ruled]{algorithm2e} 

\usepackage{etoolbox}

\makeatletter
\patchcmd{\@algocf@start}
  {-1.5em}
  {0pt}
  {}{}
\makeatother

\usepackage{algorithmicx}
\SetKwInput{KwInput}{Input}
\SetKwInput{KwOutput}{Output}

\usepackage{psfrag}
\usepackage[dvips]{epsfig}   
\usepackage[caption=false]{subfig}
\usepackage{stackengine}

\usepackage[low-sup]{subdepth} 
\usepackage{setspace}
\usepackage{bigstrut}

\usepackage{fixfoot}
\newcommand{\footref}[1]{\textsuperscript{\ref{#1}}}

\usepackage{multicol}	

\newenvironment{proof}{\begin{pf}}{\qed\end{pf}}
\newenvironment{Proof}{\vspace{-1em}\begin{proof}}{\end{proof}\vspace{-1em}}
\newenvironment{Proof*}[1]{\vspace{-1em}\begin{pf*} {\bf Proof of #1.}}{\qed\end{pf*}\vspace{-1em}}
\newenvironment{Proof**}[2]{\vspace{-1em}\begin{pf*} {\bf (#2) Proof of #1.}}{\qed\end{pf*}\vspace{-1em}}

\theoremstyle{theorem}

\newtheorem{theorem}{Theorem}[section]
\newtheorem{corollary}[theorem]{Corollary}
\newtheorem{lemma}[theorem]{Lemma}
\newtheorem{assumption}[theorem]{Assumption}
\newtheorem{definition}[theorem]{Definition}

\newtheorem{proposition}[theorem]{Proposition}
\newtheorem{claim}[theorem]{Claim}
\newtheorem{remark}[theorem]{Remark}

\newcommand{\MyAssumption}[1]{
\begin{description} [leftmargin=0.4cm]
	\item \textbf{$\!\!\!$Assumption.} \emph{#1}	
\end{description}
}
\newcommand{\MyDefinition}[1]{\textbf{Definition.} \emph{#1}}
\newcommand{\MyRemark}[1]{\textbf{Remark.} \emph{#1}}

\usepackage[authoryear, round]{natbib}  
\usepackage{hyperref}

\makeatletter
\patchcmd{\AtBegShi@Output}{%
  \@PackageWarning{atbegshi}{Ignoring void shipout box}%
}{%
  \begingroup
    \csname set@typeset@protect\endcsname
    \@PackageWarning{atbegshi}{Ignoring void shipout box}%
  \endgroup
}{}{\errmessage{\noexpand\AtBegShi@Output could not be patched}}
\makeatother

\makeatletter\@input{aux.tex}\makeatother

\usepackage[subfigure]{tocloft} 

\newcommand{\nocontentsline}[3]{}
\newcommand{\tocless}[2]{\bgroup\let\addcontentsline=\nocontentsline#1{#2}\egroup}
\newcommand{\toclesslab}[3]{\bgroup\let\addcontentsline=\nocontentsline#1{#2\label{#3}}\egroup}

\newcommand{\Section}[2]{\toclesslab{\section}{#1}{#2}}
\newcommand{\Subsection}[2]{\toclesslab{\subsection}{#1}{#2}}

\makeatletter
\def\@tocline#1#2#3#4#5#6#7{\relax
  \ifnum #1>\c@tocdepth 
  \else
    \par \addpenalty\@secpenalty\addvspace{#2}%
    \begingroup \hyphenpenalty\@M
    \@ifempty{#4}{%
      \@tempdima\csname r@tocindent\number#1\endcsname\relax
    }{%
      \@tempdima#4\relax
    }%
    \parindent\z@ \leftskip#3\relax \advance\leftskip\@tempdima\relax
    \rightskip\@pnumwidth plus4em \parfillskip-\@pnumwidth
    #5\leavevmode\hskip-\@tempdima
      \ifcase #1
       \or\or \hskip 1em \or \hskip 2em \else \hskip 3em \fi%
      #6\nobreak\relax
    \hfill\hbox to\@pnumwidth{\@tocpagenum{#7}}\par
    \nobreak
    \endgroup
  \fi}
\makeatother

\PassOptionsToPackage{linktocpage}{hyperref}  

\begin{document}

\setlength{\abovedisplayskip}{7.5pt}
\setlength{\belowdisplayskip}{7.5pt}
\setlength{\abovedisplayshortskip}{0pt}
\setlength{\belowdisplayshortskip}{0pt}

\begin{frontmatter}

\title{Policy Iterations for Reinforcement Learning Problems in Continuous Time and Space --- Fundamental Theory and Methods\thanksref{footnoteinfo}}

\thanks[footnoteinfo]{The authors gratefully acknowledge the support of Alberta Innovates--Technology Futures, the Alberta Machine Intelligence Institute, DeepMind, the Natural Sciences and Engineering Research Council of Canada, and the Japanese Science and Technology agency (JST) ERATO project JPMJER1603: HASUO Metamathematics for Systems Design.}

\author[UofW]{Jaeyoung Lee\corauthref{JYLee}}\ead{jaeyoung.lee@uwaterloo.ca},     
\author[UofA]{Richard S. Sutton}\ead{rsutton@ualberta.ca}            
\address[UofW]{Department of Electrical and Computer Engineering, University of Waterloo, Waterloo, ON, Canada, N2L 3G1.}
\address[UofA]{Department of Computing Science, University of Alberta, Edmonton, AB, Canada, T6G 2E8.}
\corauth[JYLee]{Corresponding  author. Tel.: +1 587 597 8677.}


\begin{abstract}
Policy iteration (PI) is a recursive process of policy evaluation and improvement for solving an optimal decision-making/control problem, or in other words, a reinforcement learning (RL) problem. PI has also served as the fundamental for developing RL methods. In this paper, we propose two PI methods, called differential PI (DPI) and integral PI (IPI), and their variants, for a general RL framework in continuous time and space (CTS), where the environment is modeled by a system of ordinary differential equations (ODEs). The proposed methods inherit the current ideas of PI in classical RL and  optimal control and theoretically support the existing RL algorithms in CTS: TD-learning and value-gradient-based (VGB) greedy policy update. We also provide case studies including 1) discounted RL and 2) optimal control tasks. Fundamental mathematical properties --- admissibility, uniqueness of the solution to the Bellman equation (BE), monotone improvement, convergence, and optimality of the solution to the Hamilton-Jacobi-Bellman equation (HJBE) --- are all investigated in-depth and improved from the existing theory, along with the general and case studies. 
	Finally, the proposed ones are simulated with an inverted-pendulum model and their model-based and partially model-free implementations to support the theory and further investigate them beyond.
\end{abstract}

\end{frontmatter}


\Section{Introduction}{}

Policy iteration (PI) is a class of approximate dynamic programming (ADP) for recursively solving an optimal decision-making/control problem by alternating between \emph{policy evaluation} to obtain the value function (VF) w.r.t. the current policy (a.k.a. the current control law in control theory) and \emph{policy improvement} to improve the policy by optimizing it using the obtained VF \citep{Sutton2018,Puterman1994,Lewis2009}. PI was first proposed by \citet{Howard1960} in a stochastic environment known as the Markov decision process (MDP) and has served as a fundamental principle for developing RL methods, especially for an environment modeled or approximated by an MDP in discrete time and space. Convergence of such PIs towards the optimal solution has been proven, with the finite-time convergence for a finite MDP \citep[Theorems~6.4.2 and 6.4.6]{Puterman1994}; the forward-in-time computation of PI like the other ADP methods alleviates the problem known as the curse of dimensionality \citep{Powell2007}. A discount factor $\gamma \in [0, 1]$ is normally introduced to both PI and RL to suppress the future reward and thereby have a finite return. \citet{Sutton2018} give a comprehensive overview of PI and RL algorithms with their practical applications and recent success. 

On the other hand, the dynamics of a real physical task is in the majority of cases modeled as a system of (ordinary) differential equations (ODEs) inevitably in continuous time and space (CTS). PI has also been studied in such a continuous domain mainly under the framework of deterministic optimal control, where the optimal solution is characterized by the partial differential Hamilton-Jacobi-Bellman (HJB) equation (HJBE). However, an HJBE is extremely difficult or hopeless to solve analytically, except for a very few exceptional cases. PI methods in this field are often referred to as successive approximations of the HJBE (for recursively solving it!), and the main difference among them lies in their policy evaluation --- the earlier PI methods solve the associated \emph{differential} Bellman equation (BE) (a.k.a. Lyapunov or Hamiltonian equation) to obtain each VF for the target policy (e.g., \citealp*{Leake1967,Kleinman1968,Saridis1979,Beard1997,Abu-Khalaf2005} to name a few). 
\citet*{Murray2002} proposed a trajectory-based policy evaluation that can be viewed as a deterministic Monte-Carlo prediction \citep{Sutton2018}. Motivated by those two approaches above, \citet{Vrabie2009b} proposed a partially model-free\footnote{The term ``partially model-free'' in this paper means that the algorithm can be implemented using some partial knowledge (i.e., the input-coupling terms) of the dynamics~$f$ in \eqref{eq:controlled system}.} PI scheme called integral PI (IPI), which is more relevant to RL in that the associated BE is of a temporal difference (TD) form --- see \citet{Lewis2009} for a comprehensive overview. Fundamental mathematical properties of those PIs, i.e., convergence, admissibility, and monotone improvement of the policies, are investigated in the literature above. As a result, it has been shown that the policies generated by PI methods are always monotonically improved and admissible; the sequence of VFs generated by PI methods in CTS is shown to converge to the optimal solution, quadratically in the LQR case \citep{Kleinman1968}. 
These fundamental properties are discussed, improved, and generalized in this paper in a general setting that includes both RL and optimal control problems in CTS. 

On the other hand, the aforementioned PI methods in CTS were all designed via Lyapunov's stability theory \citep{Khalil2002} to ensure that the generated policies all asymptotically stabilizes the dynamics and yield finite returns (at least on a bounded region around an equilibrium state), provided that so is the initial policy. Here, the dynamics under the initial policy needs to be asymptotically stable to run the PI methods, which is, however, quite contradictory for IPI --- it is partially model-free, but it is hard or even impossible to find such a stabilizing policy \emph{without knowing the dynamics}. Besides, compared with the RL problems in CTS, e.g., those in \citep*{Doya2000,Mehta2009,Fremaux2013}, this stability-based approach restricts the range of the discount factor $\gamma$ and the class of the dynamics and the cost (i.e., reward) as follows.  \vspace{-0.75em}
	\begin{enumerate}
		\item When discounted, the discount factor $\gamma \in (0, 1)$ must be larger than some threshold so as to hold the asymptotic stability of the target optimal policy (\citealp*{Gaitsgory2015,Modares2016}). If not, there is no point in considering stability: PI finally converges to that (possibly) \emph{non-stabilizing} optimal solution, even if the PI is convergent and the initial policy is \emph{stabilizing}. Furthermore, the threshold on $\gamma$ depends on the dynamics (and the cost), and thus it cannot be calculated without knowing the dynamics, a contradiction to the use of any (partially) model-free methods such as IPI. Due to these restrictions on $\gamma$, the PI methods mentioned above for nonlinear optimal control focused on the problems without discount factor, rather than discounted ones.
		\item In the case of optimal regulations, (i) the dynamics is assumed to have at least one equilibrium state;\footnote{For an example of a dynamics with no equilibrium state, see \citep[Example 2.2]{Haddad2008}.} (ii) the goal is to stabilize the system optimally for that equilibrium state, although bifurcation or multiple isolated equilibrium states to be considered may exist; (iii) for such optimal stabilization, the cost is crafted to be positive (semi-)definite --- when the equilibrium state of interest is transformed to zero without loss of generality \citep{Khalil2002}. Similar restrictions exist in optimal tracking problems that can be transformed into equivalent optimal regulation problems (e.g., see \citealp{Modares2014}).
	\end{enumerate}
	In this paper, we consider a general RL framework in CTS, where reasonably minimal assumptions were imposed --- 1) the global existence and uniqueness of the state trajectories, 2) (whenever necessary) continuity, differentiability, and/or existence of maximum(s) of functions, and 3) no assumption on the discount factor $\gamma \in (0, 1]$ --- to include a broad class of problems. The RL  problem in this paper not only contains those in the RL literature (e.g., \citealp{Doya2000,Mehta2009,Fremaux2013}) in CTS but also considers the cases beyond stability framework (at least theoretically), where state trajectories can be still bounded or even diverge (Proposition~\ref{prop:general admissible condition}; \secref{subsection:nonlinear optimal control}; Appendices \secsref{subsection:discounted RL with bounded state trj} and \ref{subsection:LQR} on pages 31--34). It also includes input-constrained and unconstrained problems presented in both RL and optimal control literature as its special cases.

Independent of the research on PI, several RL methods have been proposed in CTS based on RL ideas in the discrete domain. Advantage updating was proposed by \citet{BairdIII1993} and then reformulated by \citet{Doya2000} under the environment represented by a system of ODEs; see also \citet*{Tallec2019}'s recent extension of advantage updating using  deep neural networks. \citet{Doya2000} also extended TD($\lambda$) to the CTS domain and then combined it with his proposed policy improvement methods such as the value-gradient-based (VGB) greedy policy update. See also \citet{Fremaux2013}'s extension of \citet{Doya2000}'s continuous actor-critic with spiking neural networks. \citet{Mehta2009} proposed Q-learning in CTS based on stochastic approximation. Unlike in MDP, however, these RL methods were rarely relevant to the PI methods in CTS due to the gap between optimal control and RL --- the proposed PI methods bridge this gap with a direct connection to TD learning in CTS and VGB greedy policy update \citep{Doya2000,Fremaux2013}. The investigations of the ADP for the other RL methods remain as a future work or see our preliminary result \citep{Lee2017}. 

\Subsection{Main Contributions}{}

In this paper, the main goal is to build up a theory on PI in a general RL framework, from the ideas of PI in classical RL and optimal control, when the time domain and the state-action space are all continuous and a system of ODEs models the environment. As a result, a series of PI methods are proposed that theoretically support the existing RL methods in CTS: TD learning and VGB greedy policy update. Our main contributions are summarized as follows. \vspace{-0.75em}

\begin{enumerate} [leftmargin=*]
	\item Motivated by the PI methods in optimal control, we propose a model-based PI named differential PI (DPI) and a partially model-free PI called IPI, for our general RL framework. The proposed schemes do not necessarily require an initial stabilizing policy to run and can be considered a sort of fundamental PI methods in CTS. 
	    \\[-7.5pt]
	\item By case studies that contain both discounted RL and optimal control frameworks, 
		the proposed PI methods and theory for them are simplified, improved, and specialized, with strong connections to RL and optimal control in CTS. 
		 \\[-7.5pt]
	\item Fundamental mathematical properties regarding PI (and ADP) --- admissibility, uniqueness of the solution to the BE, monotone improvement, convergence, and optimality of the solution to the HJBE --- are all investigated in-depth along with the general and case studies. Optimal control case studies also examine the stability properties of PI. As a result, the existing properties for PI in optimal control are improved and rigorously generalized.
\end{enumerate}

Simulation results for an inverted-pendulum model are also provided, with the model-based and partially model-free implementations to support the theory and further investigate the proposed methods under an admissible (but not necessarily stabilizing) initial policy, with the strong connections to `bang-bang control' and `RL with simple binary reward,' both of which are beyond the scope of our theory. Here, the RL problem in this paper is formulated stability-freely (which is well-defined under the minimal assumptions), so that the (initial) admissible policy is not necessarily stabilizing in the theory and the proposed PI methods for solving it.

\Subsection{Organizations}{}

This paper is organized as follows. In \secref{section:preliminaries}, our general RL problem in CTS is formulated along with mathematical backgrounds, notations, and statements related to BEs, policy improvement, and the HJBE. In \secref{section:PIs}, we present and discuss the two main PI methods (i.e., DPI and IPI) and their variants, with strong connections to the existing RL methods in CTS. We show in \secref{section:fundamental properties of PI} the fundamental properties of the proposed PI methods: admissibility, uniqueness of the solution to the BE, monotone improvement, convergence, and optimality of the solution to the HJBE. Those properties in \secref{section:fundamental properties of PI} and the Assumptions made in \secsref{section:preliminaries} and \ref{section:fundamental properties of PI} are simplified, improved, and relaxed in \secref{section:case studies} with the following case studies: $1$) concave Hamiltonian formulations (\secref{subsection:RL under u-AC setting}); $2$) discounted RL with bounded VF/reward (\secref{subsection:discounted RL with bounded v}); $3$) RL problem with local Lipschitzness (\secref{subsection:RL with local Lipschitzness}); $4$) nonlinear optimal control (\secref{subsection:nonlinear optimal control}). In \secref{section:simulation}, we discuss and provide the simulation results of the main PI methods. Finally, conclusions follow in \secref{section:conclusion}. 

We separately provide Appendices (see page 19 below and thereafter) that contain a summary of notations and terminologies (\secref{appendix:notations and terminologies}), related works and highlights (\secref{section:related work}), details regarding the theory and implementations (\secsref{appendix:sec:replacement of bdy condition}--\ref{appendix:section:optimality}, and \ref{appendix:section:implementation details}), a  pathological example (\secref{appendix:subsection:a pathological ex}), 
	additional case studies (\secref{appendix:section:additional case studies}),
	and all the proofs (\secref{appendix:section:proofs}). Throughout the paper, any section starting with an alphabet as above will indicate a section in the appendices.

\Subsection{Notations and Terminologies}{}
	The following notations and terminologies will be used throughout the paper (see \secref{appendix:notations and terminologies} for a complete list of notations and terminologies, including those not listed below). In any mathematical statement, iff stands for ``if~and~only~if'' and s.t. for ``such that''. ``$\doteq$'' indicates the equality relationship that is true \emph{by definition}.

\textbf{(Sets, vectors, and matrices).}  
$\mathbb{N}$ and $\mathbb{R}$ are the sets of all natural and real numbers, respectively. $\mathbb{R}^{n \times m}$ is the set of all \mbox{$n$-by-$m$} real matrices. \smash{$A^\mathsf{T}$} is the transpose of $A \in \mathbb{R}^{n \times m}$. $\mathbb{R}^n \doteq \mathbb{R}^{n \times 1}$ denotes the $n$-dimensional Euclidean space. $\|x\|$ is the Euclidean norm of $x \in \mathbb{R}^n$, i.e., \smash{$\|x\| \doteq (x^\mathsf{T} x)^{1/2}$}.

\textbf{(Euclidean topology).}  
Let $\Omega \subseteq \mathbb{R}^n$. $\Omega$ is said to be compact iff it is closed and bounded. $\Int{\Omega}$ denotes the interior of $\Omega$; $\Bdy{\Omega}$ is the boundary of $\Omega$. If $\Omega$ is open, then $\Omega \cup \Bdy{\Omega}$ (resp. $\Omega$) is called an $n$-dimensional manifold with (resp. without) boundary. A manifold contains no isolated point.

\textbf{(Functions, sequences, and convergence).} 
A function $f: \Omega \to \mathbb{R}^m$ is said to be $\mathrm{C}^1$, denoted by $f \in \mathrm{C}^1$, iff all of its first-order partial derivatives exist and are continuous over the interior $\Int{\Omega}$; $\nabla f: \Int{\Omega} \to \mathbb{R}^{m \times n}$ denotes the gradient of~$f$. $f(E) \doteq \{ f(x) : x\in E\}$ for $E \subseteq \Omega$ denotes the image of $E$ under $f$. A sequence of functions $\langle f_i \rangle_{i=1}^\infty$, abbreviated by $\langle f_i \rangle$ or $f_i$,
is said to converge locally uniformly iff for each $x \in \Omega$, there is a neighborhood of $x$ on which $\langle f_i \rangle$ converges uniformly. For any two functions $f_1, f_2: \mathbb{R}^n \to [-\infty, \infty)$, we write $f_1 \leqslant f_2$ iff $f_1(x) \leq f_2(x)$ for all $x \in \mathbb{R}^n$.

\Section{Preliminaries}{section:preliminaries}

Let $\mathcal{X} \doteq \mathbb{R}^l$ be a state space and $\mathbb{T} \doteq [0, \infty)$ the underlying time space. An $m$-dimensional manifold $\mathcal{U} \subseteq \mathbb{R}^m$ with or without boundary is called an action space. We also denote $\mathcal{X}^\mathsf{T} \doteq \mathbb{R}^{1 \times l}$ for notational convenience. 
The environment in this paper is described in CTS by a system of ODEs:
\begin{align}
    \dot X_{t} = f(X_{t}, U_{t}), \;\;\, U_t \in \mathcal{U}
    \label{eq:controlled system}
\end{align}
where 
$t \in \mathbb{T}$ is time instant, 
$\mathcal{U} \subseteq \mathbb{R}^m$ is an action space, and 
the dynamics $f: \mathcal{X} \times \mathcal{U} \to \mathcal{X}$ is a continuous function; $X_t, {\dot X}_t \in \mathcal{X}$ denote the state vector and its time derivative, at time~$t$, respectively; the action trajectory $t \mapsto U_t$ is a continuous function from $\mathbb{T}$ to $\mathcal{U}$. We assume that $t = 0$ is the initial time without loss of generality\footnote{If the initial time $t_0$ is non-zero, then proceed with the time variable $t' = t - t_0$, which satisfies $t' = 0$ at the initial time $t = t_0$.} and that

\MyAssumption{The state trajectory~$t \mapsto X_t$ satisfying \eqref{eq:controlled system} is uniquely defined over the entire time interval $\mathbb{T}$.\footnote{Not imposed on our problem for generality but strongly related to this Assumption is Lipschitz continuity of $f$ and $f_\pi$. See \secref{subsection:RL with local Lipschitzness} for related discussions; for more study, see \citep[Section~3.1]{Khalil2002} with the system $\dot x = f'(t, x)$, where $f'(t, x) \doteq f(x, U_t)$ or $f_\pi(x)$.}}

A policy $\pi$ refers to a continuous function $\pi: \mathcal{X} \to \mathcal{U}$ that determines the state trajectory $t \mapsto X_t$ by $U_t = \pi(X_t)$ for all $t \in \mathbb{T}$. For notational efficiency, we employ the $\mathbb{G}$-notation $\mathbb{G}_\pi^{x} [ \,Y ]$, which means the value~$Y$ when $X_0 = x$ and $U_{t} = \pi(X_{t})$ for all $t \in \mathbb{T}$. Here, $\mathbb{G}$ stands for ``Generator,'' and $\mathbb{G}_\pi^x$ can be thought of as the corresponding notation of the expectation $\mathbb{E}_\pi [ \, \cdot \, | S_0 = x ]$ in the RL literature \citep{Sutton2018}, without playing any stochastic role. Note that the limits and integrals are exchangeable with $\mathbb{G}^x_\pi [ \,\cdot\, ]$ in order (whenever those limits and integrals are defined for any $X_0 \in \mathcal{X}$ and any action trajectory $t \mapsto U_t$). For example, for any continuous function~$v: \mathcal{X} \to \mathbb{R}$, 
\begin{align*}
	\mathbb{G}_\pi^x \bigg [ \int v(X_t) \, dt \bigg ] = 
	\int \mathbb{G}_\pi^x \big [ v(X_t) \big ] \, dt = \int v \big ( \mathbb{G}_\pi^x[X_t] \big ) \, dt,
\end{align*}
where the three mean the same: $\int v(X_t) \, dt$ when $X_0 = x$ and $U_t = \pi(X_t)$ $\forall t \in \mathbb{T}$. Also note: $\mathbb{G}_\pi^x [U_t ] = \mathbb{G}_\pi^x [ \pi(X_t)]$.

Finally, the time-derivative ${\dot v}: \mathcal{X} \times \mathcal{U} \to \mathbb{R}$ of a $\mathrm{C}^1$ function $v: \mathcal{X} \to \mathbb{R}$ is given by ${\dot v}(X_t, U_t) = \nabla v(X_t) f(X_t,U_t)$ --- applying the chain rule and \eqref{eq:controlled system}. Here, $X_t \in \mathcal{X}$ and $U_t \in \mathcal{U}$ are free variables, and ${\dot v}$ is continuous since so are $f$ and $\nabla v$.

\Subsection{RL problem in Continuous Time and Space}{subsection:RL Problem}

The RL problem considered in this paper is to find the best policy $\pi_*$ that maximizes the infinite horizon value function (VF) $v_\pi: \mathcal{X} \to [-\infty, \infty)$ defined as 
\begin{equation}
	v_\pi(x) \doteq \mathbb{G}_\pi^x \bigg [ \int_{0}^{\infty} \gamma^{t} \! \cdot \! R_{t} \, dt \bigg ],
	\label{eq:value function}
\end{equation}
where the reward $R_t$ is determined by a continuous reward function $r : \mathcal{X} \times \mathcal{U} \to \mathbb{R}$ as $R_t = r(X_t, U_t)$; $\gamma \in (0, 1]$ is the discount factor. Throughout the paper, the attenuation rate   
\begin{align*}
	\alpha \doteq - \ln \gamma \geq 0
\end{align*}
will be used interchangeably for simplicity. For a policy~$\pi$, we denote $f_\pi(x) \doteq f(x, \pi(x))$ and $r_\pi(x) \doteq r(x, \pi(x))$; both are continuous as so are $f$, $r$, and $\pi$ by definitions. 

\MyAssumption{A maximum of the reward function~$r$: 
\[
	r_\mathsf{max} \doteq \max \big \{ r(x, u): (x, u) \in \mathcal{X} \times \mathcal{U} \big \}
\] 
exists and for $\gamma = 1$, $r_\mathsf{max} = 0$.\footnote{If $r_\mathsf{max} \neq 0$ and $\gamma = 1$, then proceed with the reward function 
$
	r'(x, u) \doteq r(x, u) - {r}_\mathsf{max}
$
whose maximum is now zero.}}

Note that the integrand $t \mapsto \gamma^t R_t$ is continuous since so are $t \mapsto X_t$, $t \mapsto U_t$, and $r$. So, by the above assumption on $r$, the time integral and thus the VF $v_\pi$ in \eqref{eq:value function} are well-defined in the Lebesque sense \citep[Chapter~2.3]{Folland1999} and, as stated below, uniformly upper-bounded.

\begin{lemma}
	\label{lemma:VF range}
	There exists a constant ${\widebar v} \in \mathbb{R}$ s.t. $v_\pi \leqslant \widebar v$ for any policy~$\pi$;  
	$\widebar v = 0$ for $\gamma = 1$ and otherwise,
	$\widebar v = r_\mathsf{max} / \alpha$.
\end{lemma}

By Lemma~\ref{lemma:VF range}, the VF is always less than some constant, but it is still possible that $v_\pi(x) = - \infty$ for some $x \in \mathcal{X}$. In this paper, the finite VFs are characterized by the notion of admissibility given below. 

\MyDefinition{
	\label{def:admissibility}
	A policy $\pi$ (or its VF $v_\pi$) is said to be admissible, denoted by $\pi \in \Pi_\mathsf{a}$ (or $v_\pi \in \mathcal{V}_\mathsf{a}$), 
	iff $v_{\pi}(x)$ is finite for all $x \in \mathcal{X}$.
	Here, $\Pi_\mathsf{a}$ and $\mathcal{V}_\mathsf{a}$ denote the sets of all admissible policies and admissible VFs, respectively.
}

To make our RL problem feasible, we assume:

\MyAssumption{There exists at least one admissible policy, and 
\vspace{-0.7em}
\begin{equation}
	\text{every admissible VF is $\mathrm{C}^1$.}
	\label{eq:assumption:every admissible VF is C1}
	\vspace{-0.7em}
\end{equation}
}

The following proposition gives a criterion for admissibility and boundedness.

\begin{proposition}
		\label{prop:general admissible condition}
			A policy $\pi$ is admissible if there exist a function $\xi: \mathcal{X} \to \mathbb{R}$ and a constant $\ushort \alpha < \alpha$, both possibly depending on the policy~$\pi$, such that 
			\begin{align}
					\forall x \in \mathcal{X}\!: \; 
					\mathbb{G}_\pi^x [R_t] \geq \xi(x) \cdot \exp(\ushort \alpha \hspace{0.1em} t)
					 \; \textrm{ for all } t \in \mathbb{T}.
					\label{eq:exponential increasing condition for admissibility}
			\end{align}
			Moreover, $v_\pi$ is bounded if so is $\xi$.
\end{proposition}

\MyRemark{The criterion~\eqref{eq:exponential increasing condition for admissibility} means that the reward $R_t$ under~$\pi$ does not diverge to $-\infty$ exponentially with the rate $\alpha$ or higher. For $\gamma = 1$ (i.e., $\alpha = 0$), it means exponential convergence $R_t \to 0$. The condition~\eqref{eq:exponential increasing condition for admissibility} is fairly general and so satisfied by the examples in \secsref{subsection:discounted RL with bounded v}, \ref{subsection:discounted RL with bounded state trj}, and \ref{subsection:LQR}.}

\Subsection{Bellman Equations with Boundary Condition}{subsection:BE with boundary condition}

Define the Hamiltonian function $h: \mathcal{X} \times \mathcal{U} \times \mathcal{X}^\mathsf{T} \to \mathbb{R}$ as
\begin{equation}
	h(x,u,p) \doteq r(x,u) + p \, f(x,u)
	\label{eq:def:Hamiltonian}	
\end{equation}
(which is continuous as so are $f$ and $r$) and the $\gamma$-discounted cumulative reward $\mathfrak{R}_\eta$ up to a given time horizon $\eta > 0$ as
\[
	\mathfrak{R}_{\eta} \doteq \int_{0}^{\eta} \gamma^{t} \! \cdot \! R_{t} \, dt
\]
as a short-hand notation. The following lemma then shows the equivalence of the Bellman-like (in)equalities.

\begin{lemma}
	\label{lemma:Bellman ineq}
	Let $\sim$ be a binary relation on $\mathbb{R}$ that belongs to $\{=, \leq, \geq\}$ and $v:\mathcal{X} \to \mathbb{R}$ be $\mathrm{C}^1$. Then, for any policy $\pi$, 
	\begin{equation}
		v(x) \sim \mathbb{G}_\pi^x \big [ \, \mathfrak{R}_{\eta}
		+ \gamma^{\eta} \!\cdot \! v(X_{\eta}) \, \big ]
		\label{eq:conversion lemma:integral ineq}
	\end{equation}
	holds for all $x \in \mathcal{X}$ and all horizon $\eta > 0$ iff 
	\begin{equation}
		\alpha \cdot v (x) \sim 
		h(x,\pi(x), \nabla v(x)) \qquad \forall x \in \mathcal{X}.
		\label{eq:conversion lemma:differential ineq}
	\end{equation}
\end{lemma}

By splitting the time-integral in \eqref{eq:value function} at $\eta> 0$, we can easily see that the VF $v_\pi$ satisfies the Bellman equation (BE):
\begin{equation}
		v_\pi(x) = \mathbb{G}_\pi^x \big [ \mathfrak{R}_{\eta} + \gamma^{\eta} \!\cdot\! v_\pi(X_{\eta}) \big ]
	\label{eq:Bellman eq}
\end{equation}
that holds for any $x \in \mathcal{X}$ and any $\eta > 0$. 
Assuming $v_\pi \in \mathcal{V}_\mathsf{a}$ and using \eqref{eq:Bellman eq}, we obtain its boundary condition at $\eta = \infty$.
\begin{proposition}
	\label{prop:VF satisfies the boundary condition}
	Suppose that $\pi$ is admissible. Then, 
	\[
		\smash{\lim_{t \to \infty}} \mathbb{G}_\pi^x \big  [ \gamma^{t} \cdot v_\pi(X_{t}) \big ] = 0 \qquad \forall x \in \mathcal{X}.		
		\vspace{-0.75em}
	\]
\end{proposition}

By the application of Lemma~\ref{lemma:Bellman ineq} to the BE~\eqref{eq:Bellman eq} under \eqref{eq:assumption:every admissible VF is C1}, the following \emph{differential BE} holds whenever $\pi \in \Pi_\mathsf{a}$:
\begin{equation}
		\alpha \cdot v_\pi (x) 
			= h(x,\pi(x), \nabla v_\pi(x)), 
	      \label{eq:differential BE}
\end{equation}
where the function $x \mapsto h(x,\pi(x), \nabla v_\pi(x))$ is continuous since so are the associated functions $h$, $\pi$, and $\nabla v_\pi$. Whenever necessary, we call \eqref{eq:Bellman eq} the \emph{integral BE} to distinguish it from the \emph{differential BE}~\eqref{eq:differential BE}.

In what follows, we state that the boundary condition \eqref{eq:uniqueness condition for v}, \emph{the counterpart of that in Proposition~\ref{prop:VF satisfies the boundary condition},} is actually necessary and sufficient for a solution $v$ of the BE~\eqref{eq:Bellman eq for v} or \eqref{eq:differential BE for v} to be equal to the corresponding VF $v_{\pi}$ and ensure $\pi \in \Pi_\mathsf{a}$.

\begin{theorem} [Policy Evaluation]
	\label{thm:uniqueness condition}
	$\;$\\
	Fix the horizon $\eta > 0$ and suppose there exists a function $v: \mathcal{X} \to \mathbb{R}$ s.t. 
	either of the followings holds for a policy~$\pi$: 
	\begin{enumerate}
		\item	 $v$ satisfies the integral BE:
			\begin{equation}
				v(x) = \mathbb{G}_\pi^x \big [ \mathfrak{R}_{\eta}
				+ \gamma^{\eta} \! \cdot\! v(X_{\eta}) \big ] \qquad \forall x \in \mathcal{X};
				\label{eq:Bellman eq for v}
			\end{equation}	
		\item $v$ is $\mathrm{C}^1$ and satisfies the differential BE:
			\begin{equation}
				\alpha \cdot v (x) = 
					h(x, \pi(x), \nabla v(x)) \qquad \forall x \in \mathcal{X}.
			   \label{eq:differential BE for v}				
			\end{equation}
	\end{enumerate}
	Then, $\pi$ is admissible and $v = v_\pi$ iff 
	\begin{equation}
		\lim_{k \to \infty} \mathbb{G}_\pi^x \big  [ \, \gamma^{k \cdot \eta} \cdot v(X_{k \cdot \eta}) \, \big ] = 0 \qquad \forall x \in \mathcal{X}.
		\label{eq:uniqueness condition for v}		
		\vspace{-1.0em}		
	\end{equation}
\end{theorem}

For sufficiency, the boundary condition~\eqref{eq:uniqueness condition for v} can be replaced by the conditions on $v$ and $r_\pi$ (and $R_t$)   in Theorem~\ref{thm:boundary condition:sufficiency} in \secref{appendix:sec:replacement of bdy condition}. These conditions are particularly related to the optimal control framework in \secref{subsection:nonlinear optimal control} but applicable to any case in this paper as an alternative to \eqref{eq:uniqueness condition for v} (see \secref{appendix:sec:replacement of bdy condition} for more).

\Subsection{Policy Improvement}{subsection:policy improvement}

Define a partial order among policies: $\pi \preccurlyeq \pi'$ iff $v_\pi \leqslant v_{\pi'}$. Then, we say that a policy $\pi'$ is improved over $\pi$ iff $\pi \preccurlyeq \pi'$. In CTS, the Bellman inequality in Lemma~\ref{lemma:policy improvement lemma} for $v = v_\pi$ ensures this policy improvement over an admissible policy~$\pi$. The inequality becomes the BE~\eqref{eq:differential BE} when $v = v_\pi$ and $\pi = \pi'$.

\begin{lemma}
	\label{lemma:policy improvement lemma}
	If $v \in \mathrm{C}^1$ is upper-bounded (by zero if $\gamma = 1$) and satisfies for a policy $\pi'$
	\[
		\alpha \cdot v(x) \leq h (x, \pi'(x), \nabla v(x)) \qquad \forall x \in \mathcal{X},
		\label{eq:policy improvement inequality}
	\]
	then $\pi'$ is admissible and $v \leqslant v_{\pi'}$.
\end{lemma}

In what follows, for the existence of a maximally improving policy, we assume on the Hamiltonian function~$h$:

\MyAssumption{
There exists a function $u_*: \mathcal{X} \times \mathcal{X}^\mathsf{T} \to \mathcal{U}$ such that $u_*$ is continuous and 
		\begin{equation}
				\hspace{-0.5em}
				u_*(x, p) \in \smash{\Argmax_{u \in \mathcal{U}} h(x, u, p) \;\;\, \forall (x, p) \in \mathcal{X} \times \mathcal{X}^\mathsf{T}\!.}
			\label{eq:assumption:general condition for policy improvement}
		\end{equation}
}
Here, \eqref{eq:assumption:general condition for policy improvement} simply means that for each $(x, p)$, the function $u \mapsto h(x, u, p)$ has its maximum at ${u}_*(x, p) \in \mathcal{U}$. Then, for any admissible policy~$\pi$, there exists a continuous function $\pi':\mathcal{X} \to \mathcal{U}$ such that 
	\begin{equation}
		\label{eq:assumption:policy improvement}
		\pi'(x) \in \Argmax_{u \in \mathcal{U}} h(x, u, \nabla v_\pi(x)) \;\;\, \forall x \in \mathcal{X}.
	\end{equation}	
We call such $\pi'$ a \emph{maximal policy} (over $\pi \in \Pi_\mathsf{a}$). Given $u_*$, a maximal policy $\pi'$ can be directly obtained by
\begin{align}
	\pi'(x) = u_* (x, \nabla v_\pi (x) ). 	
	\label{eq:closed-form expression of the maximal policy under uniqueness}
\end{align}
In general, there may exist multiple maximal policies, but if $u_*$ in \eqref{eq:assumption:general condition for policy improvement} is \emph{unique}, then $\pi'$ satisfying \eqref{eq:assumption:policy improvement} is \emph{uniquely} given by \eqref{eq:closed-form expression of the maximal policy under uniqueness}. For non-affine optimal control problems, \citet{Leake1967} and \citet*{Bian2014} imposed assumptions similar to the above Assumption on $u_*$ plus its uniqueness. Here, the existence of $u_*$ is ensured if $\mathcal{U}$ is compact; $u_*$ is unique if the function $u \mapsto h(x, u, p)$ is strictly concave and $\mathrm{C}^1$ for each $(x, p)$ --- (i) see \secref{appendix:section:maximal function} for details and more; (ii) for such examples, see \secref{subsection:RL under u-AC setting}; Cases 1 and 2 in \secref{section:simulation}.

\begin{theorem}
	[Policy Improvement] 
	\label{thm:policy improvement thm}
	$\;$\\
	Suppose $\pi$ is admissible. Then, the policy $\pi'$ given by~\eqref{eq:assumption:policy improvement} is also admissible and satisfies $\pi \preccurlyeq \pi'$. 
\end{theorem}

\Subsection{Hamilton-Jacobi-Bellman Equation (HJBE)}{subsection:HJBE}

Under the Assumptions made so far, the optimal solution of the RL problem can be characterized via (i) the HJBE~\eqref{eq:HJBE}:
 \begin{equation}
		\alpha \cdot v_*(x) 
			= \smash{\max_{u \in \mathcal{U}}} \, h(x, u, \nabla v_*(x)) \qquad \forall x \in \mathcal{X}
	\label{eq:HJBE}
\end{equation}
and (ii) the associated policy $\pi_*: \mathcal{X} \to \mathcal{U}$ such that 
\begin{equation}
	\pi_*(x) \in \smash{\Argmax_{u \in \mathcal{U}}} \, h(x, u, \nabla v_*(x))
	\qquad
	\forall x \in \mathcal{X},
	\label{eq:optimal policy}
\end{equation}
both of which are the keys to prove the convergence of PIs towards the optimal solution $v_*$ (and $\pi_*$) in \secref{section:fundamental properties of PI}. Note that once a $\mathrm{C}^1$ solution $v_*: \mathcal{X} \to \mathbb{R}$ to the HJBE~\eqref{eq:HJBE} exists, then so does a continuous function (i.e., a policy) $\pi_*$ satisfying \eqref{eq:optimal policy}, by the Assumption of the existence of a continuous function~$u_*$ satisfying~\eqref{eq:assumption:general condition for policy improvement}, and is given by
\begin{equation}
	\pi_*(x) = u_* (x, \nabla v_* (x) ). 	
	\label{eq:closed-form expression of optimal policy}	
\end{equation}
In what follows, we show that satisfying the HJBE \eqref{eq:HJBE} and \eqref{eq:optimal policy} is necessary for $(v_*, \pi_*)$ to be optimal over the entire admissible space.

\begin{theorem}
	\label{thm:nc for optimality}
	If there exists an optimal policy~$\pi_*$ whose VF $v_*$ satisfies $v \leqslant v_*$ for any $v \in \mathcal{V}_\mathsf{a}$,
	then $v_*$ and $\pi_*$ satisfy the HJBE~\eqref{eq:HJBE} and \eqref{eq:optimal policy}, respectively.
\end{theorem}

There may exist another optimal policy $\pi_{*}'$ than $\pi_*$, but their VFs are always the same by $\pi_* \preccurlyeq \pi_*'$  and $\pi_*' \preccurlyeq \pi_*$ and equal to a solution $v_*$ to the HJBE~\eqref{eq:HJBE} by Theorem~\ref{thm:nc for optimality}. In this paper, if exist, $\pi_*$ denotes any one of the optimal policies, and $v_*$ is the unique common VF for them which we call the optimal VF. In general, they denote a solution $v_*$ to the HJBE~\eqref{eq:HJBE} and an associated HJB policy $\pi_*$ s.t. \eqref{eq:optimal policy} holds (or an associated function~$\pi_*$ satisfying \eqref{eq:optimal policy} that is potentially discontinuous --- see \secsref{subsection:convergence} and \ref{subsection:sc for optimality}). 

\begin{remark}
\label{remark:design of reward}
The reward function~$r$ has to be appropriately designed in such a way that the function $u \mapsto h(x, u, p)$ for each $(x, p)$ at least has a maximum (so that \eqref{eq:assumption:general condition for policy improvement} holds for some $u_*$). Otherwise, the maximal policy~$\pi'$ in \eqref{eq:assumption:policy improvement} and/or the solution~$v_*$ to the HJBE~\eqref{eq:HJBE} (and accordingly, $\pi_*$ in \eqref{eq:optimal policy}) may not exist since neither do the maxima in those equations. Such a pathological example is given in \secref{appendix:subsection:a pathological ex} for a simple non-affine dynamics~$f$. In \secref{subsubsection:case II}, we revisit this issue and propose a technique applicable to a class of non-affine RL problems to ensure the existence and continuity of $u_*$.		
\end{remark}

We note that the optimality of the HJB solution $(v_*, \pi_*)$ is studied more in \secref{appendix:section:optimality}, e.g., the sufficient conditions and case studies, in connection with the PIs presented in \secref{section:PIs} below.


\Section{Policy Iterations}{section:PIs}

Now, we are ready to state our two main PI schemes, DPI and IPI. Here, the former is a model-based approach, and the latter is a partially model-free PI. Their simplified (partially model-free) versions discretized in time will be also discussed after that. Until \secref{section:simulation}, we present and discuss those PI schemes in an ideal sense without introducing (i) any function approximator, such as neural network, and (ii) any discretization in the state  space.\footnote{When we implement any of the PI schemes, both are obviously required (except linear quadratic regulation (LQR) cases) since the structure of the VF is veiled and it is impossible to perform the policy evaluation and improvement for an (uncountably) infinite number of points in the continuous state space $\mathcal{X}$ (see also \secref{section:simulation} for implementation examples, with \secref{appendix:section:implementation details} for details).}

\Subsection{Differential Policy Iteration (DPI)}{subsection:DPI}

Our first PI, named differential policy iteration (DPI), is a model-based PI scheme extended from optimal control to our RL framework (e.g., see \citealp{Leake1967,Beard1997,Abu-Khalaf2005}). Algorithm~\ref{algorithm:DPI} describes the whole procedure of DPI --- it starts with an initial admissible policy~$\pi_0$ (line~1) and performs policy evaluation and improvement until $v_{i}$ and/or $\pi_i$ converges (lines~2--5). In \emph{policy evaluation} (line~3), the agent solves the differential BE~\eqref{eq:differential BE in DPI} to obtain the VF \smash{$v_i = v_{\pi_{i-1}}$} for the last policy $\pi_{i-1}$. Then, $v_i$ is used in \emph{policy improvement} (line~4) so as to obtain the next policy $\pi_i$ by maximizing the associated Hamiltonian function in~\eqref{eq:policy improvement in DPI}. Here, if $v_i =  v_*$, then $\pi_{i} = \pi_*$ by \eqref{eq:optimal policy} and \eqref{eq:policy improvement in DPI}. \pagebreak

\setlength{\algomargin}{0em}

 \begin{algorithm2e}[h!]
 \begin{spacing} {1.1}
 \DontPrintSemicolon
 \BlankLine
 \vspace{5pt}
 \nl \textbf{Initialize:} 
 	\smash{$\begin{cases} 
 			\textrm{$\pi_0$, an initial admissible policy;} \\
 			\textrm{$i \leftarrow 1$, iteration index;}
  	  \end{cases}$}\;
 \vspace{7.5pt}
 \nl \Repeat{convergence is met.}
 {
 \nl \textbf{Policy Evaluation:} given $\pi_{i-1}$, find a $\mathrm{C}^1$ function $v_i: \mathcal{X} \to  \mathbb{R}$ satisfying the differential BE:
 \begin{align}
 	\notag \\[-22.5pt]
	\alpha \cdot v_i (x) = h(x, \pi_{i-1}(x), \nabla v_i(x)) \;\,\ \forall x \in \mathcal{X}; 
    \label{eq:differential BE in DPI}
   \\[-22.5pt] \notag 
 \end{align}
  $\;$\;$\;$\\[-20pt]
 \nl\textbf{Policy Improvement:} find a policy $\pi_{i}$ such that
 \begin{align}
 	\notag \\[-20pt]
 	\pi_{i}(x) \in \Argmax_{u \in \mathcal{U}} h(x, u, \nabla v_i(x)) \;\, \forall x \in \mathcal{X};
 	\label{eq:policy improvement in DPI}
 \end{align}
 $\;$\;$\;$\\[-35pt] 
 \nl$i \leftarrow i+1$;
 }
 \end{spacing}
 \caption{Differential Policy Iteration (DPI)}
 \label{algorithm:DPI}
\end{algorithm2e}

Basically, DPI is model-based (see the definition~\eqref{eq:def:Hamiltonian} of $h$) and does not rely on any state trajectory data. On the other hand, its policy evaluation is closely related to TD learning methods in CTS \citep{Doya2000,Fremaux2013}. To see this, note that \eqref{eq:differential BE in DPI} can be expressed w.r.t. $(X_t, U_t)$ as $\mathbb{G}_{\pi_{i-1}}^x [\delta_t(v_i) ] = 0$ for all $x \in \mathcal{X}$ and $t \in \mathbb{T}$, where $\delta_t$ denotes the TD error defined as
\[
	\delta_t(v) \doteq R_t + {\dot v}(X_t, U_t) - \alpha \cdot v(X_t)	
\]
	for any $\mathrm{C}^1$ function~$v: \mathcal{X} \to \mathbb{R}$. \citet{Fremaux2013} used $\delta_t(v)$ as the TD error in their model-free actor-critic and approximated $v$ and the model-dependent part ${\dot v}$ of $\delta_t(v)$ by a spiking neural network. $\delta_t(v)$ is also the TD error in TD(0) in CTS \citep{Doya2000}, where ${\dot v}(X_t, U_t)$ is approximated by 
$
	(v(X_t) - v(X_{t - \Delta t})) / \Delta t
$
in backward time, for a sufficiently small time step $\Delta t$ chosen in the interval $(0, \alpha^{-1})$; under this backward-in-time approximation, $\delta_t(v)$ can be expressed in a similar form to the TD error in discrete-time as
	\begin{equation}
		\delta_t(v) \approx R_t + {\hat \gamma}_\mathsf{d} \cdot V(X_t) - V(X_{t - \Delta t})
		\label{eq:Doya's TD error}
	\end{equation}
for $V \doteq v / \Delta t$ and ${\hat \gamma}_\mathsf{d} \doteq 1 - \alpha \Delta t \approx e^{- \alpha \Delta t}$ ($= \gamma^{\Delta t}$). Here, the discount factor ${\hat \gamma}_\mathsf{d}$ belongs to $(0, 1)$ if so is $\gamma$, thanks to $\Delta t \in (0, \alpha^{-1})$, and ${\hat \gamma}_\mathsf{d} = 1$ whenever $\gamma = 1$. In summary, policy evaluation of DPI solves the differential BE~\eqref{eq:differential BE in DPI} that idealizes the existing TD learning methods in CTS \citep{Doya2000,Fremaux2013}. 

\Subsection{Integral Policy Iteration (IPI)}{subsection:IPI}
Algorithm~\ref{algorithm:IPI} describes the second PI, integral policy iteration (IPI), whose difference from DPI is that \eqref{eq:differential BE in DPI} and \eqref{eq:policy improvement in DPI} for the policy evaluation and improvement are replaced by \eqref{eq:integral BE in IPI} and \eqref{eq:policy improvement in IPI}, respectively. The other steps are the same as DPI, except that the time horizon $\eta > 0$ is initialized (line~1) before the main loop.

	 In \emph{policy evaluation} (line~3), IPI solves the integral BE~\eqref{eq:integral BE in IPI} for a given fixed horizon~$\eta > 0$ without using the explicit knowledge of the dynamics~$f$ of the system~\eqref{eq:controlled system} --- there are no explicit terms of $f$ in \eqref{eq:integral BE in IPI}, and the information on the dynamics $f$ is implicitly captured by the state trajectory data \smash{$\{X_t : 0 \leq t \leq \eta \}$} generated under $\pi_{i-1}$ at each of the $i$th iteration, for a number of initial states~$X_0 \in \mathcal{X}$. Note that by Theorem~\ref{thm:uniqueness condition}, solving the integral BE~\eqref{eq:integral BE in IPI} for a fixed $\eta > 0$ and its differential version~\eqref{eq:differential BE in DPI} in DPI are equivalent (as long as $v_i$ satisfies the boundary condition \eqref{eq:boundary condition for PI} in \secref{section:fundamental properties of PI}). 

 \begin{algorithm2e}[t!]
 \begin{spacing} {1.1}
 \DontPrintSemicolon
 \BlankLine
 \vspace{10pt}
 \nl \textbf{Initialize:} 
 	\smash{$\begin{cases} 
 			\textrm{$\pi_0$, an initial admissible policy;} \\
 			\textrm{$\eta > 0$, time horizon;} \\
 			\textrm{$i \leftarrow 1$, iteration index;}
  	  \end{cases}$}\;
 \vspace{12.5pt}
 \nl \Repeat{convergence is met.}
 {
 \nl \textbf{Policy Evaluation:} given $\pi_{i-1}$, find a $\mathrm{C}^1$ function $v_i: \mathcal{X} \to \mathbb{R}$ satisfying the integral BE:
 \begin{align}
 	\notag \\[-22.5pt]
 	v_{i}(x) = \mathbb{G}_{\pi_{i-1}}^x \;\!\! \big [ \mathfrak{R}_{\eta} + \gamma^{\eta} \! \cdot \! v_{i}(X_{\eta}) \big ] \;\, \forall x \in \mathcal{X}; 
    \label{eq:integral BE in IPI}
 \end{align}
  $\;$\;$\;$\\[-30pt]
 \nl\textbf{Policy Improvement:} find a policy $\pi_{i}$ such that
 \begin{align}
 	\notag \\[-20pt]
 	\!\!\! \pi_{i}(x) \in \smash{\Argmax_{u \in \mathcal{U}}} \big [ \;\! r(x,u) \! + \! \nabla v_i (x) f_{\mathsf{c}}(x,u) \;\! \big ] \;\, \forall x \in \mathcal{X};
 	\label{eq:policy improvement in IPI}
 \end{align}
 $\;$\;$\;$\\[-35pt] 
 \nl$i \leftarrow i+1$;
 }
 \end{spacing}
 \caption{Integral Policy Iteration (IPI)}
 \label{algorithm:IPI}
\end{algorithm2e}
		
	 In \emph{policy improvement} (line~4), we consider the decomposition~\eqref{eq:decomposition of f(x,u)} of the dynamics $f$: 
	\begin{equation}
		f(x,u) = f_{\mathsf{d}}(x) + f_{\mathsf{c}}(x,u), 
		\label{eq:decomposition of f(x,u)}
	\end{equation}
	where $f_{\mathsf{d}}:\mathcal{X} \to \mathcal{X}$ called a drift dynamics is independent of the action $u$ and assumed \emph{unknown}, and $f_{\mathsf{c}}: \mathcal{X} \times \mathcal{U} \to \mathcal{X}$ is the corresponding input-coupling dynamics assumed \emph{known a priori};\footnote{There are an infinite number of ways of choosing $f_{\mathsf{d}}$ and $f_{\mathsf{c}}$; one typical choice is $f_\mathsf{d}(x) = f(x,0)$ and $f_\mathsf{c}(x,u) = f(x,u) - f_\mathsf{d}(x)$.}  both $f_\mathsf{d}$ and $f_\mathsf{c}$ are assumed continuous. Since the term $\nabla v_\pi(x) f_\mathsf{d}(x)$ does not contribute to the maximization with respect to $u$, policy improvement~\eqref{eq:assumption:policy improvement} can be rewritten under the decomposition~\eqref{eq:decomposition of f(x,u)} as 
 \begin{align}
 	\pi'(x) \in \smash{\Argmax_{u \in \mathcal{U}}} \big [ \, r(x,u) + \nabla v_\pi (x) f_{\mathsf{c}}(x,u) \big ] \quad \forall x \in \mathcal{X}
 	\label{eq:partially model-free policy improvement}
	\\[-20pt] \notag 
 \end{align}
 by which the policy improvement (line~4) of Algorithm~\ref{algorithm:IPI} is directly obtained. Note that the policy improvement \eqref{eq:policy improvement in IPI} in Algorithm~\ref{algorithm:IPI} and \eqref{eq:partially model-free policy improvement} are partially model-free --- the maximizations do not depend on the unknown drift dynamics~$f_\mathsf{d}$. 

The policy evaluation and improvement of IPI are completely and partially model-free, respectively. Thus the whole procedure of Algorithm~\ref{algorithm:IPI} is partially model-free, i.e., it can be done even when a drift dynamics $f_{\mathsf{d}}$ is completely unknown. In addition to this partially model-free nature, the horizon $\eta > 0$ in IPI can be any value --- it can be large or small --- as long as the cumulative reward $\mathfrak{R}_\eta$ has no significant error when approximated in practice. In this sense, the time horizon~$\eta$ plays a similar role as the number~$n$ in the $n$-step TD predictions in discrete-time \citep{Sutton2018}. Indeed, if $\eta = n \Delta t$ for some $n \in \mathbb{N}$ and a sufficiently small $\Delta t > 0$, then by the forward-in-time approximation $\mathfrak{R}_\eta \approx G_n \cdot \Delta t$, where 
\[
	G_n \doteq R_0 + \gamma_\mathsf{d} \! \cdot \! R_{\Delta t} + \gamma_\mathsf{d}^2 \! \cdot \! R_{2 \Delta t} + \cdots + \gamma_\mathsf{d}^{n-1} \! \cdot \! R_{(n-1) \Delta t}
\]
and $\gamma_\mathsf{d} \doteq \gamma^{\Delta t} \in (0, 1]$, the integral BE~\eqref{eq:integral BE in IPI} is expressed as 
\begin{equation}
	V_{i}(x) \approx \mathbb{G}_{\pi_{i-1}}^x \! \big [ G_n + {\gamma}_\mathsf{d}^n \cdot V_{i}(X_{\eta}) \big ], 
	\label{eq:n-step BE}
\end{equation}
where $V_i \doteq v_i / \Delta t$. We can also apply a higher-order approximation of $\mathfrak{R}_\eta$ --- for instance, under the trapezoidal approximation, we have 
\[
	V_{i}(x) \approx \mathbb{G}_{\pi_{i-1}}^x \! \Big [ G_n + \tfrac{1}{2}  \cdot (\gamma_\mathsf{d}^n \! \cdot \! R_{\eta} - R_0) + {\gamma}_\mathsf{d}^n \cdot V_i (X_{\eta}) \Big ],
\]
which employs the end-point reward $R_\eta$ while \eqref{eq:n-step BE} does not. 
Note that the TD error~\eqref{eq:Doya's TD error} is not easy to generalize for such multi-step TD predictions. When $n = 1$, on the other hand, the $n$-step BE~\eqref{eq:n-step BE} becomes 
\begin{equation}
	V_i(x) \approx \mathbb{G}_{\pi_{i-1}}^x \big [ R_{0} + {\gamma}_\mathsf{d} \! \cdot \! V_i (X_{\! \Delta t}) \big ] \;\;\, \forall x \in \mathcal{X},
	\label{eq:1-step BE}
\end{equation}
which is similar to the BE in discrete-time \citep{Sutton2018} and $\mathbb{G}_{\pi}^x [\delta_t(v)] \approx 0$ for the TD error~\eqref{eq:Doya's TD error} in CTS.

\Subsection{Variants with Time Discretizations}{}

As discussed in \secsref{subsection:DPI} and \ref{subsection:IPI} above, the BEs in DPI and IPI can be discretized in time in order to \vspace{-0.75em}
\begin{enumerate}
	\item approximate ${\dot v}_i = \nabla v_i \cdot \! f$ in DPI, model-freely;
	\item calculate the cumulative reward $\mathfrak{R}_\eta$ in IPI;
	\item yield TD formulas similar to the BEs in discrete-time.
\end{enumerate}
For instance, for a sufficiently small $\Delta t$, the discretized BE corresponding to DPI and TD(0) in CTS \citep{Doya2000} is: 
\[
	V_\pi(x) \approx \mathbb{G}_\pi^x \big [ R_{\Delta t} + {\hat \gamma}_\mathsf{d} \! \cdot \! V_\pi(X_{\! \Delta t}) \big ] \qquad \forall x \in \mathcal{X},
\]
where $V_{\pi} \doteq v_\pi / \Delta t$. The discretized BE for IPI is obviously of the form~\eqref{eq:1-step BE} for $n = 1$ and \eqref{eq:n-step BE} for $n > 1$ (or one of the BEs with a higher-order approximation of $\mathfrak{R}_\eta$). If the integral BE~\eqref{eq:Bellman eq} is discretized with the trapezoidal approximation for $n = 1$, then we also have 
\[
	\smash{V_{\pi}(x) \approx \mathbb{G}_\pi^x \Big [  
	\tfrac{1}{2} \! \cdot \! (R_0 + \gamma_\mathsf{d} \! \cdot \! R_{\Delta t}) + {\gamma}_\mathsf{d} \! \cdot \! V_{\pi}(X_{\! \Delta t}) \Big ] \;\;\, \forall x \in \mathcal{X}.}
\]
Combining any one of those BEs, discretized in time, with the following policy improvement:
\[
 	\pi'(x) \in \Argmax_{u \in \mathcal{U}} \big [ r(x,u) + \Delta t \cdot \nabla V_\pi (x) f_{\mathsf{c}}(x,u) \big ] \; \forall x \in \mathcal{X}, 
\]
where $\Delta t \cdot \nabla V_\pi$ replaces $\nabla v_\pi$ in \eqref{eq:partially model-free policy improvement}, we can further obtain a partially model-free variant of the proposed PI methods. For example, a one-step IPI variant ($n = 1$) is shown in \secref{subsection:discounted RL with bounded v} (when the reward or initial VF is bounded). These variants are practically important since they contain neither ${\dot v}_i$ nor ${\dot V}_i$ (both of which depend on the full-dynamics $f$) nor the cumulative reward $\mathfrak{R}_\eta$ (which has been approximated out in the variants of IPI). As these variants are approximate versions of DPI and IPI, they also approximately satisfy the same properties as DPI and IPI shown in the subsequent sections.

\Section{Fundamental Properties of Policy Iterations}{section:fundamental properties of PI}

This section shows the fundamental properties of DPI and IPI --- admissibility, the uniqueness of the solution to each policy evaluation, monotone improvement, and convergence (towards an HJB solution). We also discuss the optimality of the HJB solution (\secsref{subsection:optimality:sufficient conditions} and \ref{subsection:sc for optimality}) based on the convergence properties of PIs. In any mathematical statements, $\langle v_i \rangle$ and $\langle \pi_i \rangle$ denote the sequences of the solutions to the BEs and the policies, both generated by Algorithm~\ref{algorithm:DPI} or \ref{algorithm:IPI} under: 
	
\noindent
\textbf{Boundary Condition.} \emph{If $\pi_{i-1}$ is admissible, then}
	\vspace{-0.3em}
	\begin{equation}
		\smash{\lim_{t \to \infty}} \mathbb{G}_{\pi_{i-1}}^x \big [ \gamma^{t} \cdot  v_i(X_{t}) \big ] = 0 \;\;\, \forall x \in \mathcal{X}.
		\label{eq:boundary condition for PI}	
	\end{equation}

\begin{theorem}
	\label{thm:fundamental properties}
	$\pi_{i-1}$ is admissible and $v_i = v_{\pi_{i-1}}$  $\forall i \in \mathbb{N}$. Moreover, the policies are monotonically improved, that is,
	\[
		\pi_{0} \preccurlyeq \pi_1 \preccurlyeq \cdots \preccurlyeq \pi_{i-1} \preccurlyeq \pi_i \preccurlyeq \cdots.
	\]
\end{theorem}

\begin{theorem} [Convergence]
 	\label{thm:convergence of PI}
 	Denote ${\hat v}_*(x) \doteq \sup_{i \in \mathbb{N}} v_i(x)$. Then, ${\hat v}_*$ is lower semicontinuous; $v_i \to {\hat v}_*$ \textbf{\emph{a.}} pointwise; 
 	\textbf{\emph{b.}} uniformly on $\Omega \subset \mathcal{X}$ if $\Omega$ is compact and ${\hat v}_*$ is continuous over~$\Omega$; \textbf{\emph{c.}} locally uniformly if ${\hat v}_*$ is continuous. 
\end{theorem}

In what follows, ${\hat v}_*$ always denotes the converging function ${\hat v}_*(x) \doteq \sup_{i \in \mathbb{N}} v_i(x) = \lim_{i \to \infty} v_i(x)$ in Theorem~\ref{thm:convergence of PI}.

\Subsection{Convergence towards $v_*$ and $\pi_*$}{subsection:convergence}

Now, we provide convergence $v_i \to v_*$ to a solution $v_*$ of the HJBE~\eqref{eq:HJBE}. One core technique is to use 
the PI operator $\mathcal{T}: \mathcal{V}_{\mathsf{a}} \to \mathcal{V}_{\mathsf{a}}$ defined on the space $\mathcal{V}_{\mathsf{a}}$ of admissible VFs as 
		\begin{align*}
			\begin{cases}
				\mathcal{T} v_{\pi_{i - 1}} \doteq v_{\pi_{i}} \textrm{ for any $i \in \mathbb{N}$;}
				\\
				\mathcal{T}v_{\pi}  \,\!\,\, \doteq v_{\pi'} \;\;\;\,\,\!\! \textrm{ for any other $v_\pi \in \mathcal{V}_\mathsf{a}$,}
			\end{cases}				
		\end{align*}
where $\pi'$ is a maximal policy over the given policy $\pi \in \Pi_\mathsf{a}$. Let $\mathcal{T}^N$ be the $N$th recursion of $\mathcal{T}$ defined as 
$
	\mathcal{T}^0 v \doteq v
$ 
and 
$
	\mathcal{T}^N v \doteq \mathcal{T}^{N-1}(\mathcal{T}v)
$ for $v \in \mathcal{V}_\mathsf{a}$.
Then, the VF sequence $\langle v_i \rangle$ satisfies 
$
	v_{i} = \mathcal{T}^{i-1} v_1
$ for all $i \in \mathbb{N}$.

In what follows, we denote $v^*$ a (unique) fixed point of $\mathcal{T}$.

\begin{proposition}
	\label{prop:a fixed point is a solution to HJBE}
	If $v^*$ is a fixed point of $\mathcal{T}$, then $v^* = v_*$, i.e., $v^*$ is a solution $v_*$ to the HJBE~\eqref{eq:HJBE}.	
\end{proposition}

By Proposition~\ref{prop:a fixed point is a solution to HJBE}, convergence $v_i \to v^*$ implies that $\langle v_i \rangle$ converges towards a solution $v_*$ to the HJBE~\eqref{eq:HJBE}. In what follows, we first show the convergence $v_i \to v^*$ under:

\begin{assumption}
	\label{assumption:uniqueness of fixed point}
	$\mathcal{T}$ has a unique fixed point $v^*$.
\end{assumption}

\begin{theorem}
 	\label{thm:true convergence in metric}
 	Under Assumption~\ref{assumption:uniqueness of fixed point}, there exists a metric $d: \mathcal{V}_{\mathsf{a}} \times \mathcal{V}_{\mathsf{a}} \to [0, \infty)$ such that $\mathcal{T}$ is a contraction (and thus continuous) under $d$ and $v_i \to v^*$ in the metric~$d$.
 \end{theorem}

Theorem~\ref{thm:true convergence in metric} shows the convergence $v_i \to v^*$ in a metric~$d$ under which $\mathcal{T}$ is continuous. However, there is no information about which metric it is. In what follows, we focus on locally uniform convergence, in connection to Theorem~\ref{thm:convergence of PI}. Let $d_\Omega$ be a pseudometric on $\mathcal{V}_\mathsf{a}$ defined for $\Omega \subseteq \mathcal{X}$ as 
\[
	d_\Omega(v,w) \doteq \sup \! \big \{ \big | v(x)  - w(x) \big | \! : x \in \Omega \big \} \textrm{ for } v, \; w \in \mathcal{V}_{\mathsf{a}}. 
\]
Then, uniform convergence $v_i \to v^*$ on $\Omega$ becomes equivalent to convergence $v_i \to v^*$ in the pseudometric $d_\Omega$. 

 \begin{theorem}
 	\label{thm:true uniform convergence:1}
 	Suppose ${\hat v}_* \in \mathcal{V}_\mathsf{a}$ and for each compact subset $\Omega$ of $\mathcal{X}$, $\mathcal{T}$ is continuous under $d_\Omega$. If Assumption~\ref{assumption:uniqueness of fixed point} is true, then $v_i \to v^*$ locally uniformly and $v^* = {\hat v}_*$.
\end{theorem}

The convergence condition in Theorem~\ref{thm:true uniform convergence:1} comes from \citet{Leake1967}'s approach that is now extended to our RL framework. The next theorem is motivated by the convergence results of PIs for optimal control of input-affine dynamics \citep{Saridis1979,Beard1997,Murray2002,Abu-Khalaf2005,Vrabie2009b} and provides the conditions for stronger convergence towards $v_*$ and $\pi_*$.

\begin{assumption}
	\label{assumption:closed graph}
	For each $x \in \mathcal{X}$, the {\rm argmax}-correspondence $p \mapsto \Argmax_{u \in \mathcal{U}} h(x, u, p)$ has a closed graph. That is, for each $x \in \mathcal{X}$ and any sequence $\langle p_k \rangle$ in $\mathcal{X}^\mathsf{T}$ converging to $ p_*$, 
	\[
		\begin{cases}
			u_k \in \smash{\displaystyle \Argmax_{u \in \mathcal{U}}} \; h(x, u, p_k)
			\\[0.75em]
			\lim_{k \to \infty} u_k = u_* \in \mathcal{U}
		\end{cases}
		\hspace{-1em}
		\Longrightarrow
		u_* \in \smash{\displaystyle \Argmax_{u \in \mathcal{U}}} \; h(x, u, p_*).
	\]
\end{assumption}
\vspace{-1.5em}
\begin{assumption}
	\label{assumption:convergence of nabla v_i and pi_i}
	$
	\!\!
	\begin{cases}
		\text{\bf a.} \text{ $\langle \nabla v_i \rangle$ converges locally uniformly;}
		\\[2.5pt]
		\text{\bf b.} \text{ $\langle \pi_i \rangle$ converges pointwise.}
	\end{cases}
	$
\end{assumption}
\vspace{-0.5em}
\begin{theorem}
 	\label{thm:true uniform convergence:2}
 	Under Assumptions~\ref{assumption:closed graph} and \ref{assumption:convergence of nabla v_i and pi_i}, ${\hat v}_*$ is a solution $v_*$ to the HJBE~\eqref{eq:HJBE} such that $v_* \in \mathrm{C}^1$ and
 	\begin{enumerate} 
 		\item $v_i \to v_*$, $\nabla v_i \to \nabla v_*$ both locally uniformly;
 		\item $\pi_i \to \pi_*$ pointwise, for a function $\pi_*$ satisfying~\eqref{eq:optimal policy}.
 	\end{enumerate}
 \end{theorem}
 
\begin{remark}
	\label{remark:argmax assumption}
	If the {\rm argmax}-set is a singleton (so the maximal function $u_*$ satisfying \eqref{eq:assumption:general condition for policy improvement} is unique), then Assumption~\ref{assumption:closed graph} is equivalent to continuity of $p \mapsto u_*(x, p)$ for each $x \in \mathcal{X}$ and thus implied by the continuity of $u_*$ assumed in \secref{subsection:policy improvement}! In this particular case, $\pi_*$ in Theorem~\ref{thm:true uniform convergence:2} is uniquely given by \eqref{eq:closed-form expression of optimal policy} hence continuous (i.e., $\pi_*$ satisfies our definition of a policy). For such examples, see \secsref{subsubsection:case I} and \ref{subsection:LQR}. 
\end{remark}

In summary, we have established the following convergence properties:
\begin{enumerate} [leftmargin=1.0cm, itemsep=0.25em, topsep=-0.65em]
	\item [($\mathsf{C}1$)] convergence $v_i \to v_*$ in a metric;
	\item [($\mathsf{C}2$)] locally uniform convergence $v_i \to v_*$;
	\item [($\mathsf{C}3$)] 
		locally uniform convergence $\nabla v_i \to \nabla v_*$, and \\
		pointwise convergence $\pi_i \to \pi_*$, 
\end{enumerate}
under certain conditions and the minimal assumptions made in this section and \secref{section:preliminaries}. 

\textbf{(Weak/Strong Convergence)} Theorem~\ref{thm:true convergence in metric} ensures weak convergence ($\mathsf{C}1$) under Assumption~\ref{assumption:uniqueness of fixed point} only. Theorem~\ref{thm:true uniform convergence:1} gives strong convergence ($\mathsf{C}2$), provided that the following additional conditions hold: (i) continuity of $\mathcal{T}$ in the uniform pseudometric $d_\Omega$; (ii) the convergence within the admissible space $\mathcal{V}_\mathsf{a}$: $\lim_{i \to \infty} v_i \in \mathcal{V}_\mathsf{a}$, i.e., ${\hat v}_* \in \mathcal{V}_\mathsf{a}$. We note that
	\begin{enumerate}
		\item the unique fixed point $v^*$ therein and in Assumption~\ref{assumption:uniqueness of fixed point} is a solution $v_*$ to the HJBE~\eqref{eq:HJBE} (Proposition~\ref{prop:a fixed point is a solution to HJBE});
		\item whenever ($\mathsf{C}2$) is true, both $v^*$ and $v_*$ therein are characterized by Theorem~\ref{thm:convergence of PI} as $v^* = v_* = {\hat v_*}$. 
	\end{enumerate}

\textbf{(Stronger Convergence)} 
If the convergence described in Assumptions~\ref{assumption:closed graph} and \ref{assumption:convergence of nabla v_i and pi_i} are all true, then Theorem~\ref{thm:true uniform convergence:2} ensures the stronger convergence properties ($\mathsf{C}2$) and ($\mathsf{C}3$) for $v_* = {\hat v}_* \in \mathrm{C}^1$, wherein the limit function ${\hat v}_*$ ($= \lim_{i \to \infty} v_i$) becomes a solution $v_*$ to the HJBE~\eqref{eq:HJBE}. In this case, 
	\begin{enumerate}
		\item $\mathcal{T}$ is never used, hence no assumption is imposed on $\mathcal{T}$;
		\item $\pi_*$ in ($\mathsf{C}3$) is not necessarily a policy by our definition due to its possible discontinuity (see also Remark~\ref{remark:argmax assumption});
		\item the concave Hamiltonian formulation in \secref{subsection:RL under u-AC setting} ensures $\pi_i \to \pi_*$ \emph{locally uniformly} for a \emph{policy} $\pi_*$, with both Assumptions~\ref{assumption:closed graph} and \ref{assumption:convergence of nabla v_i and pi_i}b \emph{relaxed} (e.g., Theorem~\ref{thm:true convergence}).
	\end{enumerate}
	
\Subsection{Optimality of the HJB Solution: Sufficient Conditions}{subsection:optimality:sufficient conditions}

For each type of convergence above, we provide a sufficient condition for $v_*$ in the HJBE~\eqref{eq:HJBE} to be optimal in the sense that for any given initial admissible policy $\pi_0$, $v_i \to v_*$ in the respective manner with monotonicity $v_i \leqslant v_{i+1}$ $\forall i \in \mathbb{N}$. For the optimality of $v_*$ with the stronger convergence, ($\mathsf{C}2$) and ($\mathsf{C}3$), we additionally assume that:
\begin{assumption}
	\label{assumption:uniqueness of HJB over C1}
	The solution $v_*$ to the HJBE~\eqref{eq:HJBE}, if exists, is unique over $\mathrm{C}^1$ and upper-bounded (by zero if $\gamma = 1$).
\end{assumption}
Those sufficient conditions for optimality and related discussions are presented in Appendix \secref{subsection:sc for optimality}.

\Section{Case Studies}{section:case studies}

With strong connections to RL and optimal control in CTS, this section studies the special cases of the general RL problem formulated in \secref{section:preliminaries}. In those case studies, the proposed PI methods and theory for them are simplified and improved as summarized in Table~\ref{table:summary of case studies}. The blanks in Table~\ref{table:summary of case studies} are filled with ``Assumed'' or, in simplified policy improvement sections, ``No''. The connections to stability theory in optimal control are also made in this section. The optimality of the HJB solution $(v_*, \pi_*)$ for each case is studied and summarized in \secref{appendix:subsection:optimality case studies}; more case studies are given in \secref{appendix:section:additional case studies}.

For simplicity, we let $f^x(u) \doteq f(x,u)$ and $r^x(u) \doteq r(x,u)$ for $x \in \mathcal{X}$. Both $f^x$ and $r^x$ are continuous for each $x$ since so are $f$ and $r$. The mathematical terminologies employed in this section are given in \secref{appendix:notations and terminologies}, with a summary of notations. \pagebreak

\begin{table*}[htb!]
\scriptsize
\caption{Summary of Case Studies: Relaxations and Simplifications of the Assumptions and Policy Improvement}
\label{table:summary of case studies}
\begin{center}
\begin{tabular} {l||c|c|c|c|c|c|}
    \toprule 
    \\[-1.65em]
	\multirow{2}{*}{Problem Formulation} 
	& \multirow{2}{*}{\makecell{\vspace{-12.5pt}\\Concave \\ $\!\!$Hamiltonian$\!\!$}}
	& \multicolumn{2}{c|}{Discounted RL with bounded }
	& \multirow{2}{*}{\makecell{\vspace{-12.5pt}\\RL with local\\ $\!\!$Lipschitzness$^{(\mathrm{b})}\!\!\!$}}
	& \multirow{2}{*}{\makecell{\vspace{-12.5pt}\\Nonlinear \\ $\!\!$optimal control$^{(\mathrm{b})}$$\!\!$}}  
	& \multicolumn{1}{c}{\multirow{2}{*}{LQR}}
	\cr
	\cline{3-4}
	& 
	& VF$^{(\mathrm{a})}$
	& state trj.
	&
	&
	&
	\multicolumn{1}{c}{}
	\cr
	\hline 
	Section  &  \ref{subsection:RL under u-AC setting} / \ref{subsubsection:case III} & \ref{subsection:discounted RL with bounded v}  & 
	\ref{subsection:discounted RL with bounded state trj} & 
	\ref{subsection:RL with local Lipschitzness} &
	\ref{subsection:nonlinear optimal control} & \multicolumn{1}{c}{\ref{subsection:LQR}}
	\cr
    \hdrule
    \cline{4-7}
    Global existence and uniqueness of state trjs. & \multicolumn{2}{c}{\multirow{3}{*}{\vspace{-1.5cm}}} & \multicolumn{3}{|c|}{True, conditionally$^{(\mathrm{c})}$} & \multirow{6}{*}{\vspace{-2cm} True}
	\bigstrut\cr	
	\cline{1-1}\cline{3-6}
    Existence of an admissible policy, i.e., $\mathcal{V}_\mathsf{a} \neq \varnothing$ & \multicolumn{1}{c|}{} & \multicolumn{2}{c|}{True} & \multicolumn{2}{c|}{} &
   	\bigstrut \cr	
	\cline{1-1}\cline{3-4}
    \makecell[l]{$\mathrm{C}^1$-regularity \eqref{eq:assumption:every admissible VF is C1} and continuity of admissible VFs} & \multicolumn{1}{c|}{} & \makecell{$\;$\\[-1.7em] Continuous, \\ conditionally$^{(\mathrm{b})}$\\[-0.5em]} & \multicolumn{3}{c|}{} &   \cr
	\cline{1-1}\cline{3-3}
    \makecell[l]{Assumptions~\ref{assumption:uniqueness of fixed point} and \ref{assumption:uniqueness of HJB over C1} (w.r.t. $\mathcal{T}$ and the HJBE)}&  \multicolumn{5}{c|}{} & 
    \bigstrut \cr
	\cline{1-1}\cline{2-2}
    Existence of a continuous maximal function $u_*$  & True &  \multicolumn{4}{c|}{} &
    \bigstrut \cr
	\cline{1-1}\cline{2-4}\cline{6-6}
    \makecell[l]{Boundary conditions~\eqref{eq:uniqueness condition for v} and \eqref{eq:boundary condition for PI}}  & \multicolumn{1}{c|}{} & \makecell{$\;$\\[-1.7em] True, \\ conditionally$^{(\mathrm{d})}$\\[-0.5em]} & \multicolumn{1}{c|}{True} & \multicolumn{1}{c|}{} & \makecell{$\;$\\[-1.7em] True, \\ conditionally$^{(\mathrm{e})}$\\[-0.5em]} &
    \cr  
    \cline{1-4}\cline{6-6}
    Assumptions~\ref{assumption:closed graph} and \ref{assumption:convergence of nabla v_i and pi_i} for $\mathsf{(C2)}$ and $\mathsf{(C3)}$ &  \multicolumn{1}{c|}{Relaxed$^{\mathrm{(f)}}$} & \multicolumn{4}{c|}{} & 
    \bigstrut\cr  
    \cline{1-1}\cline{2-2}\cline{7-7}
    Simplified policy improvement  & Yes & \multicolumn{4}{c|}{} & Yes
    \bigstrut \cr  
    \cline{2-2}\cline{7-7}
    \bottomrule
\end{tabular}
\end{center}
$\;$\\
${}^{(\mathrm{a})}$ Once the initial VF \smash{$v_{\pi_0}$} in the PI methods is bounded, so is \smash{$v_{\pi_i}$} for all $i \in \mathbb{N}$; a stronger case is when the reward function $r$ is bounded. \\
${}^{(\mathrm{b})}$ $f$ and/or $f_\pi$ is assumed locally Lipschitz. \\
${}^{(\mathrm{c})}$ True if $f_\pi$ is locally Lipschitz in \secref{subsection:discounted RL with bounded state trj} and in addition, in \secsref{subsection:RL with local Lipschitzness} and \ref{subsection:nonlinear optimal control}, if $\pi \in \Pi_\mathsf{a}$ (see the modified definitions of $\Pi_\mathsf{a}$ therein). \\
${}^{(\mathrm{d})}$ True if $v$ and $v_i$ are bounded --- this makes sense only when the target VF is bounded. \\
${}^{(\mathrm{e})}$ See Theorems~\ref{thm:uniqueness under gas} (attractiveness and asymptotic stability) and \ref{thm:uniqueness condition in optimal control} (conditions in Theorem~\ref{thm:boundary condition:sufficiency} of \secref{appendix:sec:replacement of bdy condition}), both for \eqref{eq:uniqueness condition for v}. See also Theorem~\ref{thm:fundamental properties in optimal control} for \eqref{eq:boundary condition for PI}.\\
${}^{(\mathrm{f})}$ Assumptions~\ref{assumption:closed graph} and \ref{assumption:convergence of nabla v_i and pi_i} are reduced to Assumption~\ref{assumption:convergence of nabla v_i and pi_i}a (see Theorems~\ref{thm:true uniform convergence:2}, \ref{thm:true convergence}, and \ref{thm:true convergence:nonaffine case}).
\end{table*}

\Subsection{Concave Hamiltonian Formulations}{subsection:RL under u-AC setting}

Here, we study the special settings of the reward function~$r$, which make the function $u \mapsto h(x, u, p)$ strictly concave and $\mathrm{C}^1$ (after some input-transformation in the cases of non-affine dynamics). In these cases, policy improvement maximizations \eqref{eq:assumption:general condition for policy improvement}, \eqref{eq:assumption:policy improvement}, and \eqref{eq:optimal policy} become convex optimizations whose solutions exist and are given in closed-forms.	We will see that this dramatically simplifies the policy improvement itself and strengthen the convergence properties. Although we focus on certain classes of dynamics --- the input-affine and then a class of non-affine ones --- the idea is extendible to a general nonlinear system of the form~\eqref{eq:controlled system} (see \secref{subsubsection:case III} for such an extension).

\subsubsection{Case I: Input-affine Dynamics}
\label{subsubsection:case I}
First, consider the following case: for each $x \in \mathcal{X}$, 
\begin{enumerate}
	\item $f^x$ \emph{is affine}, i.e., the input-coupling term $f_{\mathsf{c}}(x,u)$ in the decomposition \eqref{eq:decomposition of f(x,u)} is linear in $u$, so that the dynamics $f$ can be represented as 
			\begin{equation}
				f(x,u) = f_{\mathsf{d}}(x) + F_{\mathsf{c}}(x)u
				\label{eq:input-affine dynamics}
			\end{equation}
		for a matrix-valued continuous function $F_{\mathsf{c}}: \mathcal{X} \to \mathbb{R}^{l \times m};\!\!\!\!\!$
	\item $r^x$ \emph{is strictly concave and represented by} 
		\begin{equation}
			r(x,u) = \mathfrak{r}(x) - \mathfrak{c}(u)
			\label{eq:strictly concave reward}
		\end{equation}
	where $\mathfrak{r}: \mathcal{X} \to \mathbb{R}$ is continuous, $\mathfrak{c}: \mathcal{U} \to \mathbb{R}$ is \emph{strictly convex} and $\mathrm{C}^1$, and its  gradient $\nabla \mathfrak{c}$ is \emph{surjective}, i.e., $\nabla \mathfrak{c}(\Int{\mathcal{U}}) = \mathbb{R}^{1 \times m}$. 
	Here, $\Int{\mathcal{U}}$ is the interior of $\mathcal{U}$. 
\end{enumerate}

This framework includes those in (\citealp{Rekasius1964,Beard1997,Doya2000, Abu-Khalaf2005, Vrabie2009b}; \citealp*{Lee2015}) as special cases;
it still contains a broad class of dynamics such as Newtonian dynamics (e.g., robot manipulator and vehicle models). 
In this case, the mapping  $u \mapsto h(x, u, p)$ is strictly concave and $\mathrm{C}^1$ (see the definition \eqref{eq:def:Hamiltonian} of $h$). Hence, as mentioned in \secref{subsection:policy improvement} (see \secref{appendix:section:maximal function} for the behind theory), the unique maximal function $u_* \equiv u_*(x, p)$ satisfying \eqref{eq:assumption:general condition for policy improvement} corresponds to the unique regular point $\bar u \in \mathcal{U}^\circ$ s.t.
$
	- \nabla \mathfrak{c}(\bar u) + p \, F_{\mathsf{c}}(x) = 0,
$
where the gradient $\nabla \mathfrak{c}^\mathsf{T}: \Int{\mathcal{U}} \to \mathbb{R}^m$ is strictly monotone and bijective on its domain $\Int{\mathcal{U}}$ (see \secref{appendix:proofs:case studies}). Rearranging it w.r.t. $\bar u$, we obtain the closed-form solution $u_*$ of \eqref{eq:assumption:general condition for policy improvement}: 
\begin{equation}
	u_*(x, p) = \sigma \big ( F_{\mathsf{c}}^\mathsf{T}(x) \, p^\mathsf{T} \big ),
	\label{eq:u bar without max}
\end{equation}
where $\sigma \doteq (\nabla \mathfrak{c}^\mathsf{T})^{-1}$ denotes the inverse of $\nabla \mathfrak{c}^\mathsf{T}$. Here, the mapping $\sigma : \mathbb{R}^m \to \Int{\mathcal{U}}$ is also strictly monotone and continuous (see \secref{appendix:proofs:case studies}); thus, $u_*$ is continuous. Substituting \eqref{eq:u bar without max} into \eqref{eq:closed-form expression of the maximal policy under uniqueness}, we obtain the \emph{unique closed-form solution} of the policy improvement maximization \eqref{eq:assumption:policy improvement} (or \eqref{eq:partially model-free policy improvement}):
\begin{equation}
	\pi'(x) = \sigma \big ( F_{\mathsf{c}}^\mathsf{T}(x) \nabla v_\pi^\mathsf{T}(x) \big )
	\label{eq:policy improvement without max}
\end{equation}
	a.k.a. the value-gradient-based (VGB) greedy policy update \citep{Doya2000}. 
	This simplifies the policy improvement of DPI and IPI (and their variants) shown in \secref{section:PIs} as 
\begin{description}
	\item \textbf{Policy Improvement:} update the next policy $\pi_{i}$ by
\begin{equation*}
	\hspace{-2em}
	\pi_{i}(x) = \sigma \big ( F_{\mathsf{c}}^\mathsf{T}(x) \nabla v_i^\mathsf{T}(x) \big ).
\end{equation*}
\end{description}
Similarly, the HJB policy $\pi_*$ satisfying \eqref{eq:optimal policy} is also \emph{uniquely} given by \eqref{eq:closed-form expression of optimal policy} and \eqref{eq:u bar without max}, i.e., \smash{$\pi_*(x) = \sigma \big (F_\mathsf{c}^\mathsf{T} (x) \nabla v_*^\mathsf{T} (x)\big )$}, under \eqref{eq:input-affine dynamics} and \eqref{eq:strictly concave reward}. Moreover, Theorem~\ref{thm:true uniform convergence:2} can be simplified and strengthened, with Assumptions~\ref{assumption:closed graph} and \ref{assumption:convergence of nabla v_i and pi_i}b \emph{relaxed}.

 \begin{theorem}
 	\label{thm:true convergence}
 	Under \eqref{eq:input-affine dynamics}, \eqref{eq:strictly concave reward}, and Assumption~\ref{assumption:convergence of nabla v_i and pi_i}a, ${\hat v}_*$ is a solution $v_*$ to the HJBE~\eqref{eq:HJBE} such that $v_* \in \mathrm{C}^1$ and 
 	$
 		v_i \to v_*, \nabla v_i \to \nabla v_*, \text{ and } \pi_i \to \pi_*,
 	$
 	all locally uniformly.
 \end{theorem}
 
\begin{remark}
\label{remark:the needs for locally uniform convergence}	
Assumption~\ref{assumption:convergence of nabla v_i and pi_i}a is necessary for convergence in Theorem~\ref{thm:true convergence} and, in fact so are similar uniform convergence assumptions on $\langle \nabla v_i \rangle$ for convergence given in the existing literature on PIs for optimal control (e.g., \citealp{Saridis1979,Beard1997}; \citealp*{Murray2003}; \citealp{Abu-Khalaf2005,Bian2014} to name a few). This is due to the fact that even the uniform convergence of $v_i$ (e.g., Theorem~\ref{thm:convergence of PI}c) implies nothing about the convergence of its gradient $\nabla v_i$; it cannot even ensure the differentiability of the limit function ${\hat v}_*$ \citep*{Rudin1964,Thomson2001}. Here, Assumption~\ref{assumption:convergence of nabla v_i and pi_i}a or any type of (uniform) convergence of $\langle \nabla v_i \rangle$ is by no means trivial to prove, and thus its relaxation remains as a future work (even in the optimal control frameworks in the existing literature, which are similar to that in \secref{subsection:nonlinear optimal control} under \eqref{eq:input-affine dynamics}--\eqref{eq:strictly concave reward}, to the best authors' knowledge).
\end{remark}

One way to effectively take the input constraints into considerations is to construct the action space~$\mathcal{U}$ as 
	\[
		\mathcal{U} = \big \{ u \in \mathbb{R}^m :  | u_j | \leq {u}_{\mathsf{max},j}, \; 1 \leq j \leq m \big \}	,
	\]
where $u_j \in \mathbb{R}$ is the $j$th element of $u$, and ${u}_{\mathsf{max}, j} \in (0, \infty]$ is the corresponding physical constraint. In this case, $\mathfrak{c}$ in \eqref{eq:strictly concave reward} can be chosen as 
	\vspace{-0.1em}
	\begin{align}
		\mathfrak{c}(u) = \smash{\lim_{v \to u}} \; \smash{\int_0^v} (s^{\mathsf{T}})^{-1} (\mathfrak{u}) \cdot \Gamma \, d\mathfrak{u}
		\label{eq:integral formula of S}
	\end{align}
	for a positive definite matrix $\Gamma \in \mathbb{R}^{m \times m}$ and a continuous function $s: \mathbb{R}^m \to \Int{\mathcal{U}}$ that is strictly monotone, odd, and bijective and makes $\mathfrak{c}(u)$ in \eqref{eq:integral formula of S} finite at any point~$u$ on the boundary $\Bdy{\mathcal{U}}$;\footnote{$\Bdy{\mathcal{U}} = \{u \in \mathbb{R}^m: u_j = u_{\mathsf{max},j} \textrm{ for some } j = 1,2,\cdots, m\} $.} 
	This formulation gives a closed-form expression 	
	$
		\sigma(\mathfrak{u}) \! = \! (\nabla \mathfrak{c}^\mathsf{T})^{\!-1} (\mathfrak{u})  = s(\Gamma^{-1} \mathfrak{u})
	$
	and includes the sigmoidal examples (Cases 1 and 2) in \secref{section:simulation} as special cases --- see also \citep{Doya2000,Abu-Khalaf2005} for similar sigmoidal examples. Another well-known example is the unconstrained problem: 
	\begin{equation}
		\mathcal{U} = \mathbb{R}^m \;\, (u_{\mathsf{max},j} = \infty \textrm{ for each } j) \textrm{ and } s(\mathfrak{u}) = \mathfrak{u} /2, 
		\label{eq:unconstrained case}
	\end{equation}
	by which \eqref{eq:integral formula of S} becomes $\mathfrak{c}(u) = u^{\mathsf{T}} \Gamma u$; the LQR case in \secref{subsection:LQR} with $E = 0$ shows such an example. 
	\vspace{-0.25em}

\begin{remark}
	\label{remark:generalization of u-AC RL problem:1}
	Once $r^x$ is strictly concave for each $x \in \mathcal{X}$, 
	the reward function $r$ can be always represented as 
	\begin{equation}
			r(x,u) = \mathfrak{r}(x) - \mathfrak{c}(x,u),	
			\label{eq:strictly concave reward:general}
	\end{equation}
	where $\mathfrak{r}$ and $\mathfrak{c}$ are continuous; $\mathfrak{c}^x \doteq \mathfrak{c}(x, \cdot)$ for each $x \in \mathcal{X}$ is strictly convex. In this general case, if $\mathfrak{c}^x$ is $\mathrm{C}^1$ and its gradient $\nabla \mathfrak{c}^x$ is surjective for each $x \in \mathcal{X}$, then	the unique maximal function $u_*$ and policy $\pi'$ over $\pi \in \Pi_\mathsf{a}$ can be obtained in the same way to \eqref{eq:u bar without max} and \eqref{eq:policy improvement without max}	 as 
	\begin{equation*}
		\begin{cases}
		{u_*}(x, p) = \sigma^x \big ( F_{\mathsf{c}}^\mathsf{T}(x) \, p^\mathsf{T} \big )
		\\[5pt]
		\pi'(x) = \sigma^x \big ( F_\mathsf{c}^\mathsf{T}(x) \nabla v_\pi^\mathsf{T} (x) \big )
		\end{cases}
	\end{equation*}
	for the inverse $\sigma^x$ of $(\nabla \mathfrak{c}^{x})^\mathsf{T}$. In addition, 
	 if $(x, \mathfrak{u}) \mapsto \sigma^x(\mathfrak{u})$ is continuous, then Theorem~\ref{thm:true convergence} (specifically, Lemma~\ref{lemma:concave h} in \secref{appendix:proofs:case studies}) can be generalized with $\sigma$ replaced by $\sigma^x$. Some examples of such $\sigma^x$ are as follows. 
	\begin{enumerate}
		\item $\Gamma$ in \eqref{eq:integral formula of S} is a continuous function over $\mathcal{X}$. In this case, $\sigma^x$ is given by 
		$
			\sigma^x(\mathfrak{u}) = s(\Gamma^{-1}(x) \cdot \mathfrak{u})
		$. 
		\item In the LQR setting (\secref{subsection:LQR}), $\sigma^x(\mathfrak{u}) = \Gamma^{-1} ( \mathfrak{u} / 2  - E^\mathsf{T} x)$	and whenever $E = 0$, $\sigma^x(\mathfrak{u}) = \sigma(\mathfrak{u}) = \Gamma^{-1}\mathfrak{u} / 2$.
	\end{enumerate}
\end{remark}

\subsubsection{Case II: a Class of Non-affine Dynamics}
\label{subsubsection:case II}

If $f^x$ is not affine, then the choice of the reward function~$r$ is critical. Provided in \secref{appendix:subsection:a pathological ex} is such an example, where a choice of $r$ in the form of \eqref{eq:strictly concave reward} and \eqref{eq:integral formula of S} fails to give closed-form solutions to policy improvement and the HJBE~\eqref{eq:HJBE}. Moreover, in the unconstrained case, such a choice of $r$ may result in a pathological Hamiltonian $h$ as shown in \secref{appendix:subsection:a pathological ex}.

Such pathological behavior and difficulty, on the other hand, can be avoided for the non-affine dynamics $f$ of the form:
	\begin{equation}
			f(x,u) = f_{\mathsf{d}}(x) + F_\mathsf{c}(x) \varphi(u),  
			\label{eq:a class of non-affine dynamics}
	\end{equation}
	where $\varphi: \mathcal{U} \to \mathcal{A} \subseteq \mathbb{R}^m$ is a continuous function from the action space~$\mathcal{U}$ to another action space $\mathcal{A}$ 	and has its inverse $\varphi^{-1}: \Int{\mathcal{A}} \to \Int{\mathcal{U}}$ between the interiors. Note that \eqref{eq:a class of non-affine dynamics} corresponds to the decomposition \eqref{eq:decomposition of f(x,u)} with the input-coupling part $f_\mathsf{c}(x, u) = F_\mathsf{c}(x) \varphi(u)$ and includes the input-affine dynamics~\eqref{eq:input-affine dynamics} as a special case $\varphi(u) = u$ and $\mathcal{A} = \mathcal{U}$. 

Motivated by \citet*{Kiumarsi2016}, we propose to set the reward function~$r$ under \eqref{eq:a class of non-affine dynamics} as 
	\begin{equation}
		r(x, u) = \mathfrak{r}(x) - \mathfrak{c}(\varphi(u)),
		\label{eq:a class of reward fnc for non-affine dynamics}
	\end{equation}
	where $\mathfrak{r}: \mathcal{X} \to \mathbb{R}$ and $\mathfrak{c}: \mathcal{A} \to \mathbb{R}$ are functions that satisfy the properties of $\mathfrak{r}$ and $\mathfrak{c}$ in \eqref{eq:strictly concave reward} but w.r.t. the action space $\mathcal{A}$ in place of $\mathcal{U}$. Under \eqref{eq:a class of non-affine dynamics} and \eqref{eq:a class of reward fnc for non-affine dynamics}, the proposed PIs have the following properties, extended from \secref{subsubsection:case I} (e.g., from Theorem~\ref{thm:true convergence}), although the argmax-set ``$\Argmax_{u \in \mathcal{U}} h(x, u, p)$'' in this case may not be a singleton (another maximizer may exist on the boundary $\Bdy{\mathcal{U}}$). 
	
\begin{theorem}
 	\label{thm:true convergence:nonaffine case}
 	Let $\tilde \sigma (\mathfrak{u}) \doteq \varphi^{-1} [ \sigma(\mathfrak{u})]$. Under \eqref{eq:a class of non-affine dynamics} and \eqref{eq:a class of reward fnc for non-affine dynamics}, 
 	\begin{enumerate}
 		\item [\textbf{\emph{a.}}] a maximal policy $\pi'$ over $\pi \in \Pi_\mathsf{a}$ is explicitly given by 
 			\[
 				\pi'(x) = {\tilde \sigma} \big ( F_{\mathsf{c}}^\mathsf{T}(x) \nabla v_\pi^\mathsf{T}(x) \big );
 			\]
 		\item [\textbf{\emph{b.}}] if the policies are updated in policy improvement by 
	\begin{equation*}
		\pi_{i}(x) =
		\tilde \sigma \big ( F_{\mathsf{c}}^\mathsf{T}(x) \nabla v_i^\mathsf{T}(x) \big ),
	\end{equation*}
	then under Assumption~\ref{assumption:convergence of nabla v_i and pi_i}a, 
	${\hat v}_*$ is a solution $v_*$ to the HJBE \eqref{eq:HJBE} s.t. $v_* \! \in \mathrm{C}^1$ and 
	$v_i  \to v_*$, $\nabla v_i \to \!\nabla v_*$, and $\pi_i \to \pi_*$,
	all locally uniformly, where 
	\[
		\pi_*(x) = \tilde \sigma \big ( F_\mathsf{c}^\mathsf{T} (x) \nabla v_*^\mathsf{T} (x) \big ).
	\]
 	\end{enumerate}	
\end{theorem}

Similarly to Remark~\ref{remark:generalization of u-AC RL problem:1}, the results are extendible to the general case where $\varphi$ and/or $\mathfrak{c}$ depends on the state $x \in \mathcal{X}$. 

\Subsection{Discounted RL with Bounded VF}{subsection:discounted RL with bounded v}

Boundedness of a VF is stronger than admissibility. Likewise, when discounted, a bounded VF can have stronger properties and statements than admissible ones. One example is continuity in the next proposition; the extension to the general cases ($\gamma = 1$ and/or $v_\pi \in \mathcal{V}_\mathsf{a}$) is by no means trivial.

\begin{proposition}
	\label{prop:continuity of VFs}
	Suppose that $f_\pi$ is locally Lipschitz and that $\gamma \in (0,1)$. 
	Then, $v_\pi$ is continuous if $v_\pi$ is bounded. 
\end{proposition}

Continuity is a necessary condition to be $\mathrm{C}^1$.
In the RL problem formulation in \secref{section:preliminaries}, we have assumed the $\mathrm{C}^1$-regularity \eqref{eq:assumption:every admissible VF is C1} and thereby continuity on every admissible VF, but no proof was provided regarding them; Proposition~\ref{prop:continuity of VFs} above bridges this gap when the VF is discounted and bounded. In this case, the boundary condition~\eqref{eq:uniqueness condition for v} is also true as follows.

\begin{proposition}
	\label{prop:boundary condition true when discounted and bounded}
	If $v:\mathcal{X} \to \mathbb{R}$ is bounded and $\gamma \in (0, 1)$, then $v$ satisfies the boundary condition~\eqref{eq:uniqueness condition for v} for any policy~$\pi$. 
\end{proposition}

Moreover, when the VF is discounted and bounded, the BE \eqref{eq:Bellman eq for v} (resp. \eqref{eq:differential BE for v}) has the unique solution $v = v_\pi$ over all bounded (resp. bounded $\mathrm{C}^1$) functions, and the boundedness is preserved under the policy improvement operation.

\begin{corollary}
	\label{cor:fundamental policy eval and imp thm for bounded v}
		Let $\gamma \in (0, 1)$ and $\pi$ be a policy. Then, 
		\begin{enumerate} 
			\item 
			if there exists a bounded function~$v$ satisfying the integral BE \eqref{eq:Bellman eq for v} or with $v \in \mathrm{C}^1$, the differential BE \eqref{eq:differential BE for v}, then $v_\pi$ is bounded (hence, admissible) and $v = v_\pi$.
			\item 
			if $v_\pi$ is bounded (hence, admissible), then so is $v_{\pi'}$ and we have $\pi \preccurlyeq \pi'$, where $\pi'$ is a maximal policy over $\pi$.
		\end{enumerate}
\end{corollary}

In fact, if the reward function $r$ is bounded, then so is the VF for \emph{any} given policy (so long as the state trajectory $t \mapsto X_t$ exists); hence the above results become stronger as follows.

\begin{assumption}
	\label{assumption:bounded R}
	$r$ is bounded and $\gamma \in (0, 1)$. 
\end{assumption}

\begin{corollary}
	\label{cor:properties when R is bounded}
	Under Assumption~\ref{assumption:bounded R}, 
	the followings hold for any given policy~$\pi$ and any maximal policy $\pi'$ over $\pi$: 
	\begin{enumerate} 
		\item $v_\pi$ and $v_{\pi'}$ are bounded (hence, admissible); $\pi \preccurlyeq \pi'$;
		\item $v_\pi$ is continuous if $f_\pi$ is locally Lipschitz;
		\item if a bounded function~$v$ satisfies the integral BE \eqref{eq:Bellman eq for v} or with $v \in \mathrm{C}^1$, the differential BE \eqref{eq:differential BE for v}, then $v = v_\pi$.
	\end{enumerate}
\end{corollary}

For a given policy~$\pi$, the VF properties in Corollary~\ref{cor:properties when R is bounded} are also true when $\gamma \in (0, 1)$ and $r_\pi$ (but not necessarily~$r$) is bounded (see and slightly modify the proof of Corollary~\ref{cor:properties when R is bounded} in \secref{appendix:proofs:case studies}). In this case (and the general cases where $\gamma \in (0, 1)$, and $v_\pi$ is bounded somehow), Proposition~\ref{prop:continuity of VFs}, Corollary~\ref{cor:fundamental policy eval and imp thm for bounded v}, and mathematical induction show that $\mathcal{T}^N v_\pi$ for any $N \in \mathbb{N}$ satisfies the properties of the VFs in Corollary~\ref{cor:properties when R is bounded}. In other words, if for the initial policy~$\pi_0$, 

\MyAssumption{$r_{\pi_0}$ (or the VF $v_{\pi_0}$) is bounded and $\gamma \in (0, 1)$}

 which is weaker than Assumption~\ref{assumption:bounded R}, then the sequences $\langle v_i \rangle$  and $\langle \pi_i \rangle$ generated by DPI or IPI satisfy: for any $i \in \mathbb{N}$, 
	\begin{enumerate}
		\item $v_i = v_{\pi_{i-1}}$,
		\item $v_{\pi_{i-1}}$ is bounded and $\pi_{i-1} \preccurlyeq \pi_i$,
		\item $v_{\pi_{i-1}}$ is continuous if $f_{\pi_{i-1}}$ is locally Lipschitz,
	\end{enumerate}
under the boundedness of each $v_i$ to ensure the boundary condition~\eqref{eq:boundary condition for PI} to be true by Proposition~\ref{prop:boundary condition true when discounted and bounded}.
	
\SetKwRepeat{Repeat}{\textbf{repeat} {\normalfont (under $\gamma \in (0, 1)$)}}{\textbf{until }}

 \begin{algorithm2e}[h!]
 \begin{spacing} {1.1}
 \DontPrintSemicolon
 \BlankLine
  \vspace{10pt}
  \nl \textbf{Initialize:} 
 	\smash{$\begin{cases} 
			\textrm{$\pi_0$, an initial policy s.t. $v_{\pi_0}$ is bounded;} \\
 			\textrm{$\Delta t > 0$, a small time step ($0 < \Delta t \ll 1$);}\!\!\!\!\!	\\
 			\textrm{$i \leftarrow 1$;} 
  	  \end{cases}$}\;
  \vspace{20pt}
 \nl \Repeat{convergence is met.}
 {
 $\;$\\
 \nl\textbf{Policy Evaluation:} given policy $\pi_{i-1}$, find a \emph{bounded$\!\!\!\!$} $\mathrm{C}^1$ function $V_i: \mathcal{X} \to \mathbb{R}$ such that for all $x \in \mathcal{X}$,   \\[7.5pt]
 	\begin{itemize} [leftmargin=0.3cm]
 		\item [] \textbf{(IPI Variant):} $V_i(x) \approx \mathbb{G}_{\pi_{i-1}}^x \big [ R_{0} + {\gamma}_\mathsf{d} \! \cdot \! V_i (X_{\! \Delta t}) \big ]$; 
			\\[7.5pt]
 		\item [] \textbf{(DPI Variant):} for $\alpha_\mathsf{d} \doteq \alpha \Delta t$ ($= - \ln \gamma_\mathsf{d}$),  
 			\\[-7.5pt]
 			\[
 				\alpha_\mathsf{d} \! \cdot \! V_i (x) = 
					h \big ( x, \pi_{i-1}(x), \Delta t \cdot \! \nabla V_i(x) \big );
			\] 
 	\end{itemize} 
 	$\;$\\[-7.5pt]
 \nl \textbf{Policy Improvement:}  find a policy $\pi_{i}$ s.t. for all $x \in \mathcal{X}$,$\!\!\!\!\!\!\!\!\!\!\!\!$ 
 \begin{align*}
 	\notag \\[-20pt]
	 	\pi_i(x) \in \smash{\Argmax_{u \in \mathcal{U}}} \big [ \, r(x,u) + \Delta t \cdot \! \nabla V_i (x) f_{\mathsf{c}}(x,u) \big ]; 
 	\notag \\[-17.5pt]	 	
 \end{align*}
 \nl $i \leftarrow i+1$;
 }
 \end{spacing}
 \caption{Variants of IPI and DPI with Bounded $\smash{v_{\pi_0}}$}
 \label{algorithm:Variants under bounded v}
\end{algorithm2e}
	
Algorithm~\ref{algorithm:Variants under bounded v} shows the respective variants of IPI and DPI when $\gamma \in (0, 1)$ and the VF $v_{\pi_0}$ w.r.t. the initial policy~$\pi_0$ is bounded. Here, the boundedness of $v_{\pi_0}$ can be made by that of $r$ or $r_{\pi_0}$. In policy evaluation, the variants of IPI and DPI solve, for $V_i \doteq v_i / \Delta t$, the discretized BE~\eqref{eq:1-step BE} and the differential BE~\eqref{eq:differential BE in DPI}, respectively; the other steps of both variants are the same and derived from their originals (Algorithms~\ref{algorithm:DPI} and \ref{algorithm:IPI}) by replacing $\eta$ and $v_i$ with the small time step $\Delta t$ and $\Delta t \cdot V_i$, respectively. Implementation examples of both variants in Algorithm~\ref{algorithm:Variants under bounded v} are given and discussed in \secref{section:simulation} with several types of (bounded) reward functions $r$ and a function approximator for $V_i$. The other types of variants (e.g., IPI with the $n$-step prediction~\eqref{eq:n-step BE}) can be also obtained by replacing the BE in policy evaluation with one of the other BEs in \secref{section:PIs} (e.g., \eqref{eq:n-step BE}). Since these variants all assume both $\gamma \in (0, 1)$ and the boundedness of the initial VF $v_{\pi_0}$, it is sufficient to find a \emph{bounded} $\mathrm{C}^1$ function $V_i$ in each policy evaluation (line~3) for holding the properties above regarding $\langle v_i \rangle$ and $\langle \pi_i \rangle$ without assuming the boundary condition~\eqref{eq:boundary condition for PI} on $v_i$ ($= \Delta t \cdot \! V_i$).

\Subsection{RL with Local Lipschitzness}{subsection:RL with local Lipschitzness}

Let 
$
\begin{cases}
		\Pi_\mathsf{Lip} \doteq \text{the set of all \emph{locally Lipschitz} policies,}
		\\[5pt]
		\mathrm{C}^1_\mathsf{Lip} \doteq 
		\{ v \in \mathrm{C}^1 : \nabla v \text{ is locally Lipschitz} \}.
\end{cases}
$\vspace{-0.5em}

In \secsref{subsection:RL with local Lipschitzness} and \ref{subsection:nonlinear optimal control}, we consider the RL problems, where 
\MyAssumption{The dynamics~$f$ and the maximal function~$u_*$ in \eqref{eq:assumption:general condition for policy improvement} are locally Lipschitz,}
and always use the notations $\pi'$ and $\pi_*$ to 
denote the maximal and HJB policies \emph{given by \eqref{eq:closed-form expression of the maximal policy under uniqueness} and \eqref{eq:closed-form expression of optimal policy}}, respectively. 

The Assumption implies continuity of $f$ and $u_*$ and ensures:
 \begin{enumerate}
 	\item $\pi'$ and $\pi_*$ are locally Lipschitz (i.e., $\pi', \pi_* \in \Pi_\mathsf{Lip}$) so long as $v_\pi$ in \eqref{eq:closed-form expression of the maximal policy under uniqueness} and $v_*$ in \eqref{eq:closed-form expression of optimal policy} are $\mathrm{C}^1_\mathsf{Lip}$, respectively;
 	\item the dynamics~$f_\pi$ under $\pi \in \Pi_\mathsf{Lip}$ is locally Lipschitz, and thereby the state trajectory $t \mapsto \mathbb{G}_\pi^x[ X_t ]$ for each $x \in \mathcal{X}$ is uniquely defined and $\mathrm{C}^1$ over the maximal existence interval $[0, t_\mathsf{max}(x; \pi)) \subseteq \mathbb{T}$ (see \citealp[Section~3.1 and Theorem~3.1 therein]{Khalil2002}). 
 \end{enumerate}
Here, $t_\mathsf{max}(x; \pi) \in (0, \infty]$ is defined for and depends on both initial state $x$ and $\pi \in \Pi_\mathsf{Lip}$; whenever $t_\mathsf{max}(x;\pi) < \infty$, 
\[
	\mathbb{G}_\pi^x ( \| X_t \|) \to \infty \textrm{ as }
	t \to t_\mathsf{max}(x; \pi).
\]
To circumvent this finite-time explosion issue, 
we set $v_\pi(x)$ to ``$- \infty$'' whenever $t_\mathsf{max}(x; \pi)$ is finite, that is, redefine $v_\pi$ as 
\begin{equation}
	v_\pi(x) \doteq 
		\begin{cases}
			\mathbb{G}_\pi^x 
			\bigg [ {\displaystyle \int_0^\infty} \gamma^t \cdot R_t \, dt \bigg ] 
			\textrm{ if } t_\mathsf{max}(x; \pi) = \infty,\!\!\!\!\!\!\!\!
			\\[12.5pt]
			\; - \infty, \quad \textrm{ otherwise.}
		\end{cases}
	\label{eq:def:VF under local Lipschitzness}
\end{equation}
Here, existence and uniqueness of the state trajectories were not assumed; $t_\mathsf{max}(\, \cdot \,; \pi)$ and thus $v_\pi$ in \eqref{eq:def:VF under local Lipschitzness} are well-defined as long as $\pi \in \Pi_\mathsf{Lip}$. Hence, with slight abuse of notation, we restrict the admissible sets $\Pi_\mathsf{a}$ and $\mathcal{V}_\mathsf{a}$ by redefining them as 
\[
	\begin{split}
		\Pi_\mathsf{a} &\doteq \{ \pi \in \Pi_\mathsf{Lip}: v_\pi(x) \textrm{ is finite for all } x\in \mathcal{X} \}, 
		\\[2.5pt]
		\mathcal{V}_\mathsf{a} &\doteq \{ v_\pi : \pi \in \Pi_\mathsf{a} \}.
	\end{split}
\]
Note that for each $x \in \mathcal{X}$, the value $v_\pi(x)$ is finite and the state trajectory $t \mapsto \mathbb{G}_\pi^x [ X_t ]$ is defined uniquely and $\mathrm{C}^1$ over the entire time interval $\mathbb{T}$ if $\pi \in \Pi_\mathsf{a}$ (or equivalently, if $v_\pi \in \mathcal{V}_\mathsf{a}$). Here, the global existence of the unique state trajectories was assumed in the general RL problem formulated in \secref{section:preliminaries} but now is encapsulated by admissibility. 

In what follows, we provide the policy improvement theorem extended from Theorem~\ref{thm:policy improvement thm}, without assuming any existence and uniqueness of the state trajectories, but under

\MyAssumption{$\mathcal{V}_\mathsf{a} \subset \smash{\mathrm{C}^1_\mathsf{Lip}}$}

to ensure the maximal policy $\pi' \in \Pi_\mathsf{Lip}$ whenever $\pi \in \Pi_\mathsf{a}$.

\begin{theorem} [Policy Improvement]
	\label{thm:policy improvement thm in locally Lipschitz case}
	If there exist a compact subset $\Omega \subset \mathcal{X}$ and $\mathcal{K}_\infty$ functions $\rho_1$, $\rho_2$ such that for a policy $\pi \in \Pi_\mathsf{a}$,
	\begin{align*}
		\rho_1 ( \|x\|_\Omega ) \leq \widebar v - v_\pi(x) \leq \rho_2 ( \|x\|_\Omega ) \quad \forall x \in \mathcal{X},
	\end{align*}
	where $\|x\|_\Omega \doteq \inf_{y \in \Omega} \| x - y \|$, then $\pi' \in \Pi_\mathsf{a}$ and $\pi \preccurlyeq \pi'$.
\end{theorem}
	
\Subsection{Nonlinear Optimal Control}{subsection:nonlinear optimal control}

The objective of optimal control is to stabilize the system~\eqref{eq:controlled system} w.r.t. a given equilibrium point $(x_\mathsf{e}, u_\mathsf{e})$ while minimizing a given cost functional. Here, any point in $\mathcal{X} \times \mathcal{U}$ such that ${\dot x}_\mathsf{e} = f(x_\mathsf{e}, u_\mathsf{e}) \equiv 0$ is called an equilibrium point~$(x_\mathsf{e}, u_\mathsf{e})$; it can be transformed to $(0, 0)$ and thus let $(x_\mathsf{e}, u_\mathsf{e}) = (0, 0)$ without loss of generality \citep{Khalil2002} and assume that $f(0, 0) = 0$. Note that if a policy~$\pi$ satisfies $\pi(0) = 0$, then we have $0 = f_\pi(0)$, i.e.,  $x_\mathsf{e} = 0$ is an equilibrium point of the system~\eqref{eq:controlled system} under~$\pi$. 

\vspace{-0.5em}

The optimal control framework in this subsection is a particular case of the locally Lipschitz RL problem in \secref{subsection:RL with local Lipschitzness} above. Hence, we impose the same assumptions on it: the local Lipschitzness of  $f$ and $u_*$, with $\pi'$ and $\pi_*$ denoting the respective policies \emph{given by \eqref{eq:closed-form expression of the maximal policy under uniqueness} and \eqref{eq:closed-form expression of optimal policy}}, the inclusion $\mathcal{V}_\mathsf{a} \subset \mathrm{C}^1_\mathsf{Lip}$, and the extended definition \eqref{eq:def:VF under local Lipschitzness} of the VF $v_\pi$ for $\pi \in \Pi_\mathsf{Lip}$. 

\vspace{-0.5em}

On the other hand, we define a class of policies $\Pi_0$ as 
\begin{align*}
	\Pi_0 \doteq \{ \pi \in \Pi_\mathsf{Lip}: \pi(0) = 0 \}
\end{align*}
and, with slight abuse of notation, redefine $\Pi_\mathsf{a}$ and $\mathcal{V}_\mathsf{a}$ by 
\[
		\Pi_\mathsf{a} \doteq \{ \pi \in \Pi_0: 
		v_\pi(x) \text{ is finite for all } x \in \mathcal{X} \}
\]
and $\mathcal{V}_\mathsf{a} \doteq \{ v_\pi : \pi \in \Pi_\mathsf{a} \}$. Here, we have merely added the condition $\pi(0) = 0$ into the definitions of $\Pi_\mathsf{a}$ and $\mathcal{V}_\mathsf{a}$ in \secref{subsection:RL with local Lipschitzness}. With these notations, $x_\mathsf{e} = 0$ comes to be an equilibrium point of the system \eqref{eq:controlled system} under $\pi \in \Pi_0$ ($\supseteq \Pi_\mathsf{a}$).

\vspace{-0.5em}

Similarly to \secref{subsection:RL with local Lipschitzness}, this subsection does not assume existence and uniqueness of the state trajectories; $\pi \in \Pi_\mathsf{a}$ (or $v_\pi \in \mathcal{V}_\mathsf{a}$) ensures: for every $x\in \mathcal{X}$, the state trajectory $t \mapsto \mathbb{G}_\pi^x [X_t]$ is uniquely defined and $\mathrm{C}^1$ over $\mathbb{T}$. In addition, the boundary conditions \eqref{eq:uniqueness condition for v} and \eqref{eq:boundary condition for PI} are not assumed but either proven to be true or replaced by their sufficient ones shown in \secref{appendix:sec:replacement of bdy condition} (e.g., see Theorem~\ref{thm:uniqueness condition in optimal control}).

\vspace{-0.5em}

Whenever necessary, we use the cost functions $c \doteq -r$ and $c_\pi \doteq - r_\pi$, the cost VF $J_\pi \doteq - v_\pi$, and $J_* \doteq - v_*$, rather than $-r$, $- r_\pi$, $-v_\pi$, and $- v_*$, respectively, for simplicity and consistency to optimal control conventions; the cost at time $t \in \mathbb{T}$ is denoted by $C_t \doteq c(X_t, U_t) = - R_t$. 

\vspace{-0.5em}

We consider a positive definite cost function $c$, i.e., assume
\begin{equation}
	c(x, u) > 0  \quad \forall (x, u) \neq (0, 0),	\textrm{ and } c(0, 0) = 0.
	\label{eq:n.d. R in optimal control}
\end{equation}
Then, by~\eqref{eq:n.d. R in optimal control} and the definition, the value $J_\pi(x)$ is always restricted to $[0, \infty]$ and, similarly to \eqref{eq:def:VF under local Lipschitzness}, $J_\pi(x) = \infty$ whenever $t_\mathsf{max}(x; \pi) < \infty$; otherwise, 
$J_\pi(x) = \mathbb{G}_\pi^x  \big [ \int_0^\infty \! \gamma^t C_t \, dt \big ]$.

\vspace{-0.5em}

\begin{lemma}
	\label{lemma:positive definiteness of c_pi}
	$c_\pi$ for $\pi \in \Pi_0$ is positive definite.
\end{lemma}

\vspace{-0.5em}

\begin{lemma}
\label{lemma:Lyapunov v}
	Let $\pi \in \Pi_\mathsf{a}$. Then, \textbf{\emph{a.}} $J_\pi$ is positive definite; \textbf{\emph{b.}} $x \mapsto {\dot J}_\pi(x, \pi(x))$ is negative semidefinite iff 
		\begin{equation}
			\alpha J_\pi \leqslant c_\pi
			\label{eq:condition for stability}
		\end{equation}	
		and \textbf{\emph{c.}} $x \mapsto {\dot J}_\pi(x, \pi(x))$ is negative definite iff 
		\begin{equation}
			\alpha J_\pi(x) < c_\pi(x) \qquad \forall x \in \mathcal{X} \setminus \{0\}.
			\label{eq:condition for asymp stability}
		\end{equation}
\end{lemma}

In what follows, 
we assume that $c_\pi$ for any $\pi \in \Pi_0$ is radially nonvanishing\footnote{This assumption excludes any function $c_\pi$ such that as $r \to \infty$, $\inf_{\|x\| \geq r} c_\pi(x) \to 0$ (e.g., $c_\pi(x) = x^2 \exp{(-x^2)}$) and is used in Theorem~\ref{thm:stability of optimal control} for proving global asymptotic stability for $\gamma =1$.} (\secref{appendices:notations:4}).
Given the conditions in Lemma~\ref{lemma:Lyapunov v}, $J_\pi$ is, in fact, a Lyapunov function \citep{Khalil2002} for the system ${\dot X}_t = f_\pi(X_t)$ as shown in the following theorem. 

\begin{theorem}
	\label{thm:stability of optimal control}
	The equilibrium point $x_\mathsf{e} = 0$ of dynamics $f_\pi$ under $\pi \in \Pi_\mathsf{a}$ is \mbox{stable} if \eqref{eq:condition for stability} holds, asymptotically stable if \eqref{eq:condition for asymp stability} is true, 
	and globally asymptotically stable if $\gamma = 1$ or, in addition to \eqref{eq:condition for asymp stability}, $J_\pi$ is radially unbounded.
\end{theorem}
	 
\begin{remark}
	\label{remark:admissibility implies asymptotic stability}
Whenever $\gamma = 1$, (i) \eqref{eq:condition for asymp stability} is true since $\alpha = 0$ and $c_\pi$ is positive definite by Lemma~\ref{lemma:positive definiteness of c_pi}; (ii) admissibility directly implies global asymptotic stability by Theorem~\ref{thm:stability of optimal control} 
(here, the radial unboundedness of $J_\pi$ is not assumed!).
\end{remark}


Next, we show global attractiveness (hence, global asymptotic stability) ensures uniqueness of the solution to the BEs.

\begin{definition}
	$x_\mathsf{e}$ is globally attractive under $\pi$ iff 
	\[
		t_\mathsf{max}(x; \pi) = \infty \text{ and } \smash{\lim_{t \to \infty}} \mathbb{G}_\pi^x [ X_t ] = x_\mathsf{e} \qquad \forall x \in \mathcal{X}.
	\]	
\end{definition}

\begin{theorem} [Policy Evaluation]
	\label{thm:uniqueness under gas}
	Let \text{$x_\mathsf{e} = 0$ under $\pi \in \Pi_0$}\\ be globally attractive. If there exists a function $v: \mathcal{X} \to \mathbb{R}$ s.t. $v$ is continuous at~$0$, $v(0) = 0$, and the BE \eqref{eq:Bellman eq for v} or, with $v \in \mathrm{C}^1$, the BE~\eqref{eq:differential BE for v} holds, then $\pi \in \Pi_\mathsf{a}$ and $v = v_\pi$.
\end{theorem}

The uniqueness of the solution to the BE can be also given under other conditions. When discounted, it has a more general condition than both \eqref{eq:condition for stability} and \eqref{eq:condition for asymp stability} for stability as well as contains the cases where $x_\mathsf{e} = 0$ is not necessarily (globally) attractive, and the state trajectories could even diverge. 

\begin{theorem} [Policy Evaluation]
	\label{thm:uniqueness condition in optimal control} 
	Let $J$ ($\doteq v$) be positive definite and $\kappa \cdot J \leqslant c_\pi$ for a policy $\pi \in \Pi_0$ and a constant $\kappa > 0$. Then, $\pi \in \Pi_\mathsf{a}$ and $v = v_\pi$ if either of \emph{\textbf{a}} or \emph{\textbf{b}} below is true.
	\begin{enumerate}
		\item [\emph{\textbf{a.}}] $v$ is $\mathrm{C}^1$, radially unbounded, and satisfies the BE~\eqref{eq:differential BE for v} or alternatively, the BE~\eqref{eq:Bellman eq for v} for arbitrary small $\eta > 0$;\\[-5pt]
		\item [\emph{\textbf{b.}}] $v$ satisfies the BE~\eqref{eq:Bellman eq for v} for a fixed $\eta > 0$, $c_\pi$ is radially unbounded, and there exist a function $\zeta: \mathcal{X} \to \mathbb{R}$ and $\ushort \alpha < \alpha$, possibly depending on $\pi$, s.t. for each $x \in \mathcal{X}$, 
			\begin{align}
					\mathbb{G}_\pi^x [C_t] \leq \zeta(x) \exp(\ushort \alpha \hspace{0.1em} t)
					\quad \forall t \in [0, t_\mathsf{max}(x; \pi)).
					\label{eq:optimal control:exponential increasing condition for admissibility}
			\end{align}	
	\end{enumerate}
\end{theorem}

The policy improvement theorem in \secref{subsection:RL with local Lipschitzness}, i.e., Theorem \ref{thm:policy improvement thm in locally Lipschitz case}, can be also extended as follows.

\begin{theorem} [Policy Improvement]
	\label{thm:policy improvement thm for optimal control}
	Let $\pi \in \Pi_\mathsf{a}$ and $J_\pi$ is radially unbounded. Then $\pi' \in \Pi_\mathsf{a}$ and $J_{\pi'} \leqslant J_{\pi}$.
\end{theorem}

From the theory and discussions above, we propose the following three conditions for the PI methods in the optimal control framework: for all $i \in \mathbb{N}$ and $J_i \doteq - v_i$, 
\begin{enumerate}
	\item [\textsf{(A)}] $\pi_0 \in \Pi_\mathsf{a}$,
	\item [\textsf{(B)}] $J_i \in \mathrm{C}^1_\mathsf{Lip}$ is positive definite and radially unbounded,
	\item [\textsf{(C)}] if $\gamma \neq 1$,  
		$
		\begin{cases}
			\text{$x_e = 0$ under $\pi_{i-1}$ is globally attractive, }\\[2.5pt]
			\text{or there exists $\kappa_i > 0$ s.t. $\kappa_i \! \cdot \! J_i \leqslant c_{\pi_{i-1}}$.}
		\end{cases}
		$
\end{enumerate}
Those three conditions are devised in order to run PI (for IPI, together with \textsf{(D)} or \textsf{(E)} below), without assuming the existence of unique state trajectories and the boundary condition~\eqref{eq:boundary condition for PI}. Here, \textsf{(C)} is imposed only when $\gamma \neq 1$ ($\alpha > 0$),  
in which case, if $\kappa_i < \alpha$, then the inequality in \textsf{(C)} is weaker than both of the stability conditions $\alpha J_i \leqslant c_{\pi_{i-1}}$ and 
\begin{align}
	\alpha J_i(x) < c_{\pi_{i-1}}(x) \qquad 
	\forall x \neq \mathcal{X} \setminus \{0\}
	\label{eq:gas condition in optimal control}
\end{align}
	that correspond to \eqref{eq:condition for stability} and \eqref{eq:condition for asymp stability}, respectively. 
	 
For running IPI under discounting $\gamma \in (0, 1)$, we impose an additional condition on each $\pi_{i-1}$:
\begin{enumerate}
	\item [\textsf{(D)}] if $\gamma \neq 1$, then \textbf{a.} $c_{\pi_{i-1}}$ is radially unbounded; \textbf{b.} there are
		$\ushort \alpha_i \in [0,  \alpha)$ and a function $\zeta_i$ s.t. \eqref{eq:optimal control:exponential increasing condition for admissibility} holds $\forall x \in \mathcal{X}$.
\end{enumerate}
Here, \textsf{(D\text\rm{\bf a})} is true if $x \mapsto c(x, u)$ is radially unbounded; \textsf{(D\textrm{\bf b})} is true for any policy~$\pi_{i - 1}$ that makes every state trajectory bounded or even diverge  exponentially with the rate smaller than $\alpha$. For instance, if $x_\mathsf{e} = 0$ under $\pi_{i-1}$ is globally attractive (so that \textsf{(C)} is true) or state trajectories are globally bounded, then \textsf{(D\textrm{\bf b})} is always valid with 
\[
	\ushort \alpha_i = 0 \textrm{ and } \zeta_i(x) = \textstyle \inf_{t \in \mathbb{T}\;} \mathbb{G}_{\pi_{i-1}}^x [C_t] < \infty,
\]
where $\zeta_i(x)$ is finite by boundedness of $t \mapsto \mathbb{G}_\pi^x [X_t]$ and continuity of both $c$ and $t \mapsto \mathbb{G}_\pi^x [X_t]$. 
Another condition for discounted IPI that can replace the condition \textsf{(D)} is 
\begin{enumerate}
	\item [\textsf{(E)}] if $\gamma \neq 1$, the BE~\eqref{eq:integral BE in IPI} holds \emph{for arbitrary small $\eta > 0$}.
\end{enumerate}
In practice, it is impossible to solve the BE~\eqref{eq:integral BE in IPI} for infinitely many $\eta$'s; 
one best practice for \textsf{(E)} is to implement IPI with \emph{one sufficiently small $\eta > 0$}. We also note that when $\gamma = 1$, \textsf{(C)}--\textsf{(E)} become irrelevant --- in this case, only \textsf{(A)} and \textsf{(B)} are required to run both IPI and DPI.

\begin{theorem}
	\label{thm:fundamental properties in optimal control}
	Under \emph{\textsf{(A)}}--\emph{\textsf{(C)}} (for IPI, with \emph{\textsf{(D)}} or \emph{\textsf{(E)}}), 
	\begin{enumerate}
		\item $\pi_{i-1} \! \in \Pi_\mathsf{a}$ and $J_i = J_{\pi_{i-1}} \geqslant J_{\pi_i}$ for all $i \in \mathbb{N}$;
		\item $x_\mathsf{e} = 0$ under $\pi_{i-1}$ is globally asymptotically stable (hence, globally attractive) if \eqref{eq:gas condition in optimal control} is true (or if $\gamma = 1$).
	\end{enumerate}
\end{theorem}

Without assuming the boundary condition~\eqref{eq:boundary condition for PI} and existence of unique state trajectories, the other properties in \secref{section:fundamental properties of PI} can be also extended under the above conditions \textsf{(A)}--\textsf{(C)} (for IPI, together with \textsf{(D)} or \textsf{(E)}), by following the same proofs in \secref{section:fundamental properties of PI}, but with Theorem~\ref{thm:fundamental properties} therein replaced by Theorem~\ref{thm:fundamental properties in optimal control}.

Note that the radial unboundedness of $J_i$ in \textsf{(B)} makes sense only when $J_{\pi_{i-1}}$ or the optimal cost VF $J_*$ is radially unbounded. For the latter case, it is guaranteed that $J_\pi$ for every $\pi \in \Pi_\mathsf{a}$ is radially unbounded by optimality $0 \leqslant J_* \leqslant J_\pi$.

Limitations also exist. First, it is difficult to check {\textsf{(C)}--\textsf{(E)}} and that $J_{\pi_{i-1}}$ is radially unbounded; $J_*$ is unknown until we have found at the end. Secondly, the results cannot be applied to the locally admissible cases, where the VF (equivalently $J_\pi$) is finite only around the equilibrium point $x_\mathsf{e} = 0$ locally, not globally over $\mathcal{X}$. 
Lastly, not easy to verify in general is the local Lipschitzness assumptions on $u_*$ and $\nabla J_\pi$ which are necessary for $\pi' \in \Pi_\mathsf{Lip}$. An example that is free from these limitations is the LQR (see \secref{subsection:LQR}). 

\begin{remark}
	\label{remark:admissibility}
	This article is the first to define admissibility \mbox{without asymptotic stability} to the best authors' knowledge. This concept can be broadly applied, e.g., to the discounted LQR cases in \secref{subsection:LQR} where the system may not be stable under an admissible policy due to $\gamma \in (0, 1)$. In fact, when $\gamma = 1$, admissibility of a policy~$\pi$ (i.e., $\pi \in \Pi_\mathsf{a}$) implies global asymptotic stability under $\pi$, with $J_\pi$ served as a Lyapuonv function, as discussed in Remark~\ref{remark:admissibility implies asymptotic stability}. This reveals that asymptotic stability can be excluded from the definition of admissibility, even in the existing optimal control frameworks (as long as the VF is $\mathrm{C}^1$). We also believe that our concept of admissibility can be generalized even when $v_\pi$ (or equivalently, $J_\pi$) is locally finite around the equilibrium $x_\mathsf{e} = 0$ (i.e., locally admissible), not globally.
\end{remark}

\MyRemark{If the dynamics~$f$ is non-affine, then the cost function $c$ has to be properly designed (e.g., by the techniques introduced in \secsref{subsubsection:case II} and  \ref{subsubsection:case III}) to avoid the pathological Hamiltonian discussed in Remark~\ref{remark:design of reward} and \secref{subsubsection:case II}. Note that as shown in \secref{appendix:subsection:a pathological ex}, such a pathological phenomenon can happen even when $c$ is positive definite (and quadratic when unconstrained), which is a typical choice in optimal control.
	}

\begin{figure*} [t!]
	\centering
	\subfloat[Case 1: concave Hamiltonian with bounded reward --- DPI]{\includegraphics[width=\columnwidth, height=2.9cm]{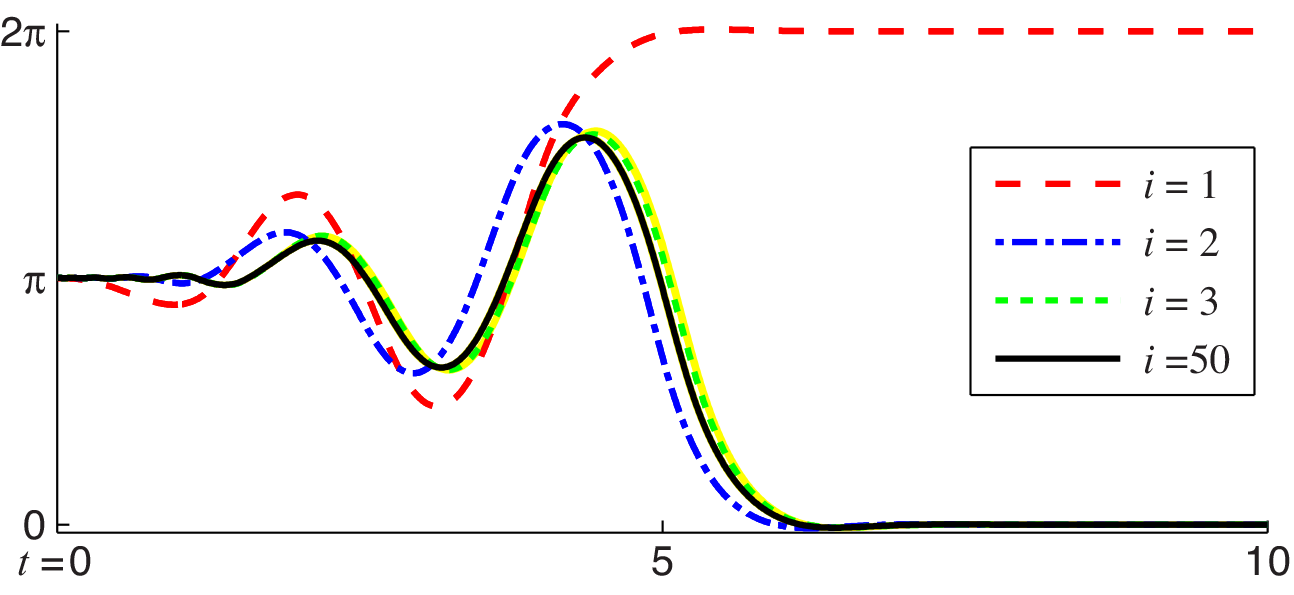}} \hfill
	\subfloat[Case 2: optimal control --- DPI]{\includegraphics[width=\columnwidth, height=2.9cm]{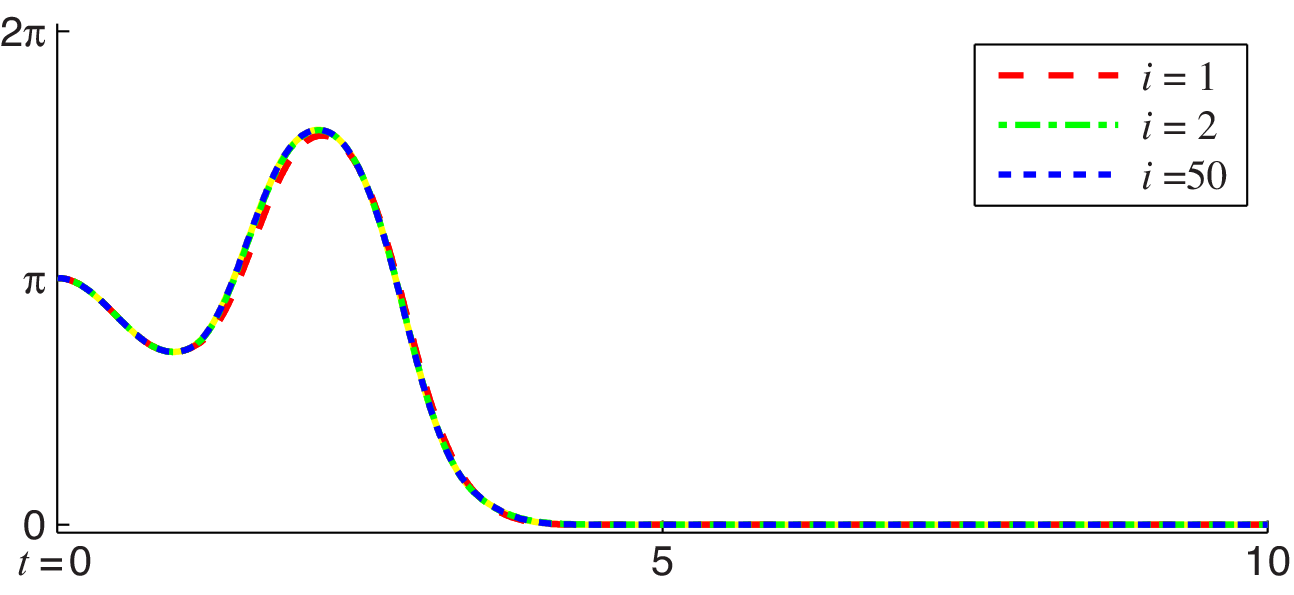}} \\
	\subfloat[Case 1: concave Hamiltonian with bounded reward ---IPI]{\includegraphics[width=\columnwidth, height=2.9cm]{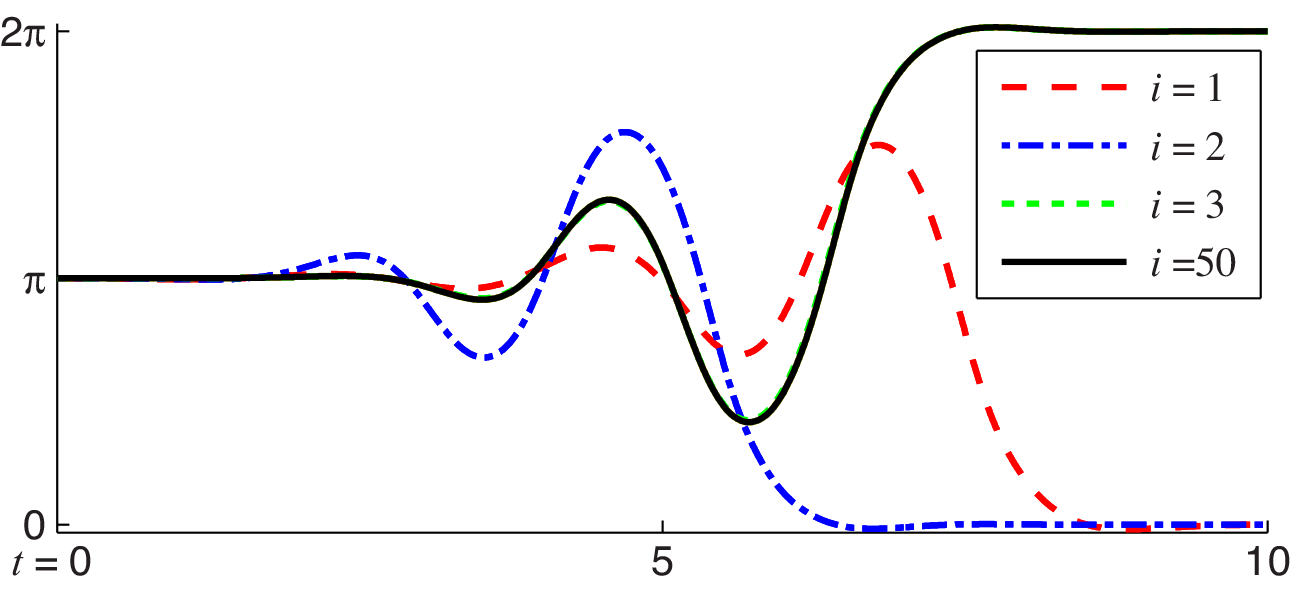}}  \hfill
	\subfloat[Case 3: bang-bang control --- IPI with $r(x, u) = \cos x_1$]{\includegraphics[width=\columnwidth, height=2.9cm]{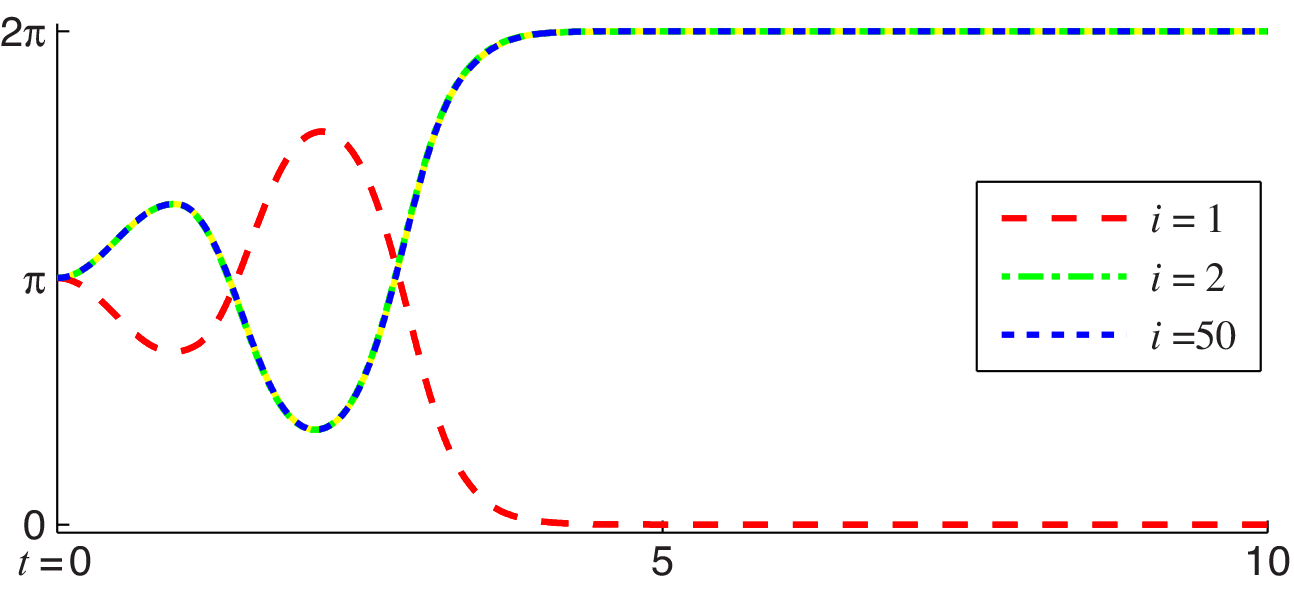}} \\
	\subfloat[Case 4: bang-bang control with binary reward -- DPI]{\includegraphics[width=\columnwidth, height=2.9cm]{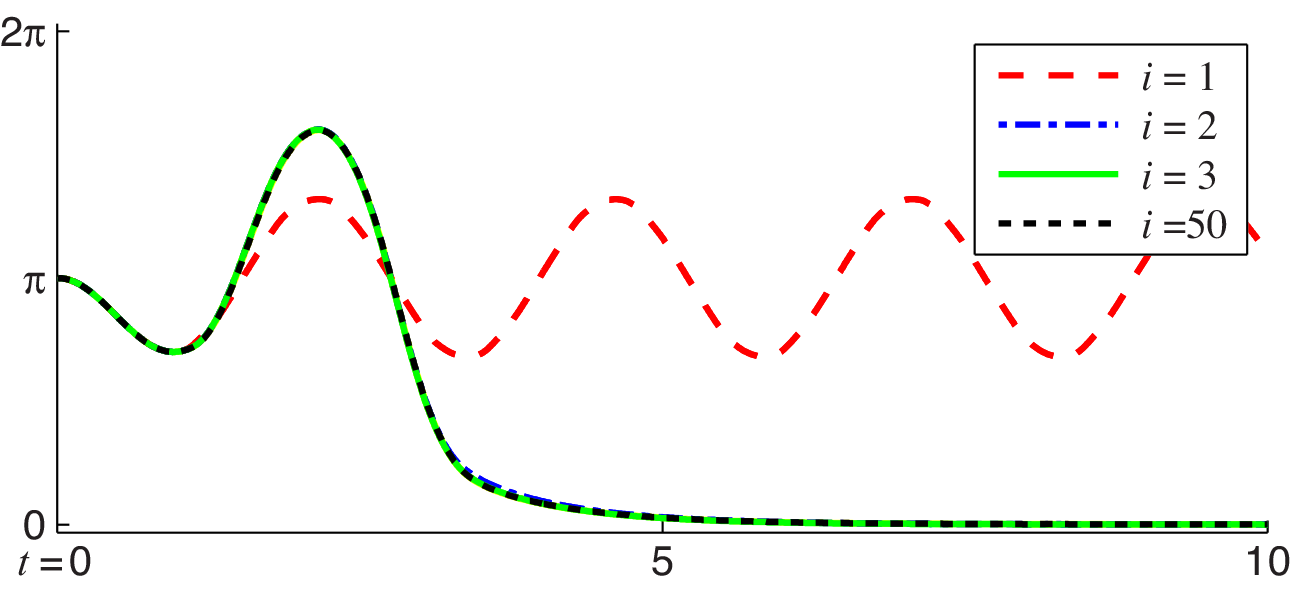}} \hfill
	\subfloat[Case 4: bang-bang control with binary reward -- IPI]{\includegraphics[width=\columnwidth, height=2.9cm]{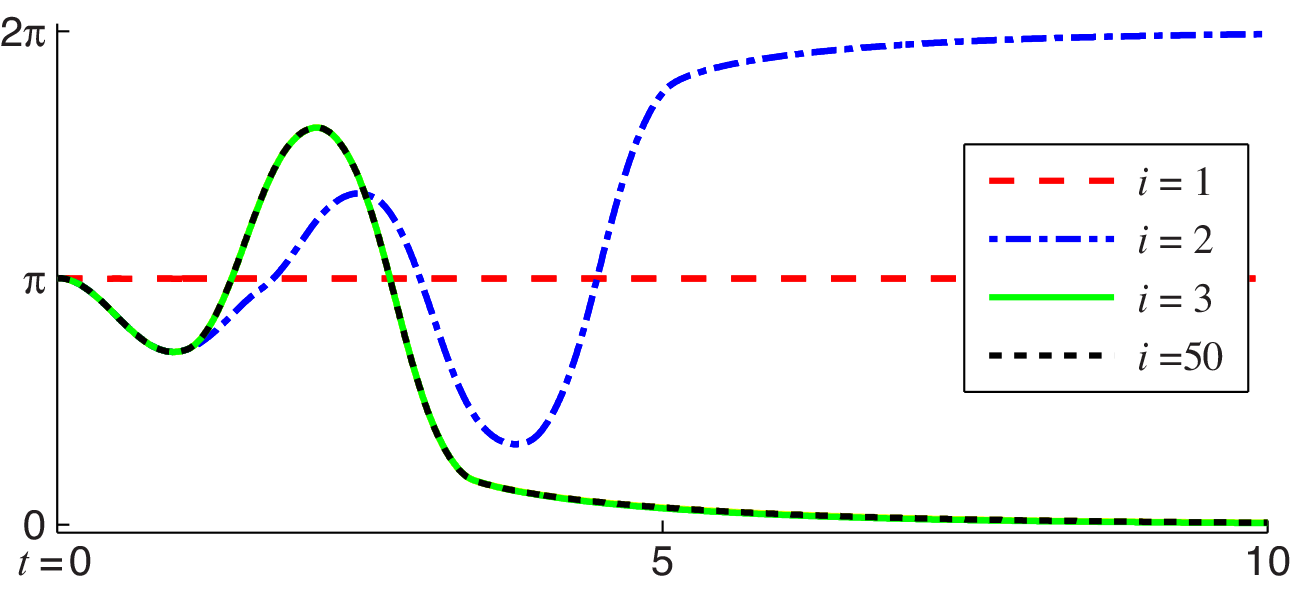}}
	\caption{The trajectories of the pendulum angular-position $\vartheta_t$ generated by the policies obtained during and after the PIs for each case study. All of the trajectories start from $X_0 = (\pi, 0)$, and the yellow regions correspond to those $\vartheta_t$-trajectories at iterations $i = (3),4,5,\cdots,49$.}
	\label{fig:pos trj}
\end{figure*}

\begin{figure*} [t!]
\centering
\subfloat[${\hat V}_{50}$ in Case 1 --- DPI]{\includegraphics[width=.245\linewidth]{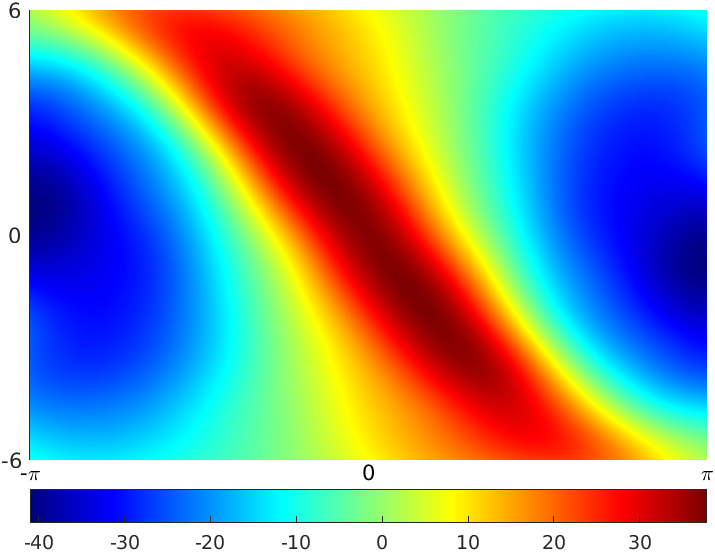}}
\subfloat[${\hat V}_{50}$ in Case 1 --- IPI]{\includegraphics[width=.245\linewidth]{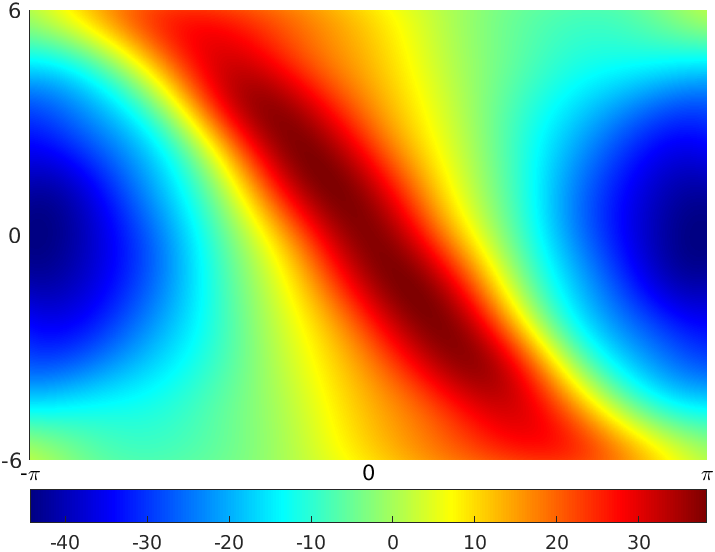}}
\subfloat[${\hat \pi}_{50}$ in Case 1 --- DPI]{\includegraphics[width=.245\linewidth]{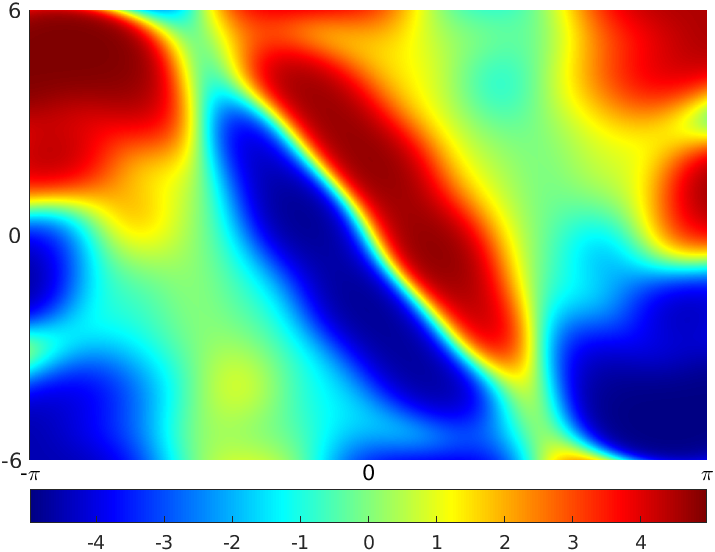}}
\subfloat[${\hat \pi}_{50}$ in Case 1 --- IPI]{\includegraphics[width=.245\linewidth]{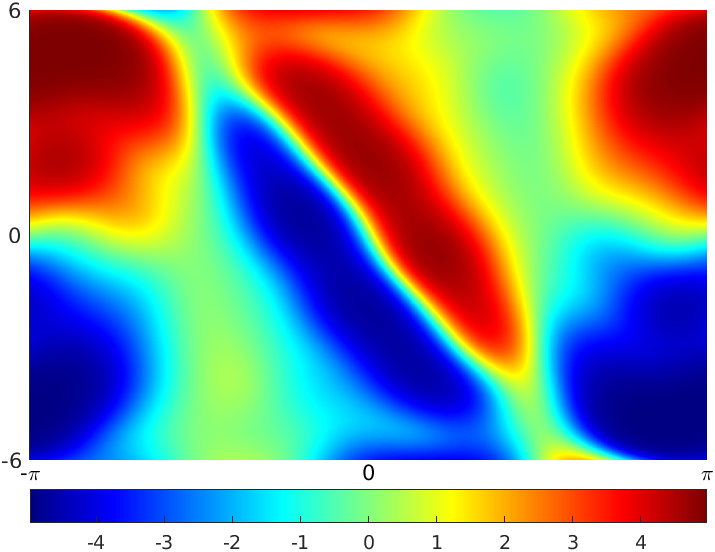}} \\[-0.35em]
\subfloat[${\hat V}_{50}$ in Case 2 --- DPI ]{\includegraphics[width=.245\linewidth]{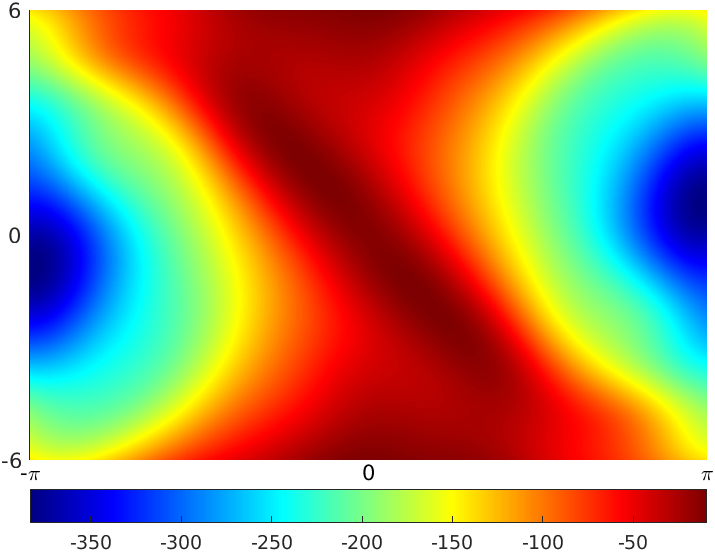}}
\subfloat[${\hat V}_{50}$ in Case 3 --- IPI w/ \eqref{eq:sim:reward for optimal bang-bang control}]{\includegraphics[width=.245\linewidth]{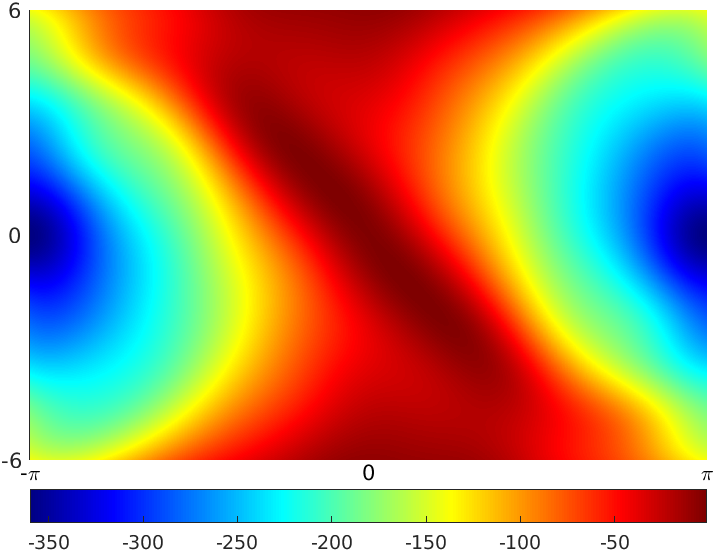}}
\subfloat[${\hat \pi}_{50}$ in Case 2 --- DPI]{\includegraphics[width=.245\linewidth]{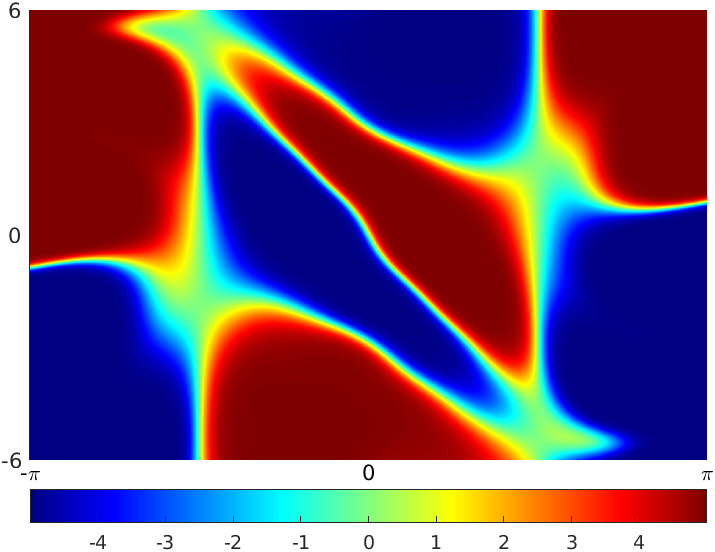}}
\subfloat[${\hat \pi}_{50}$ in Case 3 --- IPI w/ \eqref{eq:sim:reward for optimal bang-bang control}]{\includegraphics[width=.245\linewidth]{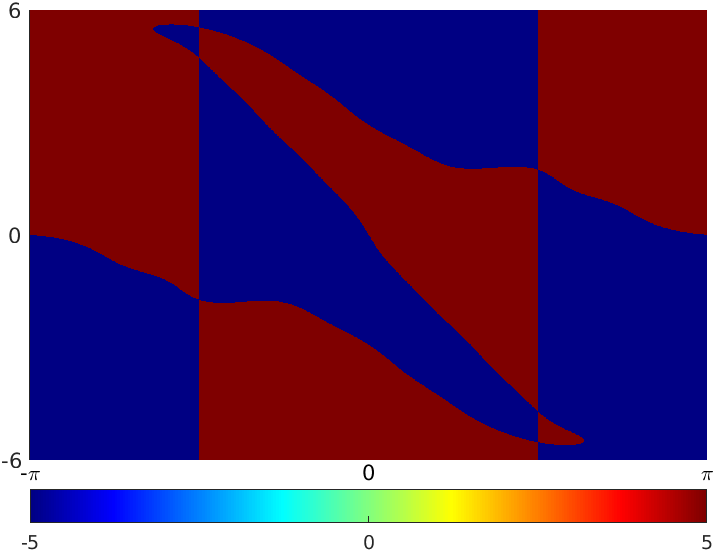}}\\[-0.35em]
\subfloat[${\hat V}_{50}$ in Case 4 --- DPI]{\includegraphics[width=.245\linewidth]{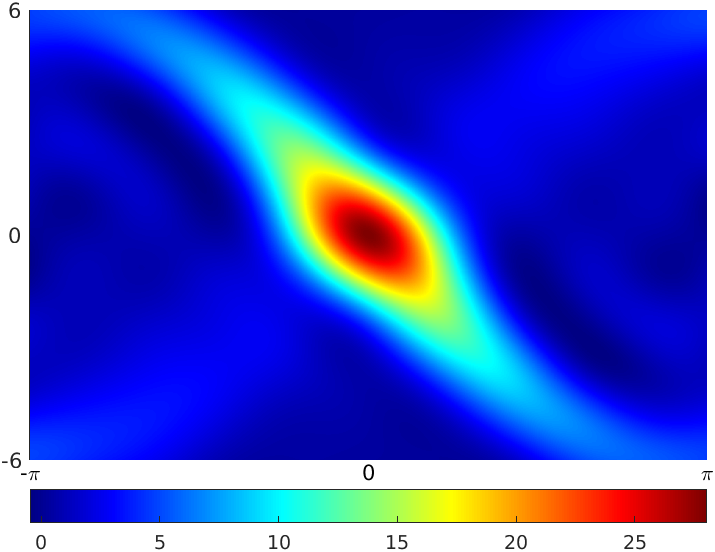}}
\subfloat[${\hat V}_{50}$ in Case 4 --- IPI]{\includegraphics[width=.245\linewidth]{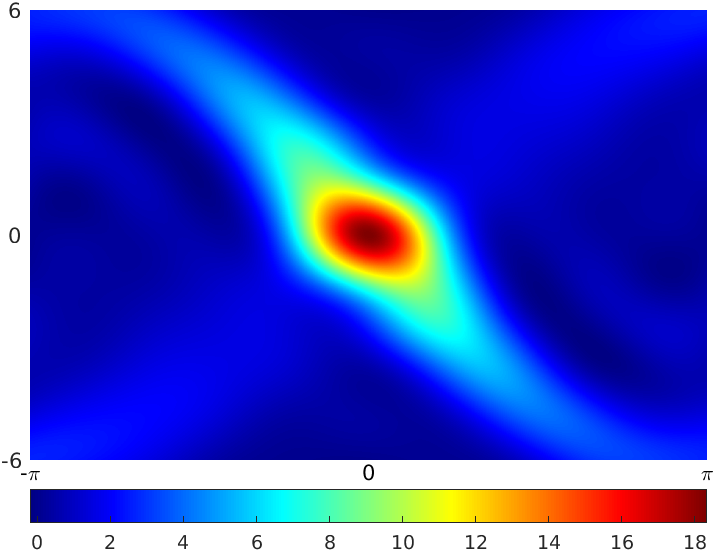}}
\subfloat[${\hat \pi}_{50}$ in Case 4 --- DPI]{\includegraphics[width=.245\linewidth]{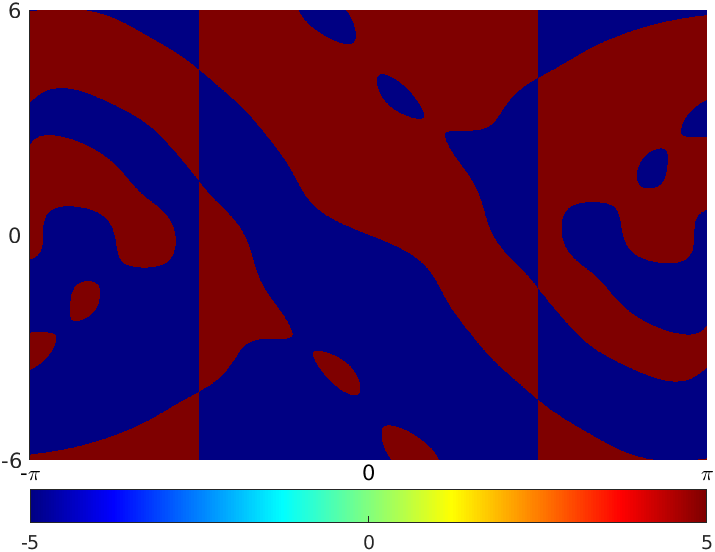}}
\stackunder[6.6pt]{\includegraphics[width=.245\linewidth]{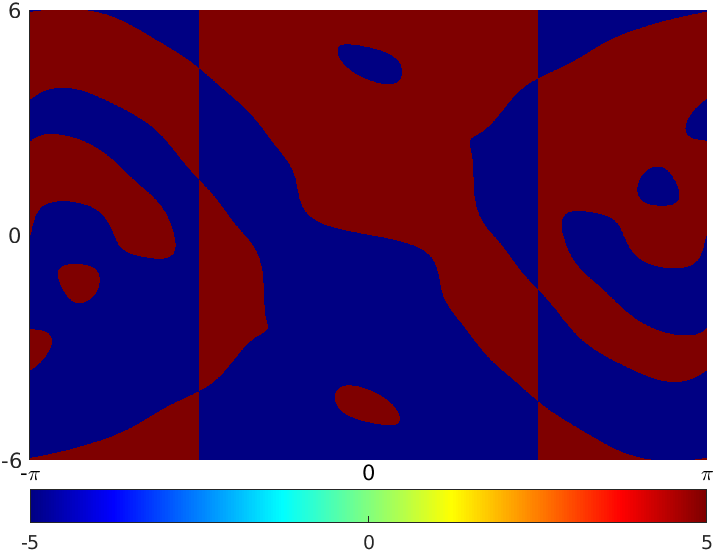}}{($\ell$) ${\hat \pi}_{50}$ in Case 4 --- IPI}
\caption{The optimal value function ${\hat V}_{50}(x) = V(x ; \theta_i^*)|_{i=50}$ (left sides) and the optimal policy ${\hat \pi}_{50}(x) = \pi(x ; \theta_i^*)|_{i=50}$ (right sides), estimated by DPI and IPI variants over $\Omega$. The horizontal and vertical axes correspond to $x_1$ ($= \vartheta$) and $x_2$ ($={\dot \vartheta}$), respectively.}
\label{fig:vi and pi}
\end{figure*}

\Section{Inverted-Pendulum Simulation Examples}{section:simulation}

To support the theory and further investigate the proposed PI methods, we simulate the variants of DPI and IPI shown in Algorithm~\ref{algorithm:Variants under bounded v} applied to an inverted-pendulum model:
\[
	{\ddot \vartheta}_t = - 0.01 {\dot \vartheta}_t + 9.8 \sin \vartheta_t - U_t \cos \vartheta_t, 
\]
where $\vartheta_t \in \mathbb{R}$ and $U_t \in \mathcal{U}$ are the angular position of and the external torque input to the pendulum at time $t$, respectively; the action space is given by $\mathcal{U} = [- u_\mathsf{max}, u_\mathsf{max} ] \subset \mathbb{R}$, with the torque limit $u_{\mathsf{max}} = 5$ [N$\cdot$m]. Letting $X_t \doteq [\, \vartheta_t \;\, {\dot \vartheta}_t \,]^{\mathsf{T}}$, then the dynamics can be expressed as \eqref{eq:controlled system} and \eqref{eq:input-affine dynamics} with 
\[
	f_{\mathsf{d}}(x) = \begin{bmatrix} x_2 \\  9.8 \, \sin x_1 - 0.01 x_2  \end{bmatrix} \textrm{ and }
	F_{\mathsf{c}}(x) = \begin{bmatrix}  0  \\  - \cos x_1 \end{bmatrix},
\]
where $x = [\, x_1 \,\; x_2 \,]^{\mathsf{T}} \in \mathcal{X}$ ($= \mathbb{R}^2$). In the simulations, we set the discount factor $\gamma = 0.1$ and the time step $\Delta t = 10$~[ms]; the zero initial policy $\pi_0(x) \equiv 0$ is employed.

The solution $V_i$ of the policy evaluation at each iteration $i$ is represented by a linear function approximator $V$ as 
\begin{equation}
	V_i(x) \approx V(x; \theta_i) \doteq \theta_{i}^\mathsf{T} \phi(x),  
	\label{eq:RBFN approximation of vi}
\end{equation}
for its weights $\theta_i \in \mathbb{R}^L$ and features $\phi: \mathcal{X} \to \mathbb{R}^L$, with $L = 121$. Each policy evaluation determines $\theta_i$ by the least-squares solution $\theta_i^*$ minimizing the Bellman errors over the set of initial states uniformly distributed as the ($N \times M$)-grid points over the region $\Omega = [- \pi, \pi] \times \! [-6, 6] \subset \mathcal{X}$. Here, $N$ and $M$ are the total numbers of the grids in the $x_1$- and $x_2$-directions, respectively; we choose $N = 20$ and $M = 21$, so the total $420$ number of grid points in $\Omega$ are used as initial states. When inputting to $V$, the first component $x_1$ of $x$ is normalized to a value within $[-\pi, \pi]$ by adding $\pm 2 \pi k$ to it for some $k \in \mathbb{Z}$. 

In what follows, we simulate four different settings, whose learning objective is to swing up and eventually settle down the pendulum at the upright position $\theta_t = 2\pi k$ for some $k \in \mathbb{Z}$, under the torque limit $|U_t| \leq u_\mathsf{max}$. For each case, we basically consider the reward function $r$ given by \eqref{eq:strictly concave reward} and \eqref{eq:integral formula of S} with
\begin{equation}
	s(\mathfrak{u}) = u_{\mathsf{max}} \tanh(\mathfrak{u}/u_{\mathsf{max}}).
	\label{eq:sim:s}
\end{equation}
As the inverted pendulum dynamics is input-affine, this setting corresponds to the concave Hamiltonian formulation in \secref{subsubsection:case I} (with a bounded $r$ if $\mathfrak{r}$ is bounded). The implementation details (the features $\phi$, policy evaluation, and policy improvement) are provided in \secref{appendix:section:implementation details}; the MATLAB/Octave source code for the simulations is also available online.\footnote{\href{http://github.com/JaeyoungLee-UoA/PIs-for-RL-Problems-in-CTS/}{github.com/JaeyoungLee-UoA/PIs-for-RL-Problems-in-CTS/}}

\Subsection{Case 1: Concave Hamiltonian with Bounded Reward}{subsection:simulation:case1}

First, we consider the reward function $r$ given by \eqref{eq:strictly concave reward} and \eqref{eq:integral formula of S} with $s(\cdot)$ given by \eqref{eq:sim:s}, $\Gamma = 10^{-2}$, and $\mathfrak{r}(x) = \cos x_1$. As mentioned above, this setting corresponds to the concave Hamiltonian formulation in \secref{subsection:RL under u-AC setting}, resulting in the following policy improvement update rule (see \secref{appendix:section:implementation details} for details):
\begin{align}
	\!\!\!\!\!
	\pi_i(x)
	\approx 
	\pi(x; \theta_i^*) 
	= - 5 \tanh \! \big (\! \cos x_1 \! \cdot \! \nabla_{\!x_2} \phi (x) \! \cdot \! \theta_i^* / 5 \big ).
	\label{eq:sim:VGB update}
\end{align}
As $\mathfrak{r}$ (hence $r$) is bounded, this setting also corresponds to ``discounted RL under Assumption~\ref{assumption:bounded R}'' in \secref{subsection:discounted RL with bounded v}. Therefore, the initial and subsequent VFs in PIs are all bounded; the properties in \secsref{subsubsection:case I} and \ref{subsection:discounted RL with bounded v} are all true; the Assumptions in Table~\ref{table:summary of case studies} w.r.t. \secsref{subsubsection:case I} and \ref{subsection:discounted RL with bounded v} are also all relaxed.

Figs.~\ref{fig:pos trj}(a), (c) and Figs.~\ref{fig:vi and pi}(a)--(d) show the trajectories of $\vartheta_t$ under the policies obtained during PI and the estimates of the optimal solution $(v_*, \pi_*$) finally obtained at the iteration $i = 50$, respectively; the yellow regions in Fig.~\ref{fig:pos trj} correspond to the trajectories of $\vartheta_t$ generated by the intermediate policies obtained by the PIs at iterations $i = (3),4,5,\cdots,49$. Although both DPI and IPI variants generate rather different trajectories of $\vartheta_t$ in Figs.~\ref{fig:pos trj}(a), (c), due to the difference in the estimates of the VF and policy (e.g., see Figs.~\ref{fig:vi and pi}(a)--(d)), both methods have achieved the learning objective merely after the first iteration. Here, the difference in the $\vartheta_t$-trajectories mainly comes from the different initial behaviors near $\vartheta = \pi$ --- see the differences in the policies in Figs.~\ref{fig:vi and pi}(c), (d) (and also the VF estimates in Figs.~\ref{fig:vi and pi}(a), (b)) near the borderlines $\vartheta = \pm \pi$. Also note that both DPI and IPI methods have achieved our learning objective without using an initial \emph{stabilizing} policy that is usually required in the optimal control setting under the total discounting $\gamma = 1$ (e.g., \citealp{Abu-Khalaf2005,Vrabie2009b,Lee2015}).

\Subsection{Case 2: Optimal Control}{subsection:simulation:case2}
	A better performance can be obtained if the state reward function $\mathfrak{r}$ in Case 1 is replaced by 
	\[
		\mathfrak{r}(x) = - x_1^2 - \epsilon \! \cdot \! x_2^2 \;\;\, \textrm{ with } \epsilon = 10^{-2}. 	
	\]
	This setting corresponds to the nonlinear optimal control introduced and discussed in \secref{subsection:nonlinear optimal control}. In this case, whenever input to $\mathfrak{r}$, the first component $x_1$ is normalized to a value within $[-\pi, \pi]$. Here, $\mathfrak{r}$ is still not bounded due to the existence of the term $-\epsilon \cdot x_2^2$, but Algorithm~\ref{algorithm:Variants under bounded v} (without assuming the boundedness of $\smash v_{\pi_0}$) can be successfully applied as shown in Figs.~\ref{fig:pos trj}(b), \ref{fig:vi and pi}(e), and \ref{fig:vi and pi}(g). 
	Fig.~\ref{fig:pos trj}(b) illustrates the trajectories of $\vartheta_t$ under the policies obtained by the DPI variant. Compared with Case~1, this setting gives a better initial and asymptotic performance --- every trajectory of $\vartheta_t$ in Fig.~\ref{fig:pos trj}(b) is almost the same as the final one (faster convergence of the PI) and converges to the goal state $x = (0, 0)$ more rapidly than any trajectories of $\vartheta_t$ in Case~1. In particular, the initial behavior near $\vartheta = \pm \pi$ has been improved, so that the policies in this case swing up the pendulum much faster than Case 1. One possible explanation about this is that the higher magnitude of the gradient of $\mathfrak{r}$ near $x_1 = \pm \pi$ expedites the initial swing-up process (note, in Case 1, $\nabla \mathfrak{r}(\pm \pi, x_2) = 0$ for any $x_2$). See also the difference of the final VF and policy in Figs.~\ref{fig:vi and pi}(e), (g) (Case 2) from those in Figs.~\ref{fig:vi and pi}(a)--(d)  (Case 1). The results for IPI are almost similar to DPI in this case, so their figures are omitted. \pagebreak

\Subsection{Case 3: Bang-bang Control}{subsection:simulation:case3}
	If $\Gamma \to 0$, the reward function $r$ and the policy update rule \eqref{eq:sim:VGB update} in Case 1 (\secref{subsection:simulation:case1}) are simplified to $r(x, u) = \cos x_1$ and 
\[
\pi_i(x)
	\approx 
	\pi(x; \theta_i^*) =	 - 5 \cdot \mathrm{sign} \big (\cos x_1 \! \cdot \! \nabla_{\!x_2} \phi (x) \cdot \theta_i^* \big)
\]
(see \secref{appendix:section:implementation details} for details), a bang-bang type discrete control. The PI methods can be also applied to optimize this bang-bang type controller. Note that this case is beyond our scope of the theory developed in \secsref{section:preliminaries}--\ref{section:case studies} since the policy is discrete, not continuous. For this bang-bang control framework, Fig.~\ref{fig:pos trj}(d) shows the $\vartheta_t$-trajectories under the discrete policies obtained by the IPI variant in Algorithm~\ref{algorithm:Variants under bounded v}. Though the fast switching behavior of the control $U_t$ is inevitable near $x = (0, 0)$, due to $\mathrm{sign}(\cdot)$, the initial and asymptotic control performance, compared with Case~1, has been increased in the limit $\Gamma \to 0$ up to the performance of optimal control (Case 2). 

By limiting $\Gamma \to 0$, the control policy in Case 2 can be also made a bang-bang type control, but in this case, with 
\begin{equation}
	r(x,u) = - x_1^2 - \epsilon \! \cdot \! x_2^2 \;\; \textrm{ with } \epsilon = 10^{-2}. 
	\label{eq:sim:reward for optimal bang-bang control}	
\end{equation}
We have observed that the performance of the PI methods in this case is almost same as that shown in Fig.~\ref{fig:pos trj}(d) for the previous case ``$r(x, u) = \cos x_1$'', derived from Case~1. Figs.~\ref{fig:vi and pi}(f) and (h) show the envelopes of the VF and the bang-bang policy under~\eqref{eq:sim:reward for optimal bang-bang control}, both of which are consistent with the envelopes for $\Gamma = 10^{-2}$ shown in Figs.~\ref{fig:vi and pi}(e) and (g). 

\Subsection{Case 4: Bang-bang Control with Binary Reward}{subsection:sim:bang-bang with binary}

In RL problems, the reward is often binary and sparely given only at or near the goal state. To investigate this case, we also consider the bang-bang policy given in the previous subsection, but with the binary reward function: 
\[
	r(x,u) = 
	\begin{cases}
 		1, \textrm{ if } |x_1| \leq 6/\pi \textrm{ and } |x_2| \leq 1/2
 		\\[5pt]
 		0, \textrm{ otherwise,}
	\end{cases}
\]
This gives the reward signal $R_t = 1$ near the goal state $x = (0, 0)$ only. Figs.~\ref{fig:pos trj}(e) and (f) illustrate the $\theta_t$-trajectories under the policies generated by the DPI and IPI variants (i.e., Algorithm~\ref{algorithm:Variants under bounded v}), respectively. Though the initial performance is neither stable ($i = 1$) nor consistent to each other ($i = 1,2$), both PI methods eventually converge to the same seemingly near-optimal point ($i = 3,4,\cdots, 50$). Note that the  performance after learning ($i = 50$) for both cases is the same as that of Cases~2 and 3 until around $t = 3 [s]$ as can be seen from Figs.~\ref{fig:pos trj}(b) and (d)--(f). Figs.~\ref{fig:vi and pi}(i)--($\ell$) also show the estimates of the optimal VF and policy at $i = 50$. Although the details are a bit different, we can see that both methods finally result in similar consistent estimates of the VF and policy. In this binary reward case, the shapes of the VF shown in Figs.~\ref{fig:vi and pi}(i) and (j) are distinguished from the others illustrated in Figs.~\ref{fig:vi and pi}(a),(b),(e), and (f) due to the reward information condensed near the goal state $x = (0, 0)$ only. Even in this situation, our PI methods were able to achieve the goal at the end, as shown in Figs.~\ref{fig:pos trj}(e) and (f). For the DPI variant, we have simulated this case with $M = 20$, instead of $M = 21$. 

\Subsection{Discussions}{subsection:sim:discussions}

We have simulated the variants of DPI and IPI (Algorithm~\ref{algorithm:Variants under bounded v}) under the four scenarios above. Some of them have achieved the learning objective immediately at the first iteration, and in all of the simulations, the proposed methods were able to achieve the goal, eventually. On the other hand, the implementations of the PIs have the following issues.
\begin{enumerate}
	\item The least-squares solution $\theta_i^*$ of each policy evaluation minimizes the Bellman error over a finite number of initial states in $\Omega$ (as detailed in \secref{appendix:section:implementation details}), meaning that it is not the optimal choice to minimize the Bellman error over the entire region $\Omega$. As mentioned in \secref{section:PIs}, the ideal policy evaluation cannot be implemented precisely---even when $\Omega$ is compact, it is a continuous space and thus contains an (uncountably) infinite number of points that we cannot fully cover in practice. \\[-7.5pt]
	\item As the dimension of the data matrix in the least squares is $L \times (NM) = 121 \times 420$ (see \secref{appendix:section:implementation details}), calculating the least-squares solution $\theta_i^*$ is computationally expensive, and the numerical error (and thus the convergence) is sensitive to the choice of the parameters such as (the number of) the features $\phi$, the time step $\Delta t$, discounting factor~$\gamma$, and of course, $N$ and $M$. In our experiments, we have observed that Case 2 (optimal control) was least sensitive to those parameters.  \\[-7.5pt]
	\item The VF parameterization. Since the pendulum is symmetric at $x_1 = 0$, the VFs and policies obtained in Fig~\ref{fig:vi and pi} are all symmetric, and thus it might be sufficient to approximate the VF over $[0, \pi] \times [-6, 6] \subset \Omega$, with a less number of weights, and use the symmetry of the problem. Due to the over-parameterization, we have observed that the weight vector $\theta_i^*$ in certain situations never converges but oscillates between two values, even after the VF $V_i$ has almost converged over $\Omega$. 
\end{enumerate}
All of these algorithmic and practical issues are beyond the scope of this paper and remain as a future work.

\Section{Conclusions}{section:conclusion}
$\;$\\[-17.5pt]
In this paper, we proposed fundamental PI schemes called DPI (model-based) and IPI (partially model-free) to solve the general RL problem formulated in CTS. We proved their fundamental mathematical properties: admissibility, uniqueness of the solution to the BE, monotone improvement,  convergence, and the optimality of the solution to the HJBE. Strong connections to the RL methods in CTS---TD learning and VGB greedy policy update---were made by providing the proposed ones as their ideal PIs. Case studies simplified and improved the proposed PI methods and the theory for them, with strong connections to RL and optimal control in CTS. Numerical simulations were conducted with model-based and partially model-free implementations to support the theory and further investigate the proposed PI methods beyond, under an initial policy that is admissible but not stable. Unlike the existing PI methods in the stability-based frameworks, an initial \emph{stabilizing} policy is not necessarily required to run the proposed ones. We believe that this work provides the theoretical background, intuition, and improvement to both (i) PI methods in optimal control and (ii) RL methods, to be developed in the future and developed so far in CTS domain. 



{\small 
\bibliographystyle{myplainnat}

}


\onecolumn

\appendix

\begin{center}
\begin{spacing}{2}
\textbf{\Large Policy Iterations for Reinforcement Learning Problems in Continuous Time and Space --- Fundamental Theory and Methods: Appendices}
\vspace{1em}

{\large Jaeyoung Lee,$^\mathrm{a}$$\;$ Richard S. Sutton$^\mathrm{b}$}

{\small $^\mathrm{a}$\emph{Department of Electrical and Computer Eng., University of Waterloo, Waterloo, ON, Canada, N2L 3G1 (jaeyoung.lee@uwaterloo.ca)}}
\vspace{-1em}

{\small $^\mathrm{b}$\emph{Department of Computing Science, University of Alberta, Edmonton, AB, Canada, T6G 2E8 (rsutton@ualberta.ca)}}
\vspace{-1.5em}
\end{spacing}
\end{center}	

\newcommand{\hr}{\begin{center} \line(1,0){504} \end{center}}

\hr
	\vspace{-1.0em}
	{
	\footnotesize
	{\bf Abstract}
	\vspace{-1.0em}
	
This supplementary document provides additional studies and all the details of the contents presented by \citet{Lee2020b}, as listed below. Roughly speaking, we present related works, details of the theory, algorithms, and implementations, additional case studies, and all the proofs, with the same abbreviations, terminologies, and notations. All the numbers of equations, sections, theorems, lemmas, etc. that do not contain any alphabet will refer to those in the main paper \citep{Lee2020b}, whereas any numbers starting with an alphabet correspond to those in the Appendices herein.
	}
	\vspace{-1.5em}	
\hr

\vspace{-1.75em}

{
\renewcommand{\cftsecleader}{\cftdotfill{\cftdotsep}}
\renewcommand*{\contentsname}{}
\tableofcontents
}

\section{Notations and Terminologies}
\label{appendix:notations and terminologies}

We provide a complete list of notations and terminologies used in the main paper and the appendices. In any statement, iff and s.t. stand for \emph{if~and~only~if} and \emph{such that}, respectively. ``$\doteq$'' denotes the equality relationship that is true \emph{by definition}.

\begin{multicols}{2}

\newcommand{\notations}[1]{\begin{tabular}{L{4em}L{22em}} #1 \end{tabular}}
\newcommand{\abbreviations}[1]{\notations{#1}}

\subsection{Abbreviations}
\abbreviations{
	ADP & approximate dynamic programming \\
	BE & Bellman equation \\
	CTS & continuous time and space \\
	DPI & differential policy iteration \\
	IPI & integral policy iteration \\
	HJB & Hamilton-Jacobi-Bellman \\	
	HJBE & Hamilton-Jacobi-Bellman equation \\
	LQR & linear quadratic regulation \\
	MDP & Markov decision process \\
	ODE & ordinary differential equation \\
	PI & policy iteration \\
	RBF & radial basis function \\		
	RL & reinforcement learning \\
	TD & temporal difference \\
	VF & value function \\
	VGB & value-gradient-based
}

\subsection{Sets, Vectors, and Matrices}
\notations{
	$\mathbb{N}$ & set of all natural numbers \\
	$\mathbb{R}$ & set of all real numbers \\
	$\mathbb{C}$ & set of all complex numbers \\
	$\mathbb{Z}$ & set of all integers \\
	$\mathbb{R}^{n \times m}$ & set of all $n$-by-$m$ real matrices \\
	$\mathbb{R}^n$ & $n$-dimensional Euclidean space $\doteq \mathbb{R}^{n \times 1}$
}

For a matrix $A \in \mathbb{R}^{n \times m}$ and a vector $x \in \mathbb{R}^m$, 

\notations{
	$A^\mathsf{T}$ & transpose of $A$ \\
	$\rank{A}$ & rank of $A$ \\
	$\|x\|$ & Euclidean norm of $x$, i.e., $\|x\| \doteq (x^\mathsf{T} x)^{1/2}$ \\
	$\|x\|_\Omega$ & \makecell[l]{distance of $x$ from a subset $\Omega \subset \mathbb{R}^m$, i.e., \\ $\|x\|_\Omega \doteq \inf\{ \| x - y \|: y \in \Omega \}$}\\  
	$\MatNorm{A}$ & \makecell[l]{induced norm of $A$, i.e., $\MatNorm{A} \doteq {\displaystyle \sup_{\smash \|x \| = 1}} \| Ax \|$}\\
	$I$ & identity matrix with a compatible dimension
}

\subsection{Euclidean Topology}  
Let $\Omega \subseteq \mathbb{R}^n$.
\begin{description} [itemsep=.25em]
	\item $\Int{\Omega}$ denotes the \emph{interior of $\Omega$}.
	\item $\Bdy{\Omega}$ denotes the \emph{boundary of $\Omega$}.
	\item $\Omega$ is said to be \emph{compact} iff it is closed and bounded.
\end{description}
If $\Omega$ is open, then $\Omega \cup \Bdy{\Omega}$ (resp. $\Omega$) is called an \emph{$n$-dimensional manifold with} (resp. \emph{without}) \emph{boundary}. By this definition, a manifold contains no isolated point. 

\subsection{Functions, Sequences, and Convergence}
\label{appendices:notations:4}
Let $\Omega \subseteq \mathbb{R}^n$ and $f: \Omega \to \mathbb{R}^m$ be a function. 
\begin{description} [itemsep=.25em]
	\item $f \in \mathrm{C}^k$ (i.e., \emph{$f$ is $\mathrm{C}^k$}) iff the $k$th order partial derivatives of $f$ all exist and are continuous, over the interior $\Int{\Omega}$.
	\item $\nabla f: \Int{\Omega} \to \mathbb{R}^{m \times n}$ denotes \emph{the gradient of~$f$}.
	\item $f$ is \emph{locally Lipschitz} iff for each $x \in \Omega$, there exists $L > 0$ and a neighborhood $\mathcal{N}_x$ of $x$ s.t. for all $y, z \in \mathcal{N}_x$, 
\begin{equation}
	\| f(y) - f(z) \| \leq L \| y - z \|.
	\label{eq:Lipschitz inequality}
\end{equation}
	\item	 $f$ is \emph{globally Lipschitz} iff $\exists L > 0$ s.t. \eqref{eq:Lipschitz inequality} holds $\forall y, z \in \Omega$.
	\item $f \in \mathrm{C}^1_\mathsf{Lip}$ (i.e., \emph{$f$ is $\mathrm{C}^1_\mathsf{Lip}$}) iff $f$ is locally Lipschitz and $\mathrm{C}^1$.
	\item $f$ is \emph{odd} iff $f(-x) = f(x)$ for all $x \in \Omega$. 
	\item $f$ with $m = n$ is \emph{strictly monotone} iff for each $x$, $x' \in \Omega$,  \vspace{-.5em}
	\[
			(f(x) - f(x'))^{\mathsf{T}}(x - x')  > 0 
			\text{ whenever } x \neq x'.
	\]
	\vspace{-1.5em}
	\item $f(E)\doteq \{ f(x) : x\in E\}$, the image of $E \subseteq \Omega$ under $f$.
\end{description}

A sequence $\langle a_i \rangle_{i=1}^\infty$ is abbreviated as $\langle a_i \rangle$ or $a_i$ for notational simplicity. 
A sequence of functions $\langle f_i \rangle$ \emph{converges (to $f$)}
\begin{description} [itemsep=.25em]
	\item \emph{pointwise} iff $f_i(x) \to f(x)$ for each $x \in \Omega$;
	\item \emph{uniformly} on $E \subseteq \Omega$ iff $\sup_{x \in E} \| f_i(x) - f(x) \| \to 0$;
	\item \emph{locally uniformly} iff for each $x \in \Omega$, there is a neighborhood of $x$ on which $f_i \to f$ uniformly.
\end{description}
For any two functions $f_1, f_2: \mathbb{R}^n \to [-\infty, \infty)$, we write 
\[
	f_1 \leqslant f_2 \;\, \Longleftrightarrow \;\, f_1(x) \leq f_2(x) \quad x \in \mathbb{R}^n.
\]
A function $f: \Omega \to \mathbb{R}$ is said to be
\begin{description} [itemsep=.25em]
	\item \emph{positive semidefinite} iff $f(0) = 0$ and $f \geqslant 0$;
	\item \emph{negative semidefinite} iff $-v$ is positive semidefinite;
	\item \emph{positive definite} iff $f(0) = 0$ and $f(x) > 0$ for all $x \neq 0$;
	\item \emph{negative definite} iff $-f$ is positive definite;
	\item \emph{radially unbounded} iff $\inf_{\|x\| \geq r} |f(x)| \to \infty$ as $r \to \infty$;
	\item \emph{radially nonvanishing} iff $\inf_{\|x\| \geq r} |f(x)| \not \to 0$ as $r \to \infty$;
	\item \emph{convex} iff for each $x, x' \in \Omega$ and $\beta \in (0, 1)$, \vspace{.25em}
	\begin{itemize} [itemsep=.25em]
		\item [(1)] $x_\beta \doteq \beta x + (1 - \beta) x' \in \Omega$ (i.e., $\Omega$ is \emph{convex}),
		\item [(2)] $f(x_\beta) \leq \beta \cdot f(x) + (1 - \beta) \cdot f(x')$;
	\end{itemize}
	\item \emph{concave} iff $-f$ is convex;
	\item \emph{strictly convex} iff $f$ is convex and for any $\beta \in (0, 1)$, \vspace{-.5em}
	\[
		f(x_\beta) < \beta \cdot f(x) + (1 - \beta) \cdot f(x')\textrm{ whenever } x \neq x';
	\]
	\item \emph{strictly concave} iff $-f$ is strictly convex.
\end{description} 
$f: [0, \infty) \to [0, \infty)$ is said to be $\mathcal{K}_\infty$ iff $f(0) = 0$ and $f$ is strictly increasing, radially unbounded, and continuous; 

A square matrix $P \in \mathbb{R}^{n \times n}$ is 
\begin{description} [itemsep=.25em]
	\item \emph{positive (semi)definite} iff so is $z \mapsto z^\mathsf{T} P z$ and $P^\mathsf{T} = P$;
	\item \emph{negative (semi)definite} iff so is $z \mapsto z^\mathsf{T} P z$ and $P^\mathsf{T} = P$.
\end{description}
For $P, P' \in \mathbb{R}^{n \times n}$, we denote $P < P'$ (resp. $P \leq P'$)  iff $P' - P$ is positive definite (resp. positive semidefinite). 

\subsection{Reinforcement Learning} 

\renewcommand{\notations}[1]{\begin{tabular}{L{2.75em}L{22em}} #1 \end{tabular}}
	\notations{
		$l$ & dimension $\in \mathbb{N}$ of the state space $\mathcal{X}$ \\
		$m$ & dimension $\in \mathbb{N}$ of action spaces (e.g., $\mathcal{U}$ and $\mathcal{A}$)	\\
	}
	
An \emph{action space} is an $m$-dimensional manifold in $\mathbb{R}^m$ with or without boundary hence has no isolated point by definition.$\!\!$

	\notations{
		$\mathcal{X}$, $\mathcal{X}^\mathsf{T}$ & state space $\mathcal{X} \doteq \mathbb{R}^l$ and $\mathcal{X}^\mathsf{T} \doteq \mathbb{R}^{1 \times l}$ \\
		$\mathcal{U}$ & action space $\subseteq \mathbb{R}^m$	\\
		$\mathcal{A}$ & a transformed action space $\subseteq \mathbb{R}^m$ (\secsref{subsection:RL under u-AC setting} and \ref{subsubsection:case III}) \\
		$\mathbb{T}$ & time space $\mathbb{T} \doteq [0, \infty)$
	}

	\notations{
		$f$, $f^x$ & dynamics $f: \mathcal{X} \times \mathcal{U} \to \mathcal{X}$ and $f^x(u) \doteq f(x, u)$ \\
		$f_\mathsf{d}$ & drift dynamics $f_\mathsf{d}: \mathcal{X}  \to \mathcal{X}$ \\
		$f_\mathsf{c}$ & input-coupling dynamics $f_\mathsf{c}: \mathcal{X} \times \mathcal{U}  \to \mathcal{X}$ \\
		$F_\mathsf{c}$ & input-coupling matrix $F_\mathsf{c}: \mathcal{X}  \to \mathbb{R}^{n \times m}$ (\secref{subsection:RL under u-AC setting}) \\		
		$r$, $r^x$ & reward function $r: \mathcal{X} \times \mathcal{U} \to \mathbb{R}$; $r^x(u) \doteq r(x, u)$ \\
		$r_\mathsf{max}$ & the reward maximum ``$\max_{(x, u)} r(x, u)$''	\\
		$\gamma$ & discount factor $\in (0, 1]$ \\
		$\alpha$ & attenuation rate $\alpha \doteq - \ln \gamma \geq 0$	\\
		$h$ & Hamiltonian function $h  :\mathcal{X} \times \mathcal{U} \times \mathcal{X}^\mathsf{T} \to \mathbb{R}$	\\
		$u_*$ & maximal function $u_*(x, p) \in \Argmax_{u} h(x, u, p)$ \\
	}
	
	\notations{
		$t$ & time variable $\in \mathbb{T}$	\\
		$\eta$ & time horizon $\in (0, \infty]$ \\
		$X_t$ & state vector $\in \mathcal{X}$ at time $t$	\\
		${\dot X}_t$ & the time derivative $\in \mathcal{X}$ of $X_t$ at time $t$	\\		
		$U_t$ & action (also called control) vector $\in \mathcal{U}$ at time $t$	\\
		$A_t$ & a transformed action vector $\in \mathcal{A}$ at time $t$	\\		
		$R_t$ & reward at time $t$, i.e., $r(X_t, U_t) \in \mathbb{R}$	\\
		$\mathfrak{R}_\eta$ & discounted cumulative reward up to horizon $\eta$ \\
		${\dot v}$ & time derivative $dv(X_t)/dt = \nabla v(X_t) f(X_t, U_t)$		
	}	

A \emph{policy} is a continuous function from $\mathcal{X}$ to $\mathcal{U}$; for a policy~$\pi$,$\!\!\!$ 

	\notations{
	$\mathbb{G}_\pi^{x} [ Y]$ & value~$Y$ when $X_0 = x$ and $U_{t} = \pi(X_{t})$ $\forall t \in \mathbb{T}$	\\
	$v_\pi$ & value function (VF) with respect to $\pi$	\\
	$\widebar v$ & a uniform upper-bound of VFs ($\bar v = 0$ for $\gamma = 1$ $\qquad$ and $\bar v = r_\mathsf{max} / \alpha$ otherwise --- see Lemma~\ref{lemma:VF range}) \\
	$f_\pi$ & closed-loop dynamics $f_\pi(x) \doteq f(x, \pi(x))$	\\
	$r_\pi$ & closed-loop reward function $r_\pi(x) \doteq r(x, \pi(x))$	\\	
	$\pi'$ & an improved/maximal policy $\pi' \! \succcurlyeq \pi$, i.e., $v_{\pi'} \! \geqslant v_\pi$  	\\
	}

	When $f_\pi$ is locally Lipschitz, $t_\mathsf{max}(x; \pi)$ denotes the minimal time s.t. $\forall t \geq t_\mathsf{max}(x; \pi)$, \emph{no} state $\mathbb{G}_\pi^x[X_t]$ exists (\secref{subsection:RL with local Lipschitzness}).
	
	\notations{	
	$\Pi_\mathsf{a}$ & set of all admissible policies 	\\
	$\Pi_\mathsf{Lip}$ & set of all locally Lipschitz policies	\\
	$\mathcal{V}_\mathsf{a}$ & set of all admissible VFs	\\
	$d$, $d_\Omega$ & a metric and the uniform pseudometric, on $\mathcal{V}_\mathsf{a}$	\\
	$\mathcal{T}$ & the PI operator \\
	$v_*$ & a solution to the HJBE or the optimal VF	\\
	$v^*$ & a (unique) fixed point of $\mathcal{T}$ \\	
	$\pi_*$ & an HJB or optimal policy (or a function~$\pi_*$ that satisfies \eqref{eq:optimal policy} and is potentially discontinuous).
	}	

\subsection{Policy Iteration}

\notations{
		$i$ & iteration index $\in \mathbb{N}$	\\
		$v_i$, $V_i$ & solution to the BE at iteration $i$; $V_i \doteq v_i / \Delta t$ 	\\
		${\hat v}_*$ & limit function ${\hat v}_*(x) \doteq \sup_i v_i(x) = {\displaystyle \lim_{i \to\infty}} v_i(x)$	\\		
		$\pi_0$, $\pi_i$ & an initial policy and the policy at iteration $i$	\\
		$\Delta t$ & small time step $(0 < \Delta t \ll 1)$	\\
		$\gamma_\mathsf{d}$	& discount factor $\doteq \gamma^{\Delta t}$ in discrete time	\\
		${\hat \gamma}_\mathsf{d}$ & an approximation of $\gamma_\mathsf{d} \approx {\hat \gamma}_\mathsf{d} \doteq 1 - \alpha_\mathsf{d}$ 	\\
		$ \alpha_\mathsf{d}$ & attenuation rate $\doteq \alpha \Delta t = -\ln \gamma_\mathsf{d}$ in discrete time
}

\subsection{Optimal Control and LQRs}

\notations{
		$c$ & cost function $c \doteq - r$	\\
		$c_\pi$ & closed-loop cost function $c_\pi(x) \doteq c(x, \pi(x))$	\\
		$C_t$ & cost at time $t$, i.e., $c(X_t, U_t) \in \mathbb{R}$ \\
		$J_\pi$ & cost value function $J_\pi \doteq - v_\pi$	\\
		$J_i$, $J_*$ & $J_i \doteq -v_i$ and $J_* \doteq -v_*$	\\
		$\Pi_0$ & set of all policies $\pi \in \Pi_\mathsf{Lip}$ s.t. $\pi(0) = 0$ 
}

Let $A \in \mathbb{R}^{l \times l}$, $B \in \mathbb{R}^{m \times l}$, and $C \in \mathbb{R}^{p \times l}$. Then, 
\begin{description} [itemsep=.25em]
	\item $A$ is \emph{Hurwitz} iff every eigenvalue has a negative real part;
	\item $(A, B)$ \emph{stabilizable} iff $\exists K \in \mathbb{R}^{m \times l}$ s.t. $A - BK$ is Hurwitz;
	\item $(C, A)$ \emph{observable} iff for any $\eta > 0$, the initial state $X_0$ can be determined from the history $\{ (CX_t, U_t)\}_{t \in [0, \eta]}$, where $\{X_t\}_{t \in [0, \eta]}$ satisfies ${\dot X}_t = AX_t + BU_t$.
\end{description}
\end{multicols}

\section{Highlights and Related Works}
	\label{section:related work}

	First, we briefly review the related works from RL and optimal control fields. We also highlight the main aspects of (i) the proposed PI methods and the underlying theory, both developed by \citet{Lee2020b}, and (ii) the appendices herein.

	\textbf{DPI \& IPI.} Two main PI methods in our work are DPI, whose policy evaluation is associated with the differential BE, and IPI associated with the integral BE. The former was inspired by the model-based PI methods in optimal control (e.g., \citealp{Rekasius1964,Leake1967,Saridis1979,Beard1997,Abu-Khalaf2005,Bian2014})
	and has a direct connection to TD(0) in CTS \citep{Doya2000,Fremaux2013} --- see \secref{subsection:DPI}. As regards to the latter, the integral BE was first introduced by \citet{BairdIII1993} in the field of RL and then spotlighted in the optimal control community, resulting in a series of IPI methods applied to a class of input-affine dynamics for optimal regulations \citep{Vrabie2009b,Lee2015}, robust control \citep*{Wang2016}, and (discounted) LQ tracking control (\citealp*{Modares2014,Zhu2015}; \citealp{Modares2016}), with a number of extensions to off-policy IPI methods (e.g., \citealp{Bian2014,Lee2015,Wang2016,Modares2016}). In our work \citep{Lee2020b}, 
	\begin{enumerate}
		\item the proposed IPI was motivated by the first IPI given by  \citet{Vrabie2009b} for nonlinear optimal regulations;
		\item the ideas of DPI and IPI have generalized for a broad class of dynamics and reward functions in CTS shown in \secref{section:preliminaries}, which includes the existing RL tasks \citep{Doya2000,Mehta2009,Fremaux2013} and the case tasks of RL and optimal control presented in \secsref{section:case studies} and \ref{appendix:section:additional case studies}.
	\end{enumerate}
		
	\textbf{Case Studies.} 
		\begin{enumerate}
			\item \textbf{(\secref{subsection:RL under u-AC setting}. Concave Hamiltonian Formulation).} A highlight is in \secref{subsection:RL under u-AC setting}, which draws the connection to the VGB greedy policy update \citep{Doya2000}, a general idea of simplifying policy improvement in input-constrained RL problems. There exist similar ideas in the optimal control field for input-constrained \citep{Lyashevskiy1996,Abu-Khalaf2005} and unconstrained optimal regulations \citep{Rekasius1964,Saridis1979,Beard1997,Abu-Khalaf2005,Vrabie2009b,Lee2015} under input-affine dynamics, and even for the non-affine dynamics \citep{Bian2014,Kiumarsi2016}.
			\item \textbf{(\secref{subsection:nonlinear optimal control}. Nonlinear Optimal Control).} The existing PI methods for the optimal regulations, presented in the literature above and by \citet{Leake1967}, are strongly linked to \secref{subsection:nonlinear optimal control}, where we case-studied asymptotic stability and fundamental properties of DPI and IPI applied to a general optimal regulation problem with non-affine dynamics and $\gamma \in (0, 1]$. The asymptotic stability conditions given in Theorem~\ref{thm:stability of optimal control} in \secref{subsection:nonlinear optimal control} are similar to and inspired by \citet[Assumptions~2.3 and 3.8]{Gaitsgory2015}. 
			\item \textbf{(\secref{subsection:discounted RL with bounded v} Discounted RL with Bounded VF).} Another highlight is the discounted RL problem with \emph{bounded} reward function (\secref{subsection:discounted RL with bounded v}). In this case, the VF is guaranteed to be bounded for \mbox{any policy}, by which the underlying PI theory becomes dramatically simplified and clear (see Corollary~\ref{cor:properties when R is bounded}). This framework is akin to the RL tasks in a finite MDP, where the reward defined for each state transition is bounded \citep{Sutton2018}. 
		\end{enumerate}
	See also \secref{section:simulation} for simulation examples of those case studies in \secref{subsection:RL under u-AC setting}, \ref{subsection:discounted RL with bounded v}, and \ref{subsection:nonlinear optimal control}, for RL and optimal control. 
			
	\textbf{Admissibility \& Asymptotic Stability.} Theoretically, since we consider a stability-free RL framework (under the minimal assumptions in \secref{section:preliminaries}), we excluded asymptotic stability from the definition of an admissible policy. Here, the notion of admissibility in optimal control has been defined with asymptotic stability (e.g., \citealp{Beard1997,Abu-Khalaf2005,Vrabie2009b,Modares2014,Bian2014,Lee2015} to name a few), and this work is the first to define admissibility in CTS \emph{without} asymptotic stability. Conversely, in a general optimal control problem, we also showed that when $\gamma = 1$, admissibility, according to our definition, implies asymptotic stability (if the associated VF is $\mathrm{C}^1$) --- see Theorem~\ref{thm:stability of optimal control} and Remarks~\ref{remark:admissibility implies asymptotic stability} and \ref{remark:admissibility} in \secref{subsection:nonlinear optimal control}. This means that asymptotic stability can be removed from the definition of admissibility even in \emph{optimal control}. The admissibility in discounted optimal control was also investigated in \secref{subsection:nonlinear optimal control} under the condition weaker than a Lyapunov's global asymptotic stability criterion (e.g., see Theorem~\ref{thm:uniqueness condition in optimal control}). 
		
	\textbf{(Mode of) Convergence.} We characterized the convergence properties of the PI methods towards the optimal solution in the following three ways. Those three modes provide different convergence conditions and compensate for one another.	
	\begin{enumerate}
		\item In the first characterization, we employed \citet{Bessaga1959}'s converge fixed point principle to show that the VFs generated by the PI methods converge to the optimal one in a metric (Theorem~\ref{thm:true convergence in metric}). This first-type convergence, called convergence in a metric, is weaker than locally uniform convergence below but does not impose any other assumptions than the existence and uniqueness of a fixed point that turns out to be the optimal VF by Corollary~\ref{cor:weak optimality}.
		\item  The second way was to extend the approach of \citet{Leake1967}'s, suggesting continuity of the PI operator (see Theorem~\ref{thm:true uniform convergence:1}) as one of the additional conditions for locally uniform convergence.
		\item  Lastly, we also generalized the convergence proof from the optimal control literature \citep{Saridis1979,Beard1997,Murray2002,Abu-Khalaf2005,Bian2014} to our RL framework, resulting in the strongest convergence among the three, under a certain condition other than the two above (see Theorem~\ref{thm:true uniform convergence:2}). In this direction, we highlight that for the proof of this third type convergence, the gradients of the VFs obtained by the PIs need to be assumed to converge locally uniformly, \emph{even for the existing results in optimal control}, as the convergence of the generated VFs does not imply any convergence of their derivatives (see Remark~\ref{remark:the needs for locally uniform convergence}).
	\end{enumerate}

\textbf{LQR.} In \secref{subsection:LQR}, we discuss DPI and IPI applied to a class of the LQR tasks \citep[Chapter~16]{Lancaster1995} where bilinear cost terms of states and controls exist. Here, DPI falls into a particular case of the existing general matrix-form PIs \citep{Arnold1984,Mehrmann1991}, but this study slightly generalizes many existing PI methods for the LQRs (e.g., \citealp{Kleinman1968,Vrabie2009a}; \citealp*{Lee2014}) by taking such bilinear cost terms into considerations, with the relaxation of the positive definite matrix assumption imposed on the general matrix-form PI \citep[Theorem~11.3]{Mehrmann1991}.

\section{More on the Bellman Equations with the Boundary Condition} 
\label{appendix:sec:replacement of bdy condition}

Here, the theory on (the uniqueness of) the BEs established in \secref{subsection:BE with boundary condition} is elaborated with supplementary  theorems and discussions. Let $v: \mathcal{X} \to \mathbb{R}$ be a function s.t. for a policy~$\pi$, either of the followings holds:
			\begin{align}
				&\text{(1) $v$ satisfies the integral BE: }
				v(x) = \mathbb{G}_\pi^x \big [ \mathfrak{R}_{\eta}
				+ \gamma^{\eta} \! \cdot\! v(X_{\eta}) \big ] \qquad \forall x \in \mathcal{X},
				\tag{\ref{eq:Bellman eq for v}}
				\\[5pt]
				&\text{(2) $v$ is $\mathrm{C}^1$ and satisfies the differential BE: }
				\alpha \cdot v (x) = 
					h(x, \pi(x), \nabla v(x)) \qquad \forall x \in \mathcal{X}.
				\tag{\ref{eq:differential BE for v}}
			\end{align}	
In \secref{subsection:BE with boundary condition}, we showed that the boundary condition~\eqref{eq:uniqueness condition for v}:
\begin{equation}
	\lim_{k \to \infty} \mathbb{G}_\pi^x \big  [ \, \gamma^{k \cdot \eta} \cdot v(X_{k \cdot \eta}) \, \big ] = 0 \qquad \forall x \in \mathcal{X}
	\tag{\ref{eq:uniqueness condition for v}}
\end{equation}
is necessary and sufficient for the policy $\pi$ being admissible and a solution $v$ to the BE \eqref{eq:Bellman eq for v} or \eqref{eq:differential BE for v} being equal to the VF~$v_\pi$. In other words, the boundary condition~\eqref{eq:uniqueness condition for v} ensures admissibility of $\pi$ and  uniqueness of the solution $v$ to the BE. However, except for a few cases, \eqref{eq:uniqueness condition for v} is hard or even impossible to check as it is a condition at infinity in time. The theorem below shows admissibility and weaker properties of the BEs but without the boundary condition~\eqref{eq:uniqueness condition for v}.

\begin{theorem}
	\label{thm:general policy evaluation thm}
	Let $\eta > 0$ be fixed and suppose $v$ satisfies either the integral BE~\eqref{eq:Bellman eq for v} or, with $v \in \mathrm{C}^1$, the differential BE~\eqref{eq:differential BE for v}. If $v$ is upper bounded (by zero if $\gamma = 1$), then (i) $\pi$ is admissible and $v \leqslant v_\pi$; (ii) the limit in \eqref{eq:uniqueness condition for v} exists and satisfies
	\[
		 v(x) - v_\pi(x) = \lim_{k \to \infty} \mathbb{G}_\pi^x \big [ \gamma^{k \cdot  \eta} \cdot v(X_{k \cdot \eta}) \big ] \leq 0 \qquad \forall x \in \mathcal{X}.		
	\]
\end{theorem}

\begin{Proof}
	Suppose $v$ satisfies the integral BE~\eqref{eq:Bellman eq for v} without loss of generality (or, convert the differential BE~\eqref{eq:differential BE for v} into \eqref{eq:Bellman eq for v} via Lemma~\ref{lemma:Bellman ineq} and fix $\eta > 0$). Then, the repetitive applications of the BE~\eqref{eq:Bellman eq for v} to itself $k$-times result in
	\begin{align*}
		v(x) 
		= \mathbb{G}_\pi^x \big [ \mathfrak{R}_\eta + \gamma^\eta \! \cdot \! v(X_\eta) \big ]	
		= \mathbb{G}_\pi^x \big [ \mathfrak{R}_{2 \eta} + \gamma^{2 \eta} \! \cdot \! v(X_{2 \eta}) \big ]
		=
		\;\, \cdots \;\,
		= \mathbb{G}_\pi^x \big [ \mathfrak{R}_{k \cdot \eta} + \gamma^{k \cdot \eta} \! \cdot \! v(X_{k \cdot \eta}) \big ]	\qquad \forall x \in \mathcal{X}.
	\end{align*}
	Hence, taking the limit $k \to \infty$ and noting that $v_\pi(x) = \lim_{k \to \infty} \mathbb{G}_\pi^x \big [ \mathfrak{R}_{k \cdot \eta}]$, we have 
	\[
		v(x) - v_\pi(x) = \lim_{k \to \infty} \mathbb{G}_\pi^x \big [ \gamma^{k \cdot \eta} \! \cdot \! v(X_{k \cdot \eta}) \big ] \leq \sup_{x \in \mathcal{X}} v(x) \cdot \lim_{k \to \infty} \gamma^{k \cdot \eta} = 0 \qquad \forall x \in \mathcal{X},
	\]
	where the inequality is true since $v$ is upper-bounded (by zero if $\gamma = 1$) and $\gamma \in (0, 1]$. Now that we established $v \leq v_\pi$, the policy $\pi$ is admissible as $-\infty < v(x) \leq v_{\pi}(x) \leq \widebar v < \infty$ for all $x \in \mathcal{X}$ by Lemma~\ref{lemma:VF range}, and the proof is completed.
\end{Proof}

In what follows, we introduce conditions sufficient for the boundary condition \eqref{eq:uniqueness condition for v} to be true. 
\begin{lemma}
	\label{lemma:boundary condition:sufficiency}
	Suppose $v$ is upper-bounded (by zero if $\gamma = 1$). Then, $v$ and a policy $\pi$ satisfy the boundary condition~\eqref{eq:uniqueness condition for v} if one of the followings (\emph{\textbf{a}} or \emph{\textbf{b}}) is true:
	\begin{enumerate}
		\item [\emph{\textbf{a.}}] $v$ is $\mathrm{C}^1$, and there exists a positive constant $\kappa > 0$ s.t. ${\dot v} (x, \pi(x)) \geq (\alpha - \kappa) \cdot v(x)$ for all $x \in \mathcal{X}$; \\[-5pt]
		\item [\emph{\textbf{b.}}] there exists a function $\zeta: \mathcal{X} \to \mathbb{R}$ and a constant $\ushort \alpha < \alpha$, both possibly depending on $\pi$, s.t. 
			\[
					\forall x \in \mathcal{X}\!: \; 
					\mathbb{G}_\pi^x [v(X_t)] \geq \zeta(x) \cdot \exp(\ushort \alpha \hspace{0.1em} t)
					 \; \text{ for all } t \in \mathbb{T}.
			\]
	\end{enumerate}
\end{lemma}

\newpage

\begin{Proof}
	\textbf{a.} Denoting $J \doteq -v$, the inequality can be written as
	\begin{align*}
		{\dot J} (x, \pi(x)) \leq (\alpha - \kappa) \cdot J(x) \qquad \forall x \in \mathcal{X}.		
	\end{align*}
	Hence, the application of the \citet{Gronwall1919}'s inequality results in 
	$
		\mathbb{G}_\pi^x [ J(X_t) ] \leq e^{(\alpha - \kappa) t} \cdot J(x)
	$ for all $x \in \mathcal{X}$. That is, 
	\begin{align*}
		e^{-\alpha t} \cdot \ushort J  \leq \mathbb{G}_\pi^x [ e^{-\alpha t} J(X_t) ] \leq e^{-\kappa t}\cdot J(x) \qquad \forall x \in \mathcal{X}
	\end{align*}
	where $\ushort J \in \mathbb{R}$ is a lower-bound of $J$. Take $\ushort J = 0$ if $\gamma = 1$ (note: $v$ ($=-J$) is assumed upper-bounded, \emph{by zero if $\gamma = 1$}). Then, since $\kappa > 0$, $\alpha \geq 0$, and $\ushort J = 0$ whenever $\alpha = 0$, both left and right sides converge to zero as $t \to \infty$, resulting in 
	\[
		\lim_{t \to \infty} \mathbb{G}_\pi^x [ e^{-\alpha t} J(X_t) ] = 0 \qquad x \in \mathcal{X},
	\]
	which implies the boundary condition~\eqref{eq:uniqueness condition for v} since $\gamma = e^{-\alpha}$ and $J = -v$. 
	
	\textbf{b.} Since we assume $v$ is upper bounded (by zero if $\alpha = 0$), the inequality implies that
	\begin{align*}
		\forall x \in \mathcal{X}: \;\, e^{-(\alpha - \ushort \alpha) t} \cdot \zeta(x) \leq \mathbb{G}_\pi^x \big  [ \, e^{-\alpha t} \cdot v(X_{t}) \, \big ] \leq \bar \upsilon  \cdot e^{-\alpha t} \;\, \text{ for all } t \in \mathbb{T},
	\end{align*}
	where $\bar  \upsilon \in \mathbb{R}$ is an upper-bound of $v$. Here, $\bar  \upsilon$ is finite and, if $\alpha = 0$, zero. Since $\alpha - \ushort \alpha > 0$, $\alpha \geq 0$, and $\bar \upsilon = 0$ whenever $\alpha = 0$, both left and right sides converge to zero as $t \to \infty$, resulting in 
	$
		\lim_{t \to \infty} \mathbb{G}_\pi^x \big  [ \, e^{-\alpha t} \cdot v(X_{t}) \, \big ] =  0$ for all $x \in \mathcal{X}
	$.
\end{Proof}

\begin{lemma}
	\label{lemma:appendix:conversion b.t.w. BEs for arbitrary eta}
	If $v \in \mathrm{C}^1$ satisfies 	either the integral BE~\eqref{eq:Bellman eq for v} for arbitrarily small $\eta > 0$ or the differential BE~\eqref{eq:differential BE for v}, then 
	\begin{equation}
		\alpha \cdot v(x) = r_\pi(x) + {\dot v}(x, \pi(x))
		\qquad \forall x \in \mathcal{X}.
		\label{eq:appendix:differential BE with dot v}
	\end{equation}	
\end{lemma}

\begin{proof}
	If $v \in \mathrm{C}^1$ satisfies the integral BE~\eqref{eq:Bellman eq for v} for arbitrary small $\eta > 0$, then rearranging the BE as 
	\begin{align*}
		\big ( 1 - \gamma^{\eta} \big ) \cdot v(x) =
		\mathbb{G}_\pi^x \Big [ \mathfrak{R}_{\eta} + \gamma^{\eta} \! \cdot \! \big (
		v(X_\eta) - v(X_{0}) \big ) \Big ] \qquad \forall x \in \mathcal{X},
	\end{align*}
	dividing it by $\eta$, and limiting $\eta \to 0$ yields \eqref{eq:appendix:differential BE with dot v}. On the other hand, the differential BE~\eqref{eq:differential BE for v} for $v \in \mathrm{C}^1$ is also equivalent to \eqref{eq:appendix:differential BE with dot v} by $h(x, u, \nabla v(x)) = r(x, u) + {\dot v}(x, u)$ (see the definition \eqref{eq:def:Hamiltonian} of Hamiltonian~$h$ and note that ${\dot v}(x, u) = \nabla v(x) f(x, u)$). 
\end{proof}

Combining Lemmas~\ref{lemma:boundary condition:sufficiency} and \ref{lemma:appendix:conversion b.t.w. BEs for arbitrary eta} with Theorem~\ref{thm:uniqueness condition}, we obtain the following theorem, in which the boundary condition \eqref{eq:uniqueness condition for v} is not assumed but proven to be true by Lemmas~\ref{lemma:boundary condition:sufficiency} and \ref{lemma:appendix:conversion b.t.w. BEs for arbitrary eta}, under the given conditions.

\begin{theorem}
	\label{thm:boundary condition:sufficiency}
	Suppose $v$ is upper bounded (by zero if $\gamma = 1$) and satisfies $r_\pi \leqslant \kappa \cdot v$ for a policy $\pi$ and a constant $\kappa > 0$. Then, $\pi$ is admissible and $v = v_\pi$ if one of the followings (\emph{\textbf{a}} or \emph{\textbf{b}}) is true.
	\begin{enumerate}
		\item [\emph{\textbf{a.}}] $v$ is $\mathrm{C}^1$ and satisfies either the integral BE~\eqref{eq:Bellman eq for v} for arbitrarily small $\eta > 0$ or the differential BE~\eqref{eq:differential BE for v};\\[-5pt]
		\item [\emph{\textbf{b.}}] 
			$
			\begin{cases}
				\text{$v$ satisfies the integral BE~\eqref{eq:Bellman eq for v} for a fixed horizon $\eta > 0$;} \\[2.5pt]
				\text{there exist a function $\xi: \mathcal{X} \to \mathbb{R}$ and a constant $\ushort \alpha < \alpha$, both possibly depending on $\pi$, s.t. for all $x\in \mathcal{X}$,} 
			\end{cases}
			$
			\begin{align}
					\mathbb{G}_\pi^x [R_t] \geq \xi(x) \cdot \exp(\ushort \alpha \hspace{0.1em} t)
					 \; \text{ for all } t \in \mathbb{T}.
					\tag{\ref{eq:exponential increasing condition for admissibility}}
			\end{align}			
	\end{enumerate}
\end{theorem}

\begin{Proof}
	For both cases, $\pi$ is admissible by Theorem~\ref{thm:general policy evaluation thm}, and we prove $v = v_\pi$ for each case as follows.

	\textbf{a.} $v \in \mathrm{C}^1$ satisfies \eqref{eq:appendix:differential BE with dot v} by Lemma~\ref{lemma:appendix:conversion b.t.w. BEs for arbitrary eta}, hence substituting the inequality $r_\pi \leqslant \kappa \cdot v$ into \eqref{eq:appendix:differential BE with dot v} yields 
	\[
		\alpha \cdot v(x) \leq \kappa \cdot v(x) + {\dot v}(x, \pi(x)) \qquad \forall x \in \mathcal{X},
	\]
	and the application of Lemma~\ref{lemma:boundary condition:sufficiency}a and Theorem~\ref{thm:uniqueness condition} concludes $v = v_\pi$.
	
	\textbf{b.} By $r_\pi \leqslant \kappa \cdot v$ and \eqref{eq:exponential increasing condition for admissibility}, the following inequality holds:
	\[
		\mathbb{G}_\pi^x [v(X_t)] \geq \kappa^{-1} \cdot \mathbb{G}_\pi^x [R_t] \geq {\zeta}(x) \cdot \exp(\ushort \alpha \hspace{0.1em} t)
		\qquad \forall x \in \mathcal{X},
	\] 
	where $\zeta \doteq \kappa ^{-1} \cdot \xi$ and we substituted $\mathbb{G}_\pi^x [ r_\pi(X_t) ] = \mathbb{G}_\pi^x [ R_t ]$. Therefore, $v = v_\pi$ by Lemma~\ref{lemma:boundary condition:sufficiency}b and Theorem~\ref{thm:uniqueness condition}.
\end{Proof}

In Theorem~\ref{thm:boundary condition:sufficiency}a, the integral BE~\eqref{eq:Bellman eq for v} can replace the differential BE~\eqref{eq:differential BE for v}, but only when $\eta > 0$ is arbitrary small. If the BE~\eqref{eq:Bellman eq for v} is true for a \emph{fixed} $\eta > 0$, then Theorem~\ref{thm:boundary condition:sufficiency}b suggests an additional condition for $v = v_\pi$, i.e., the inequality \eqref{eq:exponential increasing condition for admissibility}. We note that for $\gamma \in (0, 1)$ (i.e., $\alpha > 0$), the lower-bound \eqref{eq:exponential increasing condition for admissibility} on $R_t$ is true for any policy~$\pi$ that makes every state trajectory (i) bounded or (ii) even diverge exponentially with the rate smaller than $\alpha$. For $\gamma = 1$ (i.e., $\alpha = 0$ and $r_\mathsf{max} = 0$ --- see \secref{subsection:RL Problem}), the inequality \eqref{eq:exponential increasing condition for admissibility} implies exponential convergence $R_t \to 0$.

The conditions in Theorem~\ref{thm:boundary condition:sufficiency} are particularly related to the optimal control framework in \secref{subsection:nonlinear optimal control} but can be also applied to any case in our work to replace the boundary conditions \eqref{eq:uniqueness condition for v} and \eqref{eq:boundary condition for PI}. For example, the boundary condition \eqref{eq:boundary condition for PI} can be replaced by the following one(s):
	\begin{enumerate}
		\item $v_i$ is upper-bounded (by zero if $\gamma = 1$) and satisfies $r_{\pi_{i-1}} \leqslant \kappa_i \cdot v_i$ for a constant $\kappa_i > 0$;
		\item for IPI, either 
			$
			\begin{cases}
				\text{$v_i$ satisfies the integral BE~\eqref{eq:integral BE in IPI} therein, \emph{for arbitrary small} $\eta > 0$, or  }
				\\[2.5pt]
				\text{$\exists$a function $\xi_i: \mathcal{X} \to \mathbb{R}$ and a constant $\ushort \alpha_i < \alpha$ s.t. $
					\mathbb{G}_{\pi_{i-1}}^x [R_t] \geq \xi_i(x) \cdot \exp(\ushort \alpha_i \hspace{0.1em} t)$ for all $(x, t) \in \mathcal{X} \times \mathbb{T}$.}
			\end{cases}
			$
	\end{enumerate}
Theorem~\ref{thm:boundary condition:sufficiency} under the above condition(s) can replace Theorem~\ref{thm:uniqueness condition} with the boundary condition~\eqref{eq:boundary condition for PI}, in the proofs and statements of Theorems (e.g.,  see Theorem~\ref{thm:fundamental properties in optimal control} in \secref{subsection:nonlinear optimal control}, with Theorem~\ref{thm:uniqueness condition in optimal control} and their proofs in \secref{appendix:proofs:case studies}; see also Theorem~\ref{thm:fundamental properties} in \secref{section:fundamental properties of PI} and its proof in \secref{appendix:C}).

\section{Existence and Uniqueness of the Maximal Function $u_*$}
\label{appendix:section:maximal function}

This appendix provides the details about the existence and uniqueness of the maximal function $u_*$ in \secref{subsection:policy improvement} satisfying
		\begin{equation}
				u_*(x, p) \in \Argmax_{u \in \mathcal{U}} h(x, u, p) \;\;\, \forall (x, p) \in \mathcal{X} \times \mathcal{X}^\mathsf{T},
			\tag{\ref{eq:assumption:general condition for policy improvement}}
		\end{equation}
by which a maximal policy~$\pi'$ over $\pi \in \Pi_\mathsf{a}$, defined as a continuous function $\pi':\mathcal{X} \to \mathcal{U}$ s.t.
	\begin{equation}
		\pi'(x) \in \Argmax_{u \in \mathcal{U}} h(x, u, \nabla v_\pi(x)) \;\;\, \forall x \in \mathcal{X},
		\tag{\ref{eq:assumption:policy improvement}}
	\end{equation}	
can be represented in a closed form:  
	\begin{align}
		\pi'(x) = u_* (x, \nabla v_\pi (x) ).
		\tag{\ref{eq:closed-form expression of the maximal policy under uniqueness}}
	\end{align}
\begin{enumerate}
	\item \textbf{(Existence)} If $\mathcal{U}$ is \emph{compact}, then for each $(x, p) \in \mathcal{X} \times \mathcal{X}^\mathsf{T}$, the maximum of the function $u \mapsto h(x, u, p)$ exists by continuity of the Hamiltonian function $h$ \citep[Theorem~4.16]{Rudin1964}. That is, \emph{a function $u_*: \mathcal{X} \times \mathcal{X}^\mathsf{T} \!\! \to \mathcal{U}$ satisfying \eqref{eq:assumption:general condition for policy improvement} always exists whenever $\mathcal{U}$ is compact.}
		\\[-5pt]
		
	\item \textbf{(Uniqueness)} If $\mathcal{U}$ is \emph{convex} and the function $u \mapsto h(x, u, p)$ is \emph{concave} and $\mathrm{C}^1$ for each $(x, p) \in \mathcal{X} \times \mathcal{X}^\mathsf{T}$, then the maximization \eqref{eq:assumption:general condition for policy improvement} falls into a convex optimization in which any regular point $\bar u \in \Int{\mathcal{U}}$ such that  
		\[
			\partial h(x, \bar u, p) / \partial \bar u = 0,
		\]
		if exists, belongs to the $\mathrm{argmax}$-set in \eqref{eq:assumption:general condition for policy improvement} \citep[Theorem~7.15]{Sundaram1996} and thus can be the maximal argument $u_*(x, p)$ satisfying~\eqref{eq:assumption:general condition for policy improvement}. In this case, $\pi'(x)$ in \eqref{eq:assumption:policy improvement} corresponds to a regular point $\bar u$ for $p = \nabla v_\pi(x)$. Besides, as exemplified in \secref{subsection:RL under u-AC setting}, if $u \mapsto h(x, u, p)$ is \emph{strictly concave}, then such a regular point $\bar u$, if exists, is unique, meaning that $u_*(x, p)$ in \eqref{eq:assumption:general condition for policy improvement} is determined \emph{uniquely}  \citep[Theorems~7.14 and 7.15]{Sundaram1996}, hence so is each $\pi'(x)$ by \eqref{eq:closed-form expression of the maximal policy under uniqueness}.
		\end{enumerate}

\section{Theory of Optimality}
\label{appendix:section:optimality}

In this appendix, we provide a theory of optimality regarding (i) an HJB solution $(v_*, \pi_*)$:
\begin{tagcases} [\forall x \in \mathcal{X}:]
	\displaystyle \alpha \cdot v_*(x) 
	= \max_{u \in \mathcal{U}} h(x, u, \nabla v_*(x)) &
	\tag{\ref{eq:HJBE}}
	\\[5pt]
	\displaystyle \pi_*(x)
	\in \Argmax_{u \in \mathcal{U}} h(x, u, \nabla v_*(x)) &
	\tag{\ref{eq:optimal policy}}
\end{tagcases}
and (ii) a fixed point $v^*$ of $\mathcal{T}$ (i.e., $v^* \in \mathcal{V}_\mathsf{a}$ s.t. $\mathcal{T} v^* = v^*$). Here, note that a fixed point $v^*$ of $\mathcal{T}$ is always a solution to the HJBE~\eqref{eq:HJBE} by Proposition~\ref{prop:a fixed point is a solution to HJBE} (but not vice versa). Hence, if every solution $v_*$ to the HJBE~\eqref{eq:HJBE} is proven to be optimal, then so is every fixed point $v^*$ of $\mathcal{T}$. We first state the following theorem regarding the optimality of the HJB solution $(v_*, \pi_*)$.

\newpage
\begin{theorem} [Optimality]
	\label{thm:general optimality of the HJB solution}
	If a solution $v_* \in \mathrm{C}^1$ to the HJBE \eqref{eq:HJBE} exists and is upper-bounded (by zero if $\gamma = 1$), then for any policy~$\pi_*$ satisfying \eqref{eq:optimal policy},
	\begin{enumerate}
		\item [\textbf{\emph{a}}.] $\pi_*$ is admissible and $v_* \leqslant v_{\pi_*}$;
		\item [\textbf{\emph{b}}.] $v_\pi \leqslant v_*$ if $\pi$ satisfies the boundary condition:
	\begin{equation}
		\smash{\lim_{t \to \infty}} \mathbb{G}_{\pi}^x \big  [ \, \gamma^{t} \cdot v_*(X_{t}) \, \big ] = 0 \qquad \forall x \in \mathcal{X}
		\label{eq:boundary condition with v_*}
	\end{equation}	
	(conversely, \eqref{eq:boundary condition with v_*} is true if $\pi$ is admissible and $v_\pi \leqslant v_*$);
	\item [\textbf{\emph{c}}.]
	$v_* = v_{\pi_*}$ if either (i) the boundary condition \eqref{eq:boundary condition with v_*} is true for $\pi = \pi_*$ or (ii) $r_{\pi_*} \leqslant \kappa  \cdot v_*$ holds for a constant $\kappa > 0$;
	\item [\textbf{\emph{d}}.] 
			$(v_*, \pi_*)$ is optimal if $v \leqslant v_*$ for any $v \in \mathcal{V}_\mathsf{a}$.
	\end{enumerate}
\end{theorem}

\begin{Proof}
\textbf{a}. Substituting \eqref{eq:optimal policy} into the HJBE~\eqref{eq:HJBE}, we have 
	\begin{equation}
		\alpha \cdot v_*(x) = h(x, \pi_*(x), \nabla v_*(x)) \qquad \forall x \in \mathcal{X}.		
		\label{eq:differential BE for v_*}
	\end{equation}
	Then, $\pi_*$ is admissible and $v_* \leqslant v_{\pi_*}$ by Lemma~\ref{lemma:policy improvement lemma} or Theorem~\ref{thm:general policy evaluation thm}. 
	
\textbf{b}. By the HJBE~\eqref{eq:HJBE}, $v_*$ and any policy~$\pi$ satisfy
	\[
		\alpha \cdot v_*(x) \geq h(x, \pi(x), \nabla v_*(x)) \qquad \forall x \in \mathcal{X},
	\]	
	hence if $\pi$ satisfies \eqref{eq:boundary condition with v_*}, then applying Lemma~\ref{lemma:Bellman ineq} and taking the limit $\eta \to \infty$ results in
\begin{align*}
	v_*(x) \geq \underbrace{\lim_{\eta \to \infty}  \mathbb{G}_\pi^x 
			\bigg [ 
				\int_0^{\eta}
				\gamma^{t} \cdot R_t \, dt 
			\bigg ]}_{=v_\pi(x)}
			+ \underbrace{\lim_{\eta \to \infty} \mathbb{G}_\pi^x [\gamma^{\eta} \cdot v_*(X_{\eta})]}_{=0} 
			= v_\pi(x)	\qquad \forall x \in \mathcal{X}.
\end{align*}	
Conversely, if $\pi$ is admissible and $v_\pi \leqslant v_*$, then Proposition~\ref{prop:VF satisfies the boundary condition} and the upper-boundedness of $v_*$ (by zero if $\gamma = 1$) results in
	\begin{align*}
	0 = \lim_{t \to \infty} \mathbb{G}_{\pi}^x \big  [ \, \gamma^{t} \cdot v_{\pi}(X_{t}) \, \big ] 
	\leq \lim_{t \to \infty} \mathbb{G}_{\pi}^x \big  [ \, \gamma^{t} \cdot v_*(X_{t}) \, \big ]
	\leq \sup_{x \in \mathcal{X}} v_*(x) \cdot \lim_{t \to \infty} \gamma^{t} \leq 0
	\end{align*}
	implying the boundary condition \eqref{eq:boundary condition with v_*}. 

\textbf{c}. The application of Theorems~\ref{thm:uniqueness condition} and \ref{thm:boundary condition:sufficiency}a to \eqref{eq:differential BE for v_*} directly proves $v_* = v_{\pi_*}$ under the respective conditions. 
	
\textbf{d}. The first part ``\textbf{a}'' and the condition ``$v \leqslant v_*$ for any $v \in \mathcal{V}_\mathsf{a}$'' imply that $\pi_*$ is admissible and $v \leqslant v_* \leqslant v_{\pi_*}$ for any $v \in \mathcal{V}_\mathsf{a}$; substituting $v = v_{\pi_*}$ results in $v_* = v_{\pi_*}$, which and the condition completes the proof.
\end{Proof}

Under the upper-boundedness of $v_* \in \mathrm{C}^1$ in Theorem~\ref{thm:general optimality of the HJB solution}, any policy $\pi_*$ given by \eqref{eq:optimal policy} dominates all policies $\pi$'s s.t. the boundary condition \eqref{eq:boundary condition with v_*} holds. On the other hand, certain additional conditions (e.g., \eqref{eq:boundary condition with v_*} holds for all admissible policies $\pi$'s) are required for the optimality condition ``$v \leqslant v_*$ for all $v \in \mathcal{V}_\mathsf{a}$'' in Theorem~\ref{thm:general optimality of the HJB solution}d to be true (e.g., see case studies in \secsref{appendix:subsection:optimality case studies} and \ref{subsection:discounted RL with bounded state trj})

\subsection{Sufficient Conditions for Optimality}
\label{subsection:sc for optimality}

Based on the properties of PIs --- convergence (Theorems~\ref{thm:convergence of PI}, \ref{thm:true convergence in metric}, \ref{thm:true uniform convergence:1}, and \ref{thm:true uniform convergence:2}) and monotonicity (Theorem~\ref{thm:fundamental properties}) --- we provide sufficient conditions for optimality, where the notion of ``optimality'' can be interpreted in a weaker sense than or in a similar manner to that shown in Theorem~\ref{thm:nc for optimality} (e.g., see \eqref{eq:strong optimality} below). In the latter case, once $v_*$ is the optimal VF, any policy $\pi_*$ satisfying \eqref{eq:optimal policy} comes to be optimal ($\because$ $v_* \leqslant v_{\pi_*}$ by Theorem~\ref{thm:policy improvement thm} and $v_{\pi_*} \leqslant v_*$ by optimality, hence  $v_* = v_{\pi_*}$). Specifically, we establish the notions of weak and strong optimality along with the following convergence properties introduced in \secref{subsection:convergence}: 
\begin{enumerate} [leftmargin=1.0cm, itemsep=0.25em, topsep=-0.65em]
	\item [($\mathsf{C}1$)] \textbf{(weak convergence)} $\mathcal{T}^{i - 1} v \to v_*$ in a metric;
	\item [($\mathsf{C}2$)] \textbf{(strong convergence)} $\mathcal{T}^{i - 1} v \to v_*$ locally uniformly;
	\item [($\mathsf{C}3$)] \textbf{(additional convergence)} $\nabla (\mathcal{T}^{i - 1} v) \to \nabla v_*$ locally uniformly and $\pi_i \to \pi_*$ pointwise,
\end{enumerate}
where we replaced $v_{i}$ with $\mathcal{T}^{i-1} v$ and $v_1 = v$.

First, we show that Assumption~\ref{assumption:uniqueness of fixed point} alone is sufficient for $v^*$ therein to be weak optimal, i.e., optimal in a metric. Note that $v^*$ is a solution $v_*$ to the HJBE~\eqref{eq:HJBE} by Proposition~\ref{prop:a fixed point is a solution to HJBE}. 

\begin{corollary}
	\label{cor:weak optimality}
		Under Assumption~\ref{assumption:uniqueness of fixed point}, there exists a metric $d$ on $\mathcal{V}_\mathsf{a}$ s.t. $\mathcal{T}$ is a contraction under $d$ and for every $v \in \mathcal{V}_\mathsf{a}$, 
		\begin{equation}
			v \leqslant \mathcal{T} v \leqslant \mathcal{T}^2 v \leqslant \cdots \leqslant \mathcal{T}^{i-1} v \; \smash{\xrightarrow{{\scriptsize i \to \infty}}} \; v^*,
			\label{eq:weak optimality}			
		\end{equation}
		where the convergence is in the metric~$d$.
\end{corollary}

\begin{Proof}
	Apply Theorems~\ref{thm:fundamental properties} and \ref{thm:true convergence in metric}.
\end{Proof}

Corollary~\ref{cor:weak optimality} characterizes $v^*$ as the optimal VF in the weak sense ($\mathsf{C1}$) --- as the unique limit point in a metric~$d$, of every monotonically increasing sequence of VFs generated by applying $\mathcal{T}$ recursively (or one of the PI methods). Under the metric~$d$, $\mathcal{T}$ is continuous since it is a contraction. 

Although the weak optimality of $v^*$ in Corollary~\ref{cor:weak optimality} looks reasonable, the downside is that convergence \eqref{eq:weak optimality} and continuity of $\mathcal{T}$ are w.r.t. an \emph{unknown} metric~$d$. With continuity of $\mathcal{T}$ under the uniform pseudometric $d_\Omega$, a stronger characterization of $v^*$ is possible, as shown in the next corollary. 

\begin{corollary}
	\label{cor:strong optimality:1}
		If $
			\lim_{i \to \infty} \mathcal{T}^{i-1} v \in \mathcal{V}_\mathsf{a} 
		$ for every $v \in \mathcal{V}_\mathsf{a}$ and for each compact subset $\Omega$ of $\mathcal{X}$, $\mathcal{T}$ is continuous under $d_\Omega$, then under Assumption~\ref{assumption:uniqueness of fixed point}, $v \leqslant v^*$ for every $v \in \mathcal{V}_\mathsf{a}$. 
\end{corollary}

\begin{Proof}
	Note that ${\hat v}_* = \lim_{i \to \infty} v_i = \lim_{i \to \infty} \mathcal{T}^{i-1} v_1 \in \mathcal{V}_\mathsf{a}$ pointwise by Theorem~\ref{thm:convergence of PI}a. Therefore, we have ${\hat v}_* \in \mathcal{V}_\mathsf{a}$ and the application of Theorems~\ref{thm:fundamental properties} and \ref{thm:true uniform convergence:1} for each $v$ ($= v_1$) $\in \mathcal{V}_\mathsf{a}$ completes the proof.
\end{Proof}

Under the given conditions on $\mathcal{T}$, Corollary~\ref{cor:strong optimality:1} states that $v^*$ in Assumption~\ref{assumption:uniqueness of fixed point} is truly the optimal VF over the space $\mathcal{V}_\mathsf{a}$ of all admissible VFs. This characterization of optimality: 
\begin{align}
	v^* \in \mathcal{V}_\mathsf{a} \text{ and } v \leqslant v^* \text{ for every } v \in \mathcal{V}_\mathsf{a}	
	\label{eq:strong optimality}
\end{align}
is exactly the same as that in Theorem~\ref{thm:nc for optimality} and obviously stronger than that in Corollary~\ref{cor:weak optimality}. Conversely, \eqref{eq:strong optimality} implies that $v^*$ is a fixed point of $\mathcal{T}$, as shown in Proposition~\ref{prop:fixed point and HJBE}a below.

\begin{proposition}
	\label{prop:fixed point and HJBE}
	$
	\begin{cases}
		\emph{\text{\textbf{a.}} If $v_* \in \mathcal{V}_\mathsf{a}$ is the optimal VF, then it is a fixed point of $\mathcal{T}$.}
		\\[5pt]
		\emph{\text{\textbf{b.}} The fixed point of $\mathcal{T}$ is unique over $\mathcal{V}_\mathsf{a}$ if so is the solution of the HJBE~\eqref{eq:HJBE}.}
	\end{cases}
	$
\end{proposition}

\begin{Proof}
	\textbf{a.} Let $v_* \in \mathcal{V}_\mathsf{a}$ be the optimal VF. Then, it satisfies the HJBE \eqref{eq:HJBE} by Theorem~\ref{thm:nc for optimality}, hence we have $v_* \leqslant \mathcal{T} v_*$ by Theorem~\ref{thm:policy improvement thm}. By optimality, $\mathcal{T} v_* \leqslant v_*$ is obvious. Therefore, $v_* = \mathcal{T} v_*$, i.e., $v_*$ is a fixed point of $\mathcal{T}$. \textbf{b.} Suppose $v_*$ is the unique solution to the HJBE~\eqref{eq:HJBE}, but there exists another VF $v_*' \neq v_*$ s.t. $v_*' = \mathcal{T} v_*'$. Then, $v_*'$ is a solution to the HJBE by Proposition~\ref{prop:a fixed point is a solution to HJBE} and thus by the uniqueness, $v_*' = v_*$, a contradiction. Therefore, if $v_*$ is a unique solution to the HJBE~\eqref{eq:HJBE} over $\mathcal{V}_\mathsf{a}$, then it is the unique fixed point of $\mathcal{T}$.
\end{Proof}

By Proposition~\ref{prop:fixed point and HJBE}b, the uniqueness of the fixed point $v^*$ of $\mathcal{T}$ can be replaced by that of the solution $v_*$ to the HJBE over $\mathcal{V}_\mathsf{a}$, and we have the following corollary that extends Theorems~\ref{thm:fundamental properties} and \ref{thm:true uniform convergence:2}.
	
\begin{corollary}
		\label{cor:strong optimality:2}
		Suppose that Assumption~\ref{assumption:convergence of nabla v_i and pi_i} holds for any initial admissible policy~$\pi_0$. Then, under Assumptions~\ref{assumption:closed graph} and \ref{assumption:uniqueness of HJB over C1}, the HJBE~\eqref{eq:HJBE} has a unique  solution $v_*$ over $\mathrm{C}^1$ s.t. 
		\begin{enumerate}
			\item the strong optimality \eqref{eq:strong optimality} and Assumption~\ref{assumption:uniqueness of fixed point} are true for $v^* = v_*$;
			\item for each initial admissible policy~$\pi_0$, there exists a function~$\pi_*$ s.t. \eqref{eq:optimal policy} holds and the generated VFs and policies satisfy the stronger convergence, i.e., \emph{($\mathsf{C}2$)} and \emph{($\mathsf{C}3$)}.
		\end{enumerate}
\end{corollary}

\begin{Proof}
	Theorems~\ref{thm:fundamental properties} and \ref{thm:true uniform convergence:2} imply that for a given admissible initial policy~$\pi_0$, there exists a solution $v_* \in \mathrm{C}^1$ to the HJBE \eqref{eq:HJBE} s.t. (i) \smash{$v_{\pi_0} \leqslant v_*$} and (ii) the convergence ($\mathsf{C}2$) and ($\mathsf{C}3$) hold for a function $\pi_*$ satisfying \eqref{eq:optimal policy}. Since the solution~$v_*$ is now unique over $\mathrm{C}^1$ by Assumption~\ref{assumption:uniqueness of HJB over C1} and $\pi_0$ is arbitrary, the former implies that $v \leqslant v_*$ for any $v \in \mathcal{V}_\mathsf{a}$. Moreover, by Theorem~\ref{thm:general optimality of the HJB solution}d, $v_*$ is the optimal VF and thus satisfies the strong optimality~\eqref{eq:strong optimality} for $v^* = v_*$. Since $\mathcal{V}_\mathsf{a} \subset \mathrm{C}^1$ by \eqref{eq:assumption:every admissible VF is C1}, $v_*$ is the unique solution of the HJBE over $\mathcal{V}_\mathsf{a}$ ($\subset \mathrm{C}^1$). Therefore, $v_*$ is the unique fixed point of $\mathcal{T}$ (i.e.,  Assumption~\ref{assumption:uniqueness of fixed point} holds for $v^* = v_*$) by Proposition~\ref{prop:fixed point and HJBE}, which completes the proof.
\end{Proof}

\newpage 

Under the given conditions in
Corollary~\ref{cor:strong optimality:1} or \ref{cor:strong optimality:2}, \eqref{eq:weak optimality} holds with \emph{locally uniform convergence} ($\mathsf{C}2$) (apply Theorem~\ref{thm:fundamental properties} for monotonicity) --- stronger than convergence ($\mathsf{C}1$) in a metric shown in Corollary~\ref{cor:weak optimality}. In addition, Corollary~\ref{cor:strong optimality:2} provides the additional convergence ($\mathsf{C}3$) without employing the PI operator $\mathcal{T}$ and any assumptions imposed on it. We note that even the stronger (i.e., locally uniform) convergence of $\langle \pi_i \rangle$ towards $\pi_* \in \Pi_\mathsf{a}$ can be obtained in the concave Hamiltonian formulation in \secref{subsection:RL under u-AC setting}, with both Assumptions~\ref{assumption:closed graph} and \ref{assumption:convergence of nabla v_i and pi_i}b for any $\pi_0 \in \Pi_\mathsf{a}$ in Corollary~\ref{cor:strong optimality:2} \emph{relaxed} (see Corollary~\ref{cor:strong optimality:concave h case} below).

In summary, we characterized $v_*$ in the Corollaries as a \emph{unique} VF to which $\langle \mathcal{T}^{i - 1}v\rangle$ for any $v \in \mathcal{V}_\mathsf{a}$ monotonically converges (i.e., satisfies \eqref{eq:weak optimality} with $v^* = v_*$) in their respective manners, 
where $v_*$ is assumed to be a unique fixed point of $\mathcal{T}$ (Corollaries~\ref{cor:weak optimality} and \ref{cor:strong optimality:1}) or a unique solution of the HJBE~\eqref{eq:HJBE} (Corollary~\ref{cor:strong optimality:2}). Here, the uniqueness is truly necessary --- otherwise, some sequence of VFs generated by PIs may converge to another VF $v_*' \neq v_*$. In this case, the optimality of $v_*$ becomes vague and not decidable unless $v_*' \leqslant v_*$ for any of such VFs $v_*'$. Since an optimal VF~$v_*$ is unique over $\mathcal{V}_\mathsf{a}$ as discussed in \secref{subsection:HJBE}, any two different VFs $v_*, v_*' \in \mathcal{V}_\mathsf{a}$ cannot be the optimal at the same time. 

A similar characterization of $v_*$ is possible without the assumptions and conditions imposed in the Corollaries, including the uniqueness of $v_*$ and $v^*$, but by proving or imposing (i) the boundary condition \eqref{eq:boundary condition with v_*} for a class of policies and (ii) one of the two conditions on $(v_*, \pi_*$) in Theorem~\ref{thm:general optimality of the HJB solution}c. This approach will be employed in the next subsection  (\secref{appendix:subsection:optimality case studies}) to characterize the optimality of $v_*$ (and $\pi_*$) under the given respective frameworks.

\subsection{Case Studies of Optimality}
\label{appendix:subsection:optimality case studies}

We now provide and discuss the condition(s) for optimality of the HJB solution $(v_*, \pi_*)$ under certain classes of RL problems shown in \secref{section:case studies} Case Studies --- specifically, the cases presented in \secsref{subsection:RL under u-AC setting}, \ref{subsection:discounted RL with bounded v}, and \ref{subsection:nonlinear optimal control}. 

\vspace{1.5em}

\textbf{Concave Hamiltonian Formulation (\secref{subsection:RL under u-AC setting}).} Under \eqref{eq:input-affine dynamics} and \eqref{eq:strictly concave reward}, Corollary~\ref{cor:strong optimality:2} can be simplified and strengthened with the assumptions on the policies and policy improvement therein \emph{relaxed}.

\begin{corollary}
	\label{cor:strong optimality:concave h case}
		If Assumption~\ref{assumption:convergence of nabla v_i and pi_i}a holds for any initial admissible policy~$\pi_0$, then under \eqref{eq:input-affine dynamics}, \eqref{eq:strictly concave reward}, and Assumption~\ref{assumption:uniqueness of HJB over C1}, there exists a unique HJB solution $(v_*, \pi_*)$ over $\mathcal{V}_\mathsf{a} \times \Pi_\mathsf{a}$ s.t. Assumption~\ref{assumption:uniqueness of fixed point} holds for $v^* = v_*$, $\pi \preccurlyeq \pi_*$ for all $\pi \in \Pi_\mathsf{a}$, $v_* = v_{\pi_*}$, and for any initial admissible policy~$\pi_0$, $v_i \to v_*$, $\nabla v_i \to \nabla v_*$, and $\pi_i \to \pi_*$, all locally uniformly.
\end{corollary}

\begin{Proof}
	Combine Lemma~\ref{lemma:concave h} with Corollary~\ref{cor:strong optimality:2}. Also note that the HJB policy $\pi_*$ satisfying \eqref{eq:optimal policy} is uniquely determined under \eqref{eq:input-affine dynamics} and \eqref{eq:strictly concave reward} by $\pi_*(x) = \sigma(F_\mathsf{c}^\mathsf{T} (x) \nabla v_*^\mathsf{T}(x))$--- see \secref{subsubsection:case I} for details.
\end{Proof}

Here, we have directly extended Corollary~\ref{cor:strong optimality:2} to \ref{cor:strong optimality:concave h case} above in the same way as extending Theorem~\ref{thm:true uniform convergence:2} to \ref{thm:true convergence}, by applying Lemma~\ref{lemma:concave h} under the concave Hamiltonian formulation \eqref{eq:input-affine dynamics} and \eqref{eq:strictly concave reward}. Therefore, as discussed in Remark~\ref{remark:generalization of u-AC RL problem:1} and \secref{subsubsection:case II}, Corollary~\ref{cor:strong optimality:concave h case} (specifically, Lemma~\ref{lemma:concave h}) can be further extended to 
\begin{enumerate}
	\item the input-affine case where the reward function $r$ satisfies the conditions in Remark~\ref{remark:generalization of u-AC RL problem:1} and $(x, \mathfrak{u}) \mapsto \sigma^x(\mathfrak{u})$ is continuous;
	\item the non-affine case \eqref{eq:a class of non-affine dynamics} and \eqref{eq:a class of reward fnc for non-affine dynamics} in a similar manner to Theorem~\ref{thm:true convergence:nonaffine case}, with $\varphi$ and $\mathfrak{c}$ possibly depending on the state $x \in \mathcal{X}$.
\end{enumerate}

\vspace{2em}

\textbf{Discounted RL Problems with Bounded VFs (\secsref{subsection:discounted RL with bounded v}).} In this case, we can dramatically improve the optimality theory with respect to the solution $v_*$ to the HJBE~\eqref{eq:HJBE} and the HJB policy  $\pi_*$ in \eqref{eq:optimal policy} (of course, under the Assumptions made in \secref{section:preliminaries}).

\begin{theorem} 
	\label{thm:bounded HJB solution}
	Let $\gamma \in (0, 1)$. If the HJBE~\eqref{eq:HJBE} has a bounded $\mathrm{C}^1$ solution $v_*$, then for any HJB policy $\pi_*$ satisfying \eqref{eq:optimal policy}, 
	\begin{enumerate} 
		\item $v_{\pi_*}$ is bounded (hence, admissible) and $v_* = v_{\pi_*}$;
		\item $\pi \preccurlyeq \pi_*$ for any policy~$\pi$.
	\end{enumerate}
	Moreover, $v_*$ is the unique solution to the HJBE~\eqref{eq:HJBE} over all bounded $\mathrm{C}^1$ functions $v: \mathcal{X} \to \mathbb{R}$.
\end{theorem}

\begin{Proof}
	The first two parts can be proven by applying Proposition~\ref{prop:boundary condition true when discounted and bounded} with $v = v_*$ and Theorem \ref{thm:general optimality of the HJB solution}b and c. For the uniqueness of $v_*$, note that if $v_*'$ is another bounded $\mathrm{C}^1$ solution to the HJBE, then we have $v_* \leqslant v_*'$ and $v_*' \leqslant v_*$, hence $v_* = v_*'$.
\end{Proof}

\newpage 

\textbf{Nonlinear Optimal Control (\secref{subsection:nonlinear optimal control}).} Under the assumptions and notations in \secref{subsection:nonlinear optimal control}, the optimality of an HJB solution $(J_*, \pi_*)$, with $J_* \doteq - v_*$, can be characterized as follows, without assuming the existence of the unique state trajectories.

\begin{theorem} 
	\label{thm:optimality in nonlinear optimal control}
	Under the assumptions and notations in \secref{subsection:nonlinear optimal control}, if there exists an HJB solution $(v_*, \pi_*)$ of \eqref{eq:HJBE} and \eqref{eq:closed-form expression of optimal policy} s.t.
	\begin{enumerate}
		\item $J_*$ is $\mathrm{C}^1_\mathsf{Lip}$, positive definite, and radially unbounded;
		\item if $\gamma \neq 1$, then either $x_\mathsf{e} = 0$ under $\pi_*$ is globally attractive or $\kappa_* J_* \leqslant c_{\pi_*}$ holds for a constant $\kappa_* > 0$,
	\end{enumerate}
	then, $\pi_* \in \Pi_\mathsf{a}$, $J_* = J_{\pi_*}$, and
	$J_* \leqslant J_\pi$ for any policy $\pi \in \Pi_0$ such that
	\begin{equation}
		t_\mathsf{max}(x; \pi) = \infty \text{ and } \smash{\lim_{t \to \infty}} \mathbb{G}_\pi^x [\gamma^{t} \cdot J_*(X_t)] = 0
		\qquad \forall x \in \mathcal{X}.
		\label{eq:boundary condition for J_*}
	\end{equation}
	Moreover, $x_\mathsf{e} = 0$ under $\pi_*$ is globally asymptotically stable if $\gamma = 1$ or 
	\begin{equation}
		\alpha J_*(x) < c_{\pi_*}(x) \qquad \forall x \in \mathcal{X} \setminus \{ 0 \}. 
		\label{eq:condition* for asymp stability}
	\end{equation}
\end{theorem}
\begin{Proof} The HJB policy $\pi_*$ satisfies \eqref{eq:closed-form expression of optimal policy}; $v_*$ ($= - J_*$) is $\mathrm{C}^1_\mathsf{Lip}$ and negative definite. Hence, $\pi_* \in \Pi_0$ by Lemma~\ref{lemma:pi'(0)=0}. Moreover, the HJBE~\eqref{eq:HJBE}, \eqref{eq:optimal policy}, and the positive definiteness of $c$, with $J_* = - v_*$ and $c = -r$, imply that 
	\[
		{\dot J}_*(x, \pi_*(x)) 
		= \alpha \cdot J_*(x) - c_{\pi_*}(x)
		\leq \alpha \cdot J_*(x)	
		\quad \forall x \in \mathcal{X}.
	\]
	$J_*$ is continuous, positive definite, and radially unbounded.  Hence, by Lemma~\ref{lemma:bound of a positive definite function}, there exist $\mathcal{K}_\infty$ functions $\rho_1$ and $\rho_2$ s.t.
	$
		\rho_1(\|x\|) \leq J_* (x) \leq \rho_2(\|x\|)
	$ for all $x \in \mathcal{X}$.
	Therefore, the application of Lemma~\ref{lemma:existence and uniqueness of state trj.} proves that $t_\mathsf{max}(x; \pi_*) = \infty$ for all $x \in \mathcal{X}$. The remaining proof is divided into the following two cases.
	\begin{enumerate}
		\item If $\gamma = 1$, then the HJBE~\eqref{eq:HJBE} and \eqref{eq:optimal policy} is reduced to 
			${\dot J}_*(x, \pi_*(x)) = -c_{\pi_*} (x)$ $\forall x \in \mathcal{X}$,
			where $c_{\pi_*}$ is positive definite by Lemma~\ref{lemma:positive definiteness of c_pi}. Therefore,
			$x_\mathsf{e} = 0$ under $\pi_*$ is globally asymptotically stable \citep[Theorem~4.2]{Khalil2002}, 
			with $J_*$ as the radially-unbounded Lyapunov function, and Theorem~\ref{thm:uniqueness under gas} results in $\pi_* \in \Pi_\mathsf{a}$ and $J_* = J_{\pi_*}$.
		\item For $\gamma \neq 1$, we first prove $\pi_* \in \Pi_\mathsf{a}$ and $J_* = J_{\pi_*}$, then global asymptotic stability. If $x_\mathsf{e} = 0$ under $\pi_*$ is globally attractive, then Theorem~\ref{thm:uniqueness under gas} proves $\pi_* \in \Pi_\mathsf{a}$ and $J_* = J_{\pi_*}$. Otherwise, if $\kappa_* J_* \leqslant c_{\pi_*}$ holds for some $\kappa_* > 0$, then $\pi_* \in \Pi_\mathsf{a}$ and $J_* = J_{\pi_*}$ by Theorem~\ref{thm:uniqueness condition in optimal control}a. Here, note that the HJBE~\eqref{eq:HJBE} and \eqref{eq:optimal policy} imply the differential BE~\eqref{eq:differential BE for v} for $v = v_*$ and $\pi =  \pi_*$. Now that $\pi_* \in \Pi_\mathsf{a}$ and $J_* = J_{\pi_*}$, $x_\mathsf{e} = 0$ under $\pi_*$ is globally asymptotically stable if $\alpha J_*(x) < c_{\pi_*}(x)$ for all $x \in \mathcal{X} \setminus \{ 0 \}$, by Theorem~\ref{thm:stability of optimal control} and the radial unboundedness of $J_*$.
	\end{enumerate}
	For any case, we have $\pi_* \in \Pi_\mathsf{a}$, $J_* = J_{\pi_*}$, and global asymptotic stability under the given conditions. Moreover, $J_* \leqslant J_\pi$ for any policy $\pi \in \Pi_0$ s.t. \eqref{eq:boundary condition for J_*} holds, by Theorem~\ref{thm:general optimality of the HJB solution}b. So, the proof is completed.
\end{Proof}

The conditions on $(J_*, \pi_*)$ in Theorem~\ref{thm:optimality in nonlinear optimal control}  can be considered a limit version of the three conditions presented in \secref{subsection:nonlinear optimal control}: 
\begin{enumerate}
	\item [\textsf{(A)}] $\pi_0 \in \Pi_\mathsf{a}$,
	\item [\textsf{(B)}] $J_i \in \mathrm{C}^1_\mathsf{Lip}$ is positive definite and radially unbounded,
	\item [\textsf{(C)}] if $\gamma \neq 1$, then either (i) $x_\mathsf{e} = 0$ under $\pi_{i-1}$ is globally attractive, or (ii) there exists $\kappa_i > 0$ s.t. $\kappa_i \! \cdot \! J_i \leqslant c_{\pi_{i-1}}$.
\end{enumerate}
So, similarly to the inequality in \textsf{(C)}, if $\kappa_* < \alpha$, then the inequality $\kappa_* J_* \leqslant c_{\pi_*}$ is weaker than both of the stability conditions $\alpha J_* \leqslant c_{\pi_*}$ and \eqref{eq:condition* for asymp stability} corresponding to \eqref{eq:condition for stability} and \eqref{eq:condition for asymp stability}, respectively. 

\MyRemark{Suppose $J_*$ is $\mathrm{C}^1$ and positive definite. Then, $x_\mathsf{e} = 0$ under $\pi_* \in \Pi_0$ is asymptotically stable if \eqref{eq:condition* for asymp stability} is true.
	This is because the HJBE~\eqref{eq:HJBE}, \eqref{eq:optimal policy}, and the condition yields ${\dot J}_*(x, \pi_*(x)) = \alpha J_*(x) - c_{\pi_*}(x) < 0$ for all $x \in \mathcal{X} \setminus \{0\}$, implying the asymptotic stability under $\pi_* \in \Pi_0$ by the Lyapunov's theorem \citep[Theorem~4.1]{Khalil2002}. Note that \eqref{eq:condition* for asymp stability} is weaker than the stability condition given by \citet[Assumption~2.3]{Gaitsgory2015}:
	\[
		\kappa_* J_*(x) \leq c(x, u)  \qquad \forall (x,u) \in \mathcal{X} \times \mathcal{U},
		\qquad\text{for some $\kappa_* > \alpha$.}
	\]
	This inequality and the positive definiteness of $J_*$ indeed imply
	$
		\alpha J_*(x) < \kappa_* J_*(x) \leq c_{\pi_*}(x)$ for all $x \in \mathcal{X} \setminus \{0\}
	$, i.e., \eqref{eq:condition* for asymp stability},
	but not vice versa. The other condition given by \citet[Assumption~3.8]{Gaitsgory2015} for global asymptotic stability can be replaced by the radial unboundedness of  $J_*$ (see Lemma~\ref{lemma:bound of a positive definite function} in \secref{appendix:section:proofs}).}

\MyRemark{The boundary condition~\eqref{eq:boundary condition for J_*} is true for any policy $\pi \in \Pi_0$ s.t. $x_\mathsf{e} = 0$ is globally attractive (or in particular, globally asymptotically stable) as in the proof of Theorem~\ref{thm:uniqueness under gas} (see \secref{appendix:proofs:case studies}). On the other hand, when discounted, \eqref{eq:boundary condition for J_*} contains the cases where the state trajectories are (i) globally bounded as in \secref{subsection:discounted RL with bounded state trj} or (ii) even diverge exponentially such as in the discounted LQR case in \secref{subsection:LQR}.
}

\section{A Pathological Example \citep{Kiumarsi2016}}
\label{appendix:subsection:a pathological ex}

Presented in this appendix is a counter-example where the dynamics is simple but non-affine, and the design of the reward function~$r$ is critical. In this example, (i) a naive choice of $r$ fails to give a closed-form solution of policy improvement and the HJBE; (ii) in the unconstrained case, such a choice results in a pathological Hamiltonian $h$ such that the solutions (i.e., $\pi'$ in \eqref{eq:assumption:policy improvement} for $\pi \in \Pi_\mathsf{a}$, $v_*$ in the HJBE~\eqref{eq:HJBE}, and $\pi_*$ in \eqref{eq:optimal policy}) do not exist. We encourage the readers to review \secref{appendix:section:maximal function} beforehand. See also \secref{subsubsection:case II} for a technique to avoid such a pathological behavior.

Consider the scalar dynamics ($l = m = 1$) with the action space $\mathcal{U} = [-u_\mathsf{max}, u_\mathsf{max}]$ for $u_\mathsf{max} \in (0, \infty]$:
	\[
		\dot X_t = X_t^3 + U_t^3.
	\]
Suppose that the reward function $r$ given by \eqref{eq:strictly concave reward} and \eqref{eq:integral formula of S} with $\Gamma = 1$, that is, $r(x,u) = \mathfrak{r}(x) - \mathfrak{c}(u)$ for a continuous function $\mathfrak{r}: \mathcal{X} \to \mathbb{R}$ and $\mathfrak{c}: \mathcal{U} \to \mathbb{R}$ given by 
	\begin{align*}
		\mathfrak{c}(u) = \lim_{v \to u} \smash{\int_0^v} (s^{\mathsf{T}})^{-1} (\mathfrak{u}) \; d\mathfrak{u}.		
	\end{align*}
Then, the Hamiltonian $h: \mathbb{R} \times \mathcal{U} \times \mathbb{R} \to \mathbb{R}$ in this case is given by 
	\begin{equation}
		h(x, u, p) = \mathfrak{r}(x) - \mathfrak{c}(u) + p \cdot (x^3 + u^3). 
		\label{eq:Hamiltonian in pathological ex}				
	\end{equation}

\textbf{(Input-constrained Case)} First, we consider $\mathcal{U} = [- 1, 1]$ with $s = \tanh$. In this case, since $\mathcal{U}$ is compact, a maximal function $u_*(x, p)$ satisfying~\eqref{eq:assumption:general condition for policy improvement} for each $(x, p) \in \mathbb{R}^2$ exists (see \secref{appendix:section:maximal function}). However, a regular point $u \in (-1 , 1)$ s.t. 
	\begin{equation}
		\partial h(x, u, p) / \partial u 
		= 
		- \tanh^{-1} u + 3 p u^2 
		= 0
		\label{eq:1st order nc:pathological ex}		
	\end{equation}
	cannot be expressed in a closed form since \eqref{eq:1st order nc:pathological ex} is nonlinear in $u$.
		
\textbf{(Unconstrained Case)} Next, consider \eqref{eq:unconstrained case}, that is, $u_\mathsf{max} = \infty$ and $s(u) = u / 2$. In this case, the maximal function ${u_*}$ does not exist since $\mathfrak{c}(u) = u^2$ and thus for any $p > 0$ and $x \in \mathbb{R}$, the Hamiltonian \eqref{eq:Hamiltonian in pathological ex} satisfies
	\[
		\lim_{u \to \infty} h(x, u, p) = \lim_{u \to - \infty} h(x, u, - p) = \infty.
	\]
	Therefore, except for the trivial cases $\nabla v_\pi = 0$ and $\nabla v_* = 0$, the maximal policy $\pi'$ in \eqref{eq:assumption:policy improvement} and the solution~$v_*$ to the HJBE \eqref{eq:HJBE} (and accordingly, $\pi_*$ in \eqref{eq:optimal policy}) fail to exist since so do the maxima in those respective equations. Note that the regular points $u$ s.t. $\partial h(x, u, p) / \partial u = 0$, explicitly given by $u = 0$ and $u = 2/(3p)$, are the local maximum and the local minimum, respectively, but the global maximum does not exist in this case.
	
	The issue here is that even though $\mathfrak{c}$ is strictly convex, \emph{$h$ is not (strictly) concave} due to the cubic term $u^3$ in the dynamics $f(x, u) = x^3 + u^3$. This means that the uniqueness of $u_*$ is not guaranteed, and the existing regular points $u$'s 
	satisfying $\partial h(x, u, p) / \partial u = 0$ are not necessarily the maximum of the Hamiltonian $h(x, u, p)$ (see \secref{appendix:section:maximal function}). 

\section{Additional Case Studies}
\label{appendix:section:additional case studies}

This appendix provides additional case studies with (strong) connections to (i) the case studies in \secref{section:case studies} and (ii) the theory established in the main article \citep{Lee2020b} and \secref{appendix:section:optimality}.

\subsection{General Concave Hamiltonian Formulation}
\label{subsubsection:case III}

Here, we extend the methods and results in \secref{subsubsection:case I} to the general nonlinear system~\eqref{eq:controlled system}. The core idea is to introduce a continuous bijection $\psi: \Int{\mathcal{U}} \to \mathbb{R}^m$ (which has a continuous inverse $\psi^{-1}$ by Lemma~\ref{lemma:continuity of the inverse}) and an $m$-dimensional action-dynamics: 
	\begin{equation}
		\mathfrak{\dot U}_t = A_t, \;\;\, A_t \in \mathcal{A}
		\label{eq:generalization with action dynamics:1}		
	\end{equation}
	where $\mathcal{A} \subseteq \mathbb{R}^m$ is an action space, and the differential action trajectory $t \mapsto A_t$ is a continuous function from $\mathbb{T}$ to $\mathcal{A}$, determining the rate of change of $\mathfrak{U}_t$ for all $t \in \mathbb{T}$, by \eqref{eq:generalization with action dynamics:1}; the effective action $\mathfrak{U}_t  \in \mathbb{R}^m$ generates the real action $U_t$ by 
	\begin{equation}
		U_t = \psi^{-1}(\mathfrak{U}_t).
		\label{eq:generalization with action dynamics:2}		
	\end{equation}
	Under \eqref{eq:generalization with action dynamics:1} and \eqref{eq:generalization with action dynamics:2}, the results for the concave Hamiltonian formulation in \secref{subsection:RL under u-AC setting} can be applied to the RL problem with the following affine dynamics 
	\begin{align*}
		\begin{bmatrix}
				{\dot X}_t \\
				\mathfrak{\dot U}_t 
		\end{bmatrix}
		=
		\begin{bmatrix}
			f(X_t, \psi^{-1}(\mathfrak{U}_t)	) \\
			0
		\end{bmatrix}
		+
		\begin{bmatrix}
			0	\\
			I
		\end{bmatrix}
		A_t, 		
	\end{align*}
	with $(X_t, \mathfrak{U}_t) \in \mathbb{R}^{l \times m}$ considered as its state and $A_t \in \mathcal{A}$ as the action, and the extended reward function $r_\mathsf{e}$:
	\begin{align*}
		r_\mathsf{e} (x, \mathfrak{u}, a) \doteq r(x, \psi^{-1}(\mathfrak{u})) - \mathfrak{c}(a), 		
	\end{align*}
	where the real action $U_t \in \mathcal{U}$ is determined by \eqref{eq:generalization with action dynamics:2}. Here, $\mathfrak{c}:\mathcal{A} \to \mathbb{R}$ satisfies the same properties as $\mathfrak{c}$ in \eqref{eq:strictly concave reward} and can be $(x,\mathfrak{u})$-dependent in the same way to the $x$-dependent $\mathfrak{c}$ in \eqref{eq:strictly concave reward:general} (Remark~\ref{remark:generalization of u-AC RL problem:1}). Note that the resulting IPI will be model-free --- it does not explicitly depend on the input-coupling dynamics $f_\mathsf{c}$ in \eqref{eq:decomposition of f(x,u)} anymore and, of course, $f_\mathsf{d}$. When $\mathcal{U} = \mathcal{A} = \mathbb{R}^m$, similar ideas were presented by \citet{Murray2002} for input-affine optimal control and \citet*{Lee2012} for LQRs.

\subsection{Discounted RL with Bounded State Trajectories}
\label{subsection:discounted RL with bounded state trj}


When $\gamma \in (0, 1)$ and the state trajectories are bounded, the properties and results similar to those in ``\secref{subsection:discounted RL with bounded v} Discounted RL with Bounded VF'' can be obtained as shown below.

\MyDefinition{
	The state trajectories under $\pi$ are said to be globally bounded iff for each $x \in \mathcal{X}$, 
	$t \mapsto \mathbb{G}_\pi^x [X_t]$ is bounded over $\mathbb{T}$.
}

\begin{proposition}
	\label{prop:boundary condition true when trjs are bounded}
	If the state trajectories under $\pi$ are globally bounded, and  $v$ is continuous, then under $\gamma \in (0, 1)$, they satisfy the boundary condition \eqref{eq:uniqueness condition for v}.
\end{proposition}

\begin{Proof}
	Since $t \mapsto \mathbb{G}_\pi^x [  X_t ]$ is bounded and $v$ is continuous, $t \mapsto \mathbb{G}_\pi^x [ v(X_t) ]$ is also bounded, for each $x \in \mathcal{X}$. Hence, the proof can be done by applying Lemma~\ref{lemma:boundary condition on v_* with discounting} in \secref{appendix:section:proofs}.
\end{Proof}

\begin{corollary} [Policy Evaluation]
	\label{cor:properties of discounted RL with bounded state trj}
	Let $\gamma \in (0, 1)$ and the state trajectories under $\pi$ be globally bounded. Then, $\pi$ is admissible, and $v = v_\pi$ is the unique solution to the BEs \eqref{eq:Bellman eq for v} and \eqref{eq:differential BE for v} over all continuous and $\mathrm{C}^1$ functions, respectively.
\end{corollary}

\begin{Proof}
	 Apply Theorem~\ref{thm:uniqueness condition} and Proposition~\ref{prop:boundary condition true when trjs are bounded}.
\end{Proof}

By Corollary~\ref{cor:properties of discounted RL with bounded state trj}, as long as the state trajectories under~$\pi_{i-1}$ are globally bounded and $\gamma \in (0, 1)$, $\pi_{i-1}$ is admissible, and the $i$th iteration of the PI methods can run without assuming the boundary condition~\eqref{eq:boundary condition for PI} that is shown to be true by Proposition~\ref{prop:boundary condition true when trjs are bounded}. In this case, however, the VF is not necessarily bounded (see the next example, LQR (\secref{subsection:LQR}), in which the admissible VF is always quadratic), and it is a bit unclear when and how the state trajectories are bounded. Some stability-related conditions sufficient for global boundedness of the state trajectories are:

\begin{enumerate}
	\item input-to-state stability \citep[Definition~4.7]{Khalil2002}, ensuring that the state trajectories are globally bounded under \emph{any given policy} whenever $\mathcal{U}$ is bounded;
		\\[-7.5pt]
	\item global asymptotic stability (e.g., see the nonlinear optimal control in \secref{subsection:nonlinear optimal control} and the LQR in \secref{subsection:LQR});
		\\[-7.5pt]
	\item global ultimate boundedness of the state trajectories \citep[Definition~4.6]{Khalil2002}, which is stronger than the global boundedness of the state trajectories but weaker than global asymptotic stability. 
\end{enumerate}  

In general, stability of the system implies boundedness of the state trajectories within some region, but \emph{not vice versa}. 

Note that the global boundedness of the state trajectories under $\pi$, including the above three special cases, guarantees their global existence and uniqueness over the entire time interval $\mathbb{T}$, under locally Lipschitz $f_\pi$, as can be shown by applying the following proposition for all $x \in \mathcal{X}$.

\begin{proposition} 
\label{prop:existence and uniqueness of state trj}
Let $f_\pi$ be locally Lipschitz and $x \in \mathcal{X}$. If there exists a compact subset $\Omega_x \subset \mathcal{X}$ s.t. $t \mapsto \mathbb{G}_\pi^x[X_t]$ lies entirely in $\Omega_x$, then the state trajectory $t \mapsto \mathbb{G}_\pi^x [X_t]$ is uniquely defined and $\mathrm{C}^1$ over $\mathbb{T}$.
\end{proposition}

\begin{Proof}
	See \citep[Section~3.1 with Theorem~3.3  therein]{Khalil2002}.
\end{Proof}

The HJB solution $(v_*, \pi_*)$ can be also characterized in the discounted case as the optimal solution among all the policies that make the state trajectories globally bounded.

\begin{corollary}
	\label{cor:optimality for bounded trajectories}
	Suppose $\gamma \in (0, 1)$ and the HJBE \eqref{eq:HJBE} has an upper-bounded solution $v_* \in \mathrm{C}^1$. Then, 
	\begin{enumerate}
		\item $\pi_*$ is admissible and $v_* \leqslant v_{\pi_*}$, for any policy $\pi_*$ s.t. \eqref{eq:optimal policy} holds;
		\item if the state trajectories under $\pi$ (resp. $\pi_*$) are globally bounded, then $v_\pi \leqslant v_*$ (resp. $v_* = v_{\pi_*}$).
	\end{enumerate} 
\end{corollary}
	
\begin{Proof}
	Obvious by Theorem~\ref{thm:general optimality of the HJB solution}a--c and Proposition~\ref{prop:boundary condition true when trjs are bounded} with $v = v_*$.
\end{Proof}

\subsection{Linear Quadratic Regulations (LQRs)}
\label{subsection:LQR}
A linear quadratic regulation (LQR) consists of 
\begin{align}
	\begin{cases}
		\textrm{a linear dynamics: } f(x,u) = A^0x + Bu,	\\[7.5pt]
		\textrm{the unconstrained action space: } \mathcal{U} = \mathbb{R}^m,	\\
		\textrm{a quadratic positive cost function: }
		c(x,u) = 
		\begin{bmatrix}
			x^\mathsf{T}  & u^\mathsf{T} 	
		\end{bmatrix}
		\mathcal{W}
		\begin{bmatrix}
			x \\ u	
		\end{bmatrix}
		\geq 0,
		\text{ with }
		\mathcal{W} \doteq 	
		\begin{bmatrix}
					\!\!S &	E \\ 
					E^\mathsf{T} & \Gamma
		\end{bmatrix},
	\end{cases}
	\label{eq:LQR case}
\end{align}
where $(A^0,B,S)$ for $A^0 \in \mathbb{R}^{l \times l}$, $B \in \mathbb{R}^{l \times m}$, $S \in \mathbb{R}^{l \times l}$ is stabilizable and observable, $\mathcal{W} \in \mathbb{R}^{(l + m) \times (l + m)}$ is positive semidefinite and nondegenerate,\footnote{$\mathcal{W}$ is nondegenerate iff $\rank{\mathcal{W}} = \rank{S} + \rank{R}$, which is true when $\mathcal{W}$ is positive definite or $E = 0$.} and $\Gamma \in \mathbb{R}^{m \times m}$ is positive definite. Note that the LQR~\eqref{eq:LQR case} falls into a special case of the nonlinear optimal control  in \secref{subsection:nonlinear optimal control} whenever the matrix $\mathcal{W}$ is positive definite. On other other hand, $f^x$ is affine and $r^x$ is strictly concave for each $x \in \mathcal{X}$, with its dynamics satisfying \eqref{eq:input-affine dynamics} for 
\[
	f_{\mathsf{d}}(x) = A^0x \textrm{ and } F_{\mathsf{c}}(x) = B
\] 
	and its reward function $r$ ($= -c$) given of the form \eqref{eq:strictly concave reward:general} in Remark~\ref{remark:generalization of u-AC RL problem:1} for 
\[
	\mathfrak{r}(x) = - x^\mathsf{T} S x  \textrm{ and } \mathfrak{c}(x, u) = u^\mathsf{T} \Gamma u + 2 x^\mathsf{T} E u.
\]
Moreover, whenever $E = 0$, it becomes \eqref{eq:strictly concave reward} with $\mathfrak{c}$ given by $\mathfrak{c}(u) = u^\mathsf{T} \Gamma u$, the unconstrained case~``\eqref{eq:integral formula of S} and \eqref{eq:unconstrained case}''. Therefore, the LQR \eqref{eq:LQR case} is an example of the concave Hamiltonian formulation in \secref{subsubsection:case I}. Also note that in LQR, $f$ is obviously globally Lipschitz, ensuring the global existence of the unique state trajectories under any globally Lipschitz policy \citep[Theorem~3.2]{Khalil2002}; if the policy $\pi$ is linear, i.e., 
\begin{equation}
	\pi(x) = - Kx \text{ for a gain matrix } K \in \mathbb{R}^{m \times l},
	\label{eq:linear policy}	
\end{equation}
then the state trajectory $t \mapsto \mathbb{G}_\pi^x [X_t]$ is explicitly given by $\mathbb{G}_\pi^x [X_t] = e^{(A^0-BK) t} x$ \citep{Chen1998}.

\SetKwRepeat{Repeat}{\textbf{repeat} {\normalfont (under the LQR formulation~\eqref{eq:LQR case})}}{\textbf{until }}

 \begin{algorithm2e}[h!]
 \begin{spacing} {1.1}
 \DontPrintSemicolon
 \BlankLine
  \nl \textbf{Initialize:} 
  	$\pi_0(x) = - K_0x$, the initial admissible policy; $i \leftarrow 1$;\;
  	 $\;$\\[-5pt]
 \nl \Repeat{convergence is met.}
 {
 $\;$\\
 \nl\textbf{Policy Evaluation:} given policy $\pi_{i-1}(x) = - K_{i-1} x$, find a quadratic function $v_i(x) = - x^\mathsf{T} P_i x$, with $P_i = P_i^\mathsf{T}$, s.t. \\[5pt]
 	\begin{itemize} [leftmargin=0.5cm]
 		\item [] \textbf{(IPI)} $v_i$ satisfies the BE~\eqref{eq:Bellman eq for v} for some $\eta > 0$; or \textbf{(DPI)} $P_i \in \mathbb{R}^{l \times l}$ satisfies the matrix formula~\eqref{eq:matrix iter in policy eval of DPI for LQR}; 
 				\\[10pt]
 	\end{itemize} 
 \nl \textbf{Policy Improvement:} $K_{i} \gets \Gamma^{-1} (B^\mathsf{T} P_i + E^\mathsf{T})$; 
 	\\[5pt]
 \nl $i \leftarrow i+1$;
 }
 \end{spacing}
 \BlankLine
 \caption{IPI and DPI for the LQR~\eqref{eq:LQR case}}
 \label{algorithm:IPI:LQR}
\end{algorithm2e}

In an LQR~\eqref{eq:LQR case}, under a linear policy~\eqref{eq:linear policy}, $J_\pi$ ($\doteq - v_\pi$) is quadratic, if finite, and can be expressed as $J_{\pi}(x) = x^\mathsf{T} P_{\pi} x$ for a positive definite matrix~$P_{\pi} \in \mathbb{R}^{l \times l}$ 
(e.g., see \citealp[Lemma~16.3.2 with Theorem~16.3.3.(d)]{Lancaster1995}; \citealp[Section~2]{Lee2014}). Moreover, 
the maximal policy~$\pi'$ in Remark~\ref{remark:generalization of u-AC RL problem:1} is linear again and can be represented as 
\begin{align*}
	\pi'(x) = - K' x \textrm{ with } K' = \Gamma^{-1} (B^\mathsf{T} P_\pi + E^\mathsf{T}).	
\end{align*}
This observation gives IPI and DPI for the LQR~\eqref{eq:LQR case} shown in Algorithm~\ref{algorithm:IPI:LQR}, where DPI solves the matrix equation: 
\begin{align}
			(A_{i-1}^\alpha)^\mathsf{T} P_{i} + P_i A_{i-1}^\alpha 
			= K_{i-1}^\mathsf{T} E^\mathsf{T} + E K_{i-1} - S - K_{i-1}^\mathsf{T} \Gamma K_{i-1},
	\label{eq:matrix iter in policy eval of DPI for LQR}
\end{align}
at each of the $i$th iteration of policy evaluation. Here, we denote 
\[
	A_{i-1}^\alpha \doteq A^\alpha - B K_{i-1} \text{ for } A^{\alpha} \doteq A^0 - \alpha I/2
\]
where $I \in \mathbb{R}^{l \times l}$ denotes the identity matrix. Note that DPI (and IPI --- see Theorem~\ref{thm:LQR IPI property}a below) in Algorithm~\ref{algorithm:IPI:LQR} is equivalent to the existing matrix-form PIs (\citealp{Arnold1984,Mehrmann1991}; see also \citealp{Kleinman1968,Lee2014} for the case $E = 0$). In addition, if $\mathcal{W}$ is positive definite, then rearranging \eqref{eq:gas condition in optimal control} using \eqref{eq:matrix iter in policy eval of DPI for LQR} yields the very stability condition:
\[
	\textrm{$(A_{i-1}^0)^\mathsf{T} P_i + P_i A_{i-1}^0$ is negative definite,}	
\]
for $J_i$ ($= -v_i$) to be the Lyapunov function for the linear dynamics $f(x, u) = A^0x + Bu$ under the policy $\pi_{i-1}(x) = K_{i-1} x$ \citep[Theorem~4.6]{Khalil2002}. Here, each $P_i$ is assumed symmetric and proven below to be positive definite by \smash{$P_i = P_{\pi_{i-1}}$}.

In fact, if the policy~$\pi$ is linear, the process~$X_t^\alpha$ generated by 
\begin{equation}
	\dot X_t^\alpha = A^\alpha X_t^\alpha + B U_t^\alpha
	\label{eq:Z-linear system}
\end{equation}
and $U_t^\alpha = \pi(X_t^\alpha)$ for all $t \in \mathbb{T}$ yields the following expression \eqref{eq:LQR value function total form} of $J_\pi$, \emph{without the discount factor $\gamma$ (or rate $\alpha$) in its cumulative cost} \citep{Anderson1989}:
\begin{align}
	J_\pi(x)
	= \mathbb{G}_{\pi}^{x} \bigg [ \int_0^\infty e^{-\alpha t} \cdot C_t \, dt \bigg ]
	= \mathbb{G}_{\pi}^{x, \alpha} \bigg [ \int_0^\infty C_t \, dt \bigg ],\footnotemark
	\label{eq:LQR value function total form}
\end{align}
where $\mathbb{G}_\pi^{x, \alpha}[Y]$ means $\mathbb{G}_\pi^x[Y]$ if $\alpha = 0$ but otherwise the value $Y$ w.r.t. the state  $X_t = X_t^\alpha$ and the action $U_t = U_t^\alpha$ $\forall t \in \mathbb{T}$; $C_t = c(X_t, U_t)$ is the quadratic cost at time $t$. \footnotetext{\label{first} The equality comes from the fact: 
$\mathbb{G}_\pi^x [ e^{-\alpha t/2} \, X_t] = e^{-\alpha t/2} \cdot e^{(A^0-BK) t} x = e^{(A^\alpha-BK) t} x = \mathbb{G}_\pi^{x, \alpha} [ X_t]$ for $\pi(x) = - Kx$.} 
Here, $(A^\alpha, B, S)$ is stabilizable and observable since so is $(A^0, B, S)$ (\secref{appendix:proofs regarding LQR}). Therefore, any discounted LQR can be transformed into an equivalent undiscounted total  one, simply by replacing $A^0$ with $A^\alpha$. 

After the transformation into \eqref{eq:Z-linear system} and \eqref{eq:LQR value function total form}, we can see that a linear policy $\pi$ is admissible \emph{iif} $X_t^\alpha$ under~$\pi$ converges to~$0$ (see \citealp[Proposition~16.2.9]{Lancaster1995}); the convergence $X_t^\alpha \to 0$ implies that any quadratic function $J$ ($= - v$), say $J(x) = x^\mathsf{T} P x$ for some $P \in \mathbb{R}^{l \times l}$, satisfies the boundary condition~\eqref{eq:uniqueness condition for v} since
	\begin{align*}
		\mathbb{G}_\pi^x \big [ \gamma^{t} J(X_{t})  \big ]  
		=  
		\mathbb{G}_\pi^x \big [ e^{- \alpha t} \!\cdot\! X_{t}^{\mathsf{T}} P X_{t} \big ] 
		= \mathbb{G}_{\pi}^{x, \alpha} \big [ J ( X_t ) \big ] \longrightarrow 0 \text{ as } t \to \infty.
		\footref{first}
	\end{align*}
	Therefore, by Theorem~\ref{thm:fundamental properties}, $\pi_i$ in Algorithm~\ref{algorithm:IPI:LQR} is admissible and $P_i = P_{\pi_{i-1}}$ for all $i \in \mathbb{N}$, but \emph{without assuming the boundary condition~\eqref{eq:boundary condition for PI}} that is true in LQR, as shown above.

As regards to the HJB solution $(v_*, \pi_*)$ and the Assumptions in \secref{section:fundamental properties of PI}, the applications of the LQR theory \citep[Theorem 16.3.3]{Lancaster1995}, Proposition~\ref{prop:fixed point and HJBE}, and  Lemma~\ref{lemma:concave h}a with Remark~\ref{remark:generalization of u-AC RL problem:1} to \eqref{eq:Z-linear system} and \eqref{eq:LQR value function total form} show that 
\begin{enumerate}
	\item $(v_*,  \pi_*)$ satisfying the HJBE~\eqref{eq:HJBE} and \eqref{eq:optimal policy} exists;
	\item $J_*$ ($\doteq - v_*$) and $ \pi_*$ are optimal and given by 
		\[
		\begin{cases} 
			J_*(x) = x^\mathsf{T} P_* x \textrm{ for a positive definite } P_* \in \mathbb{R}^{l \times l}, \\[5pt]
			\pi_*(x) = - K_* x \textrm{ with } K_* \doteq \Gamma^{-1} (B^\mathsf{T} P_* + E^\mathsf{T});
		\end{cases}
		\]
	\item Assumptions~\ref{assumption:uniqueness of fixed point}, \ref{assumption:closed graph}, and \ref{assumption:uniqueness of HJB over C1} are all true.
\end{enumerate}
Applying the theory developed in this work, we finally obtain the following result regarding the PIs applied to the LQR.
	\begin{theorem}
	\label{thm:LQR IPI property}
	The sequences $\langle K_i \rangle$ and $\langle P_i \rangle$ generated by Algorithm~\ref{algorithm:IPI:LQR} satisfy the followings:
	\begin{enumerate}
		\item [\textbf{\emph{a}}.] $\forall i \in \mathbb{N}$: $\pi_i(x) = - K_i x$ is admissible and $P_i = P_{\pi_{i-1}}$,
				 \\[-5pt]
		\item [\textbf{\emph{b}}.] $0 < P_* \leq \cdots \leq P_{i+1} \leq P_i \leq \cdots \leq P_1$,
				 \\[-5pt]
		\item [\textbf{\emph{c}}.] $\lim_{i \to \infty} P_i = P_*$ and $\lim_{i \to \infty} K_i = K_*$.
	\end{enumerate}	
\end{theorem}

\begin{Proof}
	First, Theorem~\ref{thm:fundamental properties} and the optimality of $P_*$ prove the first and second parts. Next, Theorem~\ref{thm:convergence of PI} implies that there exists $P \in \mathbb{R}^{l \times l}$ s.t. $P_i \to P$ (see \secref{appendix:proofs regarding LQR}).
	Let $M_\Omega \doteq \sup_{x \in \Omega} \|x\| < \infty $ for a compact subset $\Omega \subset \mathcal{X}$. Then, we have
	\begin{align*}
		0 \leq \smash{\sup_{x \in \Omega}} \big \|(P_i - P)x \big \| & \leq 
		\smash{\sup_{x \in \Omega}} \big ( \MatNorm{P_i - P} \! \cdot \! \| x \| \big )
		= M_\Omega \! \cdot \! \MatNorm{P_i - P},
	\end{align*}
	where $M_\Omega \cdot \MatNorm{P_i - P} \to 0$ by $P_i \to P$. Hence, $\nabla v_i$ given by $\nabla v_i(x) = - 2 x^\mathsf{T} P_i$ converges uniformly on any compact subset of $\mathcal{X}$ and by Lemma~\ref{lemma:equivalence b.t.w. locally uniform and compact convergence}, locally uniformly. Finally, by Theorem~\ref{thm:true convergence} with Remark~\ref{remark:generalization of u-AC RL problem:1}, $P = P_*$ and $K_i \to K_*$.	
\end{Proof}

By extending the existing analytical results to the LQR \eqref{eq:LQR case}, we can see more: \emph{the convergence $P_i \to P_*$ is quadratic} (see \secref{appendix:proofs regarding LQR}). Therefore, PI methods have faster convergence rates than linear in both discrete and continuous domains: it is finite in a finite MDP \citep{Puterman1994,Powell2007,Sutton2018} and quadratic in the LQR~\eqref{eq:LQR case}. Moreover, the latter could also imply the local quadratic convergence \mbox{$v_i \to v_*$} for a class of nonlinear optimal control problems in \secref{subsection:nonlinear optimal control} as the nonlinear problem can be approximated near the equilibrium point $(x_\mathsf{e}, u_\mathsf{e}) = (0,0)$ by an LQR~\eqref{eq:LQR case} with 
\begin{equation*}
	A^0 = \nabla_x f(0,0), \;\; B = \nabla_u f(0,0), \;\; \mathcal{W} = \nabla^2 c(0,0) 		
\end{equation*}
whenever the gradient $\nabla f(x, u) \in \mathbb{R}^{l \times (l+m)}$ and the Hessian $\nabla^2 c(x, u) \in \mathbb{R}^{(l+m) \times (l+m)}$ exist and are continuous, at $(0, 0)$. Here, $\nabla_x f$ and $\nabla_u f$ denote the gradients  of $f(x,u)$ w.r.t. $x$ and $u$, respectively. Therefore, the rate of convergence is possibly, locally quadratic for the nonlinear optimal control problem in \secref{subsection:nonlinear optimal control} when its linearization $(A^0, B, \mathcal{W})$ above exists and satisfies the assumptions on the LQR shown in this subsection---since $c$ and thus $\mathcal{W}$ are positive definite, those assumptions are in fact guaranteed to be true, except stabilizability of $(A^0, B)$.

\section{Implementation Details}
\label{appendix:section:implementation details}

This appendix provides details of the implementations of the PI methods (i.e., Algorithm~\ref{algorithm:Variants under bounded v}) experimented in \secref{section:simulation}.

\subsection{Structure of the VF Approximator $V_i$}

Recall that in \secref{section:simulation}, the solution to the policy evaluation, $V_i$, is represented by a linear function approximator $V$ as
\begin{equation}
	V_i(x) \approx V(x; \theta_i) \doteq \theta_{i}^\mathsf{T} \phi(x),  
	\tag{\ref{eq:RBFN approximation of vi}}
\end{equation}
for its weights $\theta_i \in \mathbb{R}^L$ and features $\phi: \mathcal{X} \to \mathbb{R}^L$, with the number of features $L = 121$. Since the policy improvement needs a differentiable structure, we choose radial basis functions (RBFs) as the features $\phi$, rather than using (tile-coded) binary ones \citep{Sutton2018}. Hence, the $j$-th component of the feature vector $\phi$ is given by 
\[
	\phi_j(x) = \mathrm{exp} \big ( \! - \! (x - c_j)^\mathsf{T} \Sigma^{-1} \, (x - c_j) \big )
\]
where $\Sigma \doteq \diag{1, 2}$ is a weighting matrix, and $\{ c_j \! \in \! \Omega : 1 \leq j \leq L \}$ is the set of RBF center points $c_j$ that are uniformly distributed within the compact region $\Omega = [- \pi, \pi] \times \! [-6, 6] \subset \mathcal{X}$. In the simulations in \secref{section:simulation}, we choose $L \! = \! 11 \times 11 \! = \! 121$; the set of center points $\{c_j\}$ includes the origin $(0, 0)$ and a finite number of points on the boundary $\partial \Omega$. Whenever inputting to the features $\phi$, the first component $x_1$ of $x$ is normalized to a value within $[-\pi, \pi]$ by adding $\pm 2 \pi k$ to it for some $k \in \mathbb{Z}$. 

\subsection{Least-Squares Solution of Policy Evaluation}

In the experiments in \secref{section:simulation}, the policy evaluation (or the BE) in Algorithm~\ref{algorithm:Variants under bounded v} is solved by batch least squares, over the set of initial states $\{ x_k : 1 \leq k \leq N \times M \}$, uniformly distributed as the ($N \times M$)-grid points over $\Omega$, where $N$, $M \in \mathbb{N}$ are the total numbers of the grids in the $x_1$- and $x_2$-directions, respectively. We chose $N = 20$ and $M = 21$, so at each of the $i$th iteration, the total $420$ number of grid points $x_k$'s in $\Omega$ are considered to determine the least-squares solution $\theta_i^*$ of policy evaluation, except for the DPI variant in Case 4 where we used $M = 20$ instead of $21$.

To describe the batch least square solution $\theta_i^*$, note that under the approximation~\eqref{eq:RBFN approximation of vi}, the BEs of the variants of DPI and IPI in Algorithm~\ref{algorithm:Variants under bounded v} can be expressed at each point $x = x_k$ as 
\begin{equation}
	y_i^{\mathsf{T}}(x_k) \! \cdot \! \theta_i + \varepsilon_i(x_k) = r(x_k, \pi_{i-1}(x_k)), 
	\label{eq:Bellman eq with LFA}
\end{equation}
where $\varepsilon_i: \mathcal{X} \to \mathbb{R}$ is the approximation error for each case, and $y_i: \mathcal{X} \to \mathbb{R}^L$ is given by 
\[
	y_i(x) \doteq
		\begin{cases}
		\mathbb{G}_{\pi_{i-1}}^x \! \big [ \phi(X_0) - \gamma_\mathsf{d} \cdot \phi(X_{\Delta t}) \big ] & \text{for the variant of IPI,} 
		\\[5pt]
		\alpha_\mathsf{d} \! \cdot \! \phi(x) - \Delta t \! \cdot \! \nabla \phi(x) \! \cdot \! f_{\pi_{i-1}} (x) & \text{for the variant of DPI.}
		\end{cases}
\]
Concatenating the vectors as and denoting them by  
\begin{align*}
	\mathcal{Y}_i &\doteq 
		\begin{bmatrix}
				y_i(x_1) &\, y_i(x_2) &\, \cdots &\, y_i(x_{NM})
		\end{bmatrix}
		\\
	\mathcal{E}_i &\doteq 
		\begin{bmatrix}
				\varepsilon_i(x_1) &\, \varepsilon_i(x_2) &\, \cdots &\, \varepsilon_i(x_{NM})
		\end{bmatrix}^\mathsf{T}
		\\		
	\mathcal{R}_i &\doteq 
		\begin{bmatrix}
			r(x_1, \pi_{i-1}(x_1)) & \cdots & r(x_{NM}, \pi_{i-1}(x_{NM})) 
		\end{bmatrix}^\mathsf{T} 
\end{align*}
the expression~\eqref{eq:Bellman eq with LFA} can be compactly rewritten as 
\[
	\mathcal{Y}_i^\mathsf{T} \! \cdot \theta_i + \mathcal{E}_i = \mathcal{R}_i,
\]
and the batch least-squares solution ${\theta}_i^*$ minimizing the approximation error $\mathcal{J}(\theta_i) \doteq \tfrac{1}{2} \| \mathcal{E}_i \|^2$ over $\{ x_k \}$ is given by 
\[
	\smash{{\theta}_i^*} = \big ( \mathcal{Y}_i \mathcal{Y}_i^\mathsf{T} \big )^{-1} \mathcal{Y}_i \, \mathcal{R}_i 
\]
so long as $\rank{\mathcal{Y}_i} = L$. At each of the $i$th iteration, we collected data $\mathcal{Y}_i$ and $\mathcal{R}_i$ at the distinct points $\{ x_k \} \subset \Omega$ and then performed the batch least squares to yield the minimizing solution~${\theta}_i^*$ of policy evaluation. 

\subsection{Reward Function and Policy Improvement Update Rule}

Recall that each experimental case in \secref{section:simulation} basically considers the reward function $r$ given by \eqref{eq:strictly concave reward} and \eqref{eq:integral formula of S} with \eqref{eq:sim:s}, that is, 
\begin{equation}
	r(x, u) = \mathfrak{r}(x) - \mathfrak{c}(u), \; \text{ with }
	\mathfrak{c}(u) = \lim_{v \to u} \int_0^v (s^{\mathsf{T}})^{-1} (\mathfrak{u}) \cdot \Gamma \, d\mathfrak{u}
	\text{ and }
	s(\mathfrak{u}) = u_{\mathsf{max}} \tanh(\mathfrak{u}/u_{\mathsf{max}})
	\label{eq:appendix:policy improvement:1}
\end{equation}
where $\Gamma > 0$ and the sigmoid function $s$ gives the following expressions of the functions $\sigma$ in \eqref{eq:policy improvement without max} and $\mathfrak{c}$: 
\begin{align*}
	\sigma(\mathfrak{u}) &= u_{\mathsf{max}} \tanh (\Gamma^{-1} \cdot \mathfrak{u}/u_{\mathsf{max}})
	= 5 \tanh \big ((5\Gamma)^{-1} \cdot \mathfrak{u} \big ),
	\\[5pt]
	\mathfrak{c}(u) &= \Gamma \cdot (u_{\mathsf{max}}^2/2) \cdot \ln \big ( {u}_+^{{u}_+} \cdot {u}_-^{{u}_-} \big )
	= 12.5 \cdot \Gamma \cdot \ln \big ( {u}_+^{{u}_+} \cdot {u}_-^{{u}_-} \big )	
\end{align*}
for $u_\pm \doteq 1 \pm u / u_{\mathsf{max}}$. Here, note that $\mathfrak{c}(u)$ is finite for all $u \in \mathcal{U}$ and has its maximum at the end points $u = \pm u_{\mathsf{max}}$ as $\mathfrak{c}(\pm u_{\mathsf{max}}) = \Gamma \cdot (u_{\mathsf{max}}^2\ln 4)/2 \approx 17.3287 \cdot \Gamma$.

As the inverted pendulum dynamics is input-affine, the above reward setting \eqref{eq:appendix:policy improvement:1} corresponds to the concave Hamiltonian formulation in \secref{subsubsection:case I}. Hence, the policy improvement becomes the following simple update rule:
\begin{equation}
	\pi_i(x)
	\approx 
	\pi(x; \theta_i^*) 
	= \sigma \big ( \Delta t \cdot F_{\mathsf{c}}^\mathsf{T}(x) \cdot \nabla V^\mathsf{T}(x; \theta_i^*) \big ) 
	= - 5 \tanh \bigg ( \frac{\Delta t }{5\Gamma} \cdot \cos x_1 \cdot \nabla_{\!x_2} \phi (x) \cdot \theta_i^* \bigg )
	\label{eq:appendix:policy improvement:2}
\end{equation}
where $\nabla_{\!x_2} \phi (x) \in \mathbb{R}^{1 \times L}$ denotes the gradient of $\phi(x_1, x_2)$ with respect to the second component $x_2$. Cases 1 and 2 in \secref{section:simulation} are associated with the above update rule~\eqref{eq:appendix:policy improvement:2}.

In the limit $\Gamma \to 0^+$, it is obvious that $\sigma(\mathfrak{u}) \to u_\mathsf{max} \cdot \mathrm{sign}(\mathfrak{u})$ and $\mathfrak{c}(u) \to 0$. In this case, thereby, the reward function \eqref{eq:appendix:policy improvement:1} and the policy improvement update rule \eqref{eq:appendix:policy improvement:2} become $r(x, u) = \mathfrak{r}(x)$ and 
\begin{align*}
	\pi_i(x) \approx \pi(x; \theta_i^*) = - u_\mathsf{max} \cdot \mathrm{sign} \big (\cos x_1 \! \cdot \! \nabla_{\!x_2} \phi (x) \! \cdot \! \theta_i^* \big).
\end{align*}
Cases 3 and 4 in \secref{section:simulation} consider this type of bang-bang policies, with continuous (Case 3) and binary state-reward $\mathfrak{r}$ (Case 4).

\section{Proofs}
\label{appendix:section:proofs}

In this appendix, we provide all the proofs of the Theorems, Lemmas, Propositions, and Corollaries stated in the main work \citep{Lee2020b}. For the proof of properties of locally uniform convergence, the following lemma is necessary.

\begin{lemma} 
	\label{lemma:equivalence b.t.w. locally uniform and compact convergence}
	A sequence of functions $g_i : \mathcal{X} \to \mathbb{R}^n$ converges to $g$ locally uniformly iff $g_i \to g$ uniformly on every compact subset of $\mathcal{X}$.
\end{lemma}

\begin{Proof}
	The proof is a simple extension of \citet{Remmert1991}'s from $n = 1$ to any $n \in \mathbb{N}$. For the proof, we generalize the metric $d_\Omega$ by redefining it as $\smash{d_\Omega(f, g) \doteq \sup_{x \in \Omega} \|f(x) - g(x)\|}$ for a subset $\Omega \subseteq \mathcal{X}$ and functions $f, g$ from $\mathcal{X}$ to $\mathbb{R}^n$. Then, the uniform convergence $g_i \to g$ on $\Omega$ is equivalent to $d_\Omega(g_i, g) \to 0$. Also note that $\mathcal{X}$ ($\doteq \mathbb{R}^l$) is a Euclidean space.
	
	First, suppose that $g_i \to g$ uniformly on every compact subset of $\mathcal{X}$, hence on every closed ball $\mathcal{\widebar B}_x(r) \doteq \{ y \in \mathcal{X} : \|x - y\| \leq r \}$. Since $\mathcal{\widebar B}_x(r)$ contains an open ball $\mathcal{B}_x(r) \doteq \{ y \in \mathcal{X} : \|x - y\| < r \}$, $g_i \to g$ uniformly on every $\mathcal{B}_x(r)$. Hence, we conclude that $g_i \to g$ locally uniformly ($\because$ each $\mathcal{B}_x(r)$ is a neighborhood of each $x \in \mathcal{X}$).

	To prove the converse, suppose that $g \to g_i$ locally uniformly so that $g \to g_i$ uniformly on a neighborhood $\mathcal{N}_x$ of each point $x \in \mathcal{X}$. Let $\Omega$ be a compact subset of $\mathcal{X}$. Then, since every neighborhood is open, $\{ \mathcal{N}_x : x \in \Omega\}$ is an open cover of $\Omega$, i.e., a collection of open sets $\mathcal{N}_x$ s.t. $\bigcup_{x \in \Omega} \mathcal{N}_x \supset \Omega$. By Heine-Borel Property \citep[Theorem~13.94]{Thomson2001}, the open cover $\{ \mathcal{N}_x : x \in \Omega\}$ of $\Omega$ can be reduced to a finite subcover of $\Omega$, say $\{\mathcal{N}_{x_j}\}_{j=1}^k$, meaning that $\mathcal{O} \doteq \bigcup_{j = 1}^k \mathcal{N}_{x_j} \supset \Omega$. Since $g_i$ locally uniformly converges to $g$ on each $\mathcal{N}_{x_j}$, so does on their finite union $\mathcal{O}$, hence on the subset $\Omega$ of $\mathcal{O}$. This completes the proof since the compact set $\Omega \subset \mathcal{X}$ is arbitrary.
\end{Proof}

\subsection{Proofs in \secref{section:preliminaries} Preliminaries}

\begin{Proof*}{Lemma~\ref{lemma:VF range} (\secref{subsection:RL Problem})}
	For any policy~$\pi$ and any $x \in \mathcal{X}$,
	\[
		v_\pi(x)
		\leq
		\lim_{\eta \to \infty} \bigg ( r_\mathsf{max} \cdot \int_0^\eta \gamma^{t} \, dt \bigg )
		=  
		\begin{cases}
			r_\mathsf{max} / \alpha \,\textrm{ for } \gamma \in (0, 1),
			\\[5pt]
			\quad\;\; 0 \;\;\;\, \textrm{ for } \gamma = 1
		\end{cases}
	\]
	(note that $r_\mathsf{max} = 0$ when $\gamma = 1$). 
	This proves the statement with $\widebar v = r_\mathsf{max} / \alpha$ for $0 < \gamma < 1$ and $\widebar v = 0$ for $\gamma = 1$.	
\end{Proof*}

\begin{Proof*}{Proposition~\ref{prop:general admissible condition} (\secref{subsection:RL Problem})}
	If the reward $R_t$ under a policy~$\pi$ satisfies \eqref{eq:exponential increasing condition for admissibility} for $\ushort \alpha < \alpha$ and for all $x \in \mathcal{X}$, then 
	\begin{align*}
		v_\pi(x) = \int_0^\infty e^{- \alpha t} \! \cdot \mathbb{G}_\pi^x [ \;\! R_t \;\! ] \, dt 
		\geq 
		\xi(x) \cdot \! \int_0^\infty e^{- (\alpha - \ushort \alpha ) t} \; dt = ( \alpha - \ushort \alpha )^{-1} \cdot \xi(x) > - \infty
		\qquad \forall x \in \mathcal{X}
	\end{align*}
	by definitions. This also shows that $v_\pi$ is lower-bounded if so is $\xi$. Finally, the proof is completed by Lemma~\ref{lemma:VF range}.
\end{Proof*}

\begin{Proof*}{Lemma~\ref{lemma:Bellman ineq} (\secref{subsection:BE with boundary condition})}
By the standard calculus and $\alpha \doteq - \ln \gamma$, 
\[
	\frac{d}{dt} \big ( \gamma^{t} \cdot v(X_t) \big )
	= \gamma^{t} \cdot \big ( \dot v (X_t, U_t) - \alpha \cdot v(X_t) \big ).			
\]
Hence, applying \eqref{eq:conversion lemma:differential ineq} and noting that 
$h(x, u, \nabla v(x)) = r(x, u) + {\dot v}(x, u)$,
 we obtain that for any $t \geq 0$ and $x \in \mathcal{X}$, 
\begin{align*}
	0 
	\sim \mathbb{G}_\pi^x \Big [ \, \gamma^{t} \! \cdot \! \big ( \, h(X_t, U_t, \nabla v(X_t)) -  \alpha \! \cdot \! v(X_t) \big ) \Big ] 
	= \mathbb{G}_\pi^x \Big [ \, \gamma^{t} \! \cdot \! \big ( R_{t} + \dot v(X_t, U_t) - \alpha \! \cdot \! v(X_t) \big ) \Big ]
	= \mathbb{G}_\pi^x \bigg [ \, \gamma^{t} \! \cdot \! R_{t} + {\frac{d}{dt}} \big ( \gamma^{t} \cdot v(X_t) \big ) \bigg ]
\end{align*}
where $\sim$ is equal to $=$, $\leq$, or $\geq$. Then, integrating it from $t = 0$ to $t = \eta$ yields \eqref{eq:conversion lemma:integral ineq}. 
	
For the proof of the opposite direction, assume that $v$ satisfies \eqref{eq:conversion lemma:integral ineq}. Then, rearranging \eqref{eq:conversion lemma:integral ineq} as 
\begin{align*}
	\big ( 1 - \gamma^{\eta} \big ) \cdot v(x) \sim 
	\mathbb{G}_\pi^x \Big [ \mathfrak{R}_{\eta} 
	+ \gamma^{\eta} \! \cdot \! \big (
	v(X_\eta) - v(X_{0}) \big ) \Big ] \qquad \forall x \in \mathcal{X} \quad \forall \eta > 0, 
\end{align*}
dividing it by $\eta$, and limiting $\eta \to 0$ yields
\begin{align*}
	-\ln \gamma \cdot v(x) \sim r_\pi(x) + {\dot v}(x, \pi(x))
	\qquad \forall x \in \mathcal{X},
\end{align*}
which implies \eqref{eq:conversion lemma:differential ineq} since $\alpha = - \ln \gamma$ and $h(x, \pi(x), \nabla v(x)) = r_\pi(x) + {\dot v}(x, \pi(x))$. 
\end{Proof*}

\begin{Proof*}{Proposition~\ref{prop:VF satisfies the boundary condition}  (\secref{subsection:BE with boundary condition})}
Fix $x \in \mathcal{X}$ and take the limit $\eta \to \infty$ of \eqref{eq:Bellman eq}. Then, we obtain
	\begin{align*}
		v_\pi(x) &=  \lim_{\eta \to \infty} \mathbb{G}_\pi^x \big [ \, \mathfrak{R}_{\eta}
		+
		\gamma^{\eta} \! \cdot \! v_\pi(X_{\eta}) \, \big ]
		= v_\pi(x) + \lim_{\eta \to \infty} \mathbb{G}_\pi^x \big [  \gamma^{\eta} \!  \cdot \! v_\pi(X_{\eta})  \big ].
	\end{align*}
Hence, noting that $v_\pi(x)$ is finite by $\pi \in \Pi_\mathsf{a}$, we obtain the boundary condition 
	$\lim_{\eta \to \infty} \mathbb{G}_\pi^x  \big [ \gamma^{\eta} \! \cdot \! v_\pi(X_\eta) \big ] = 0$, which completes the proof as $x \in \mathcal{X}$ is arbitrary.
\end{Proof*}

\newpage

\begin{Proof*}{Theorem~\ref{thm:uniqueness condition}  (\secref{subsection:BE with boundary condition})}
	Suppose $v$ satisfies the integral BE~\eqref{eq:Bellman eq for v} without loss of generality (or, convert the differential BE~\eqref{eq:differential BE for v} into \eqref{eq:Bellman eq for v} via Lemma~\ref{lemma:Bellman ineq} and fix $\eta > 0$). Then, the repetitive applications of \eqref{eq:Bellman eq for v} to itself $k$-times result in
	\begin{align*}
		v(x) 
		= \mathbb{G}_\pi^x \big [ \mathfrak{R}_\eta + \gamma^\eta \! \cdot \! v(X_\eta) \big ]	
		= \mathbb{G}_\pi^x \big [ \mathfrak{R}_{2 \eta} + \gamma^{2 \eta} \! \cdot \! v(X_{2 \eta}) \big ]
		=
		\;\, \cdots \;\,
		= \mathbb{G}_\pi^x \big [ \mathfrak{R}_{k \cdot \eta} + \gamma^{k \cdot \eta} \! \cdot \! v(X_{k \cdot \eta}) \big ]	\qquad \forall x \in \mathcal{X}.
		\end{align*}
	Taking the limit $k \to \infty$ and substituting \eqref{eq:uniqueness condition for v}, we obtain
	\begin{align*}
		v(x) 
		=\, \underbrace{\lim_{k \to \infty} \mathbb{G}_\pi^x \big [ \, \mathfrak{R}_{k \cdot \eta} \big ]}_{=v_\pi(x)} +
		\underbrace{\lim_{k \to \infty} \mathbb{G}_\pi^x \big [ \, \gamma^{k \cdot \eta} \! \cdot \! v(X_{k \cdot \eta}) \, \big ]}_{=0}
		= v_\pi(x) \qquad \forall x \in \mathcal{X}.
	\end{align*}
	Therefore, $v = v_\pi$ and since $v(x)$ is finite for each $x \in \mathcal{X}$, $\pi$ is admissible. The converse is obvious by Proposition~\ref{prop:VF satisfies the boundary condition}.
\end{Proof*}

\begin{Proof*}{Lemma~\ref{lemma:policy improvement lemma} (\secref{subsection:policy improvement})}
The inequality in Lemma~\ref{lemma:policy improvement lemma} is equivalent to 
	\[
		v(x) \leq \mathbb{G}_{\pi'}^x \big [ \mathfrak{R}_{\eta}
			+ \gamma^{\eta} \! \cdot \! v(X_{\eta}) \big ]
		\qquad \forall x \in \mathcal{X} \quad \forall \eta > 0.
	\]
by Lemma~\ref{lemma:Bellman ineq}. Then, taking the limit supremum at $\eta \to \infty$, we obtain for each $x \in \mathcal{X}$:
	\[
		v(x) \leq v_{\pi'}(x) + \smash{\limsup_{\eta \to \infty}} \; \mathbb{G}_{\pi'}^x \big [ \gamma^{\eta} \!\cdot  v(X_{\eta}) \big ] \leq v_{\pi'}(x)
		\qquad \forall x \in \mathcal{X}
	\]
where we have substituted
	\begin{align*}
		\limsup_{\eta \to \infty} \; \mathbb{G}_{\pi'}^x [ \gamma^{\eta} \!\cdot\! v(X_{\eta}) ] \leq \sup_{x \in \mathcal{X}} v(x) \cdot \lim_{\eta \to \infty} \gamma^\eta = 0
	\end{align*}
which is true since $v$ is upper-bounded (by zero if $\gamma = 1$) and $\gamma \in (0, 1]$. Since $v(x) $ is finite for all $x \in \mathcal{X}$, we have
	\[
		-\infty < v(x) \leq v_{\pi'}(x) \leq \widebar v < \infty \qquad \forall x \in \mathcal{X}		
	\]
by Lemma~\ref{lemma:VF range}. Therefore, $\pi'$ is admissible and $v \leqslant v_{\pi'}$.
\end{Proof*}

\begin{Proof*}{Theorem~\ref{thm:policy improvement thm} (\secref{subsection:policy improvement})}
The policy~$\pi'$ given by \eqref{eq:assumption:policy improvement} satisfies: 
	\[
		h (x, \pi'(x), \nabla v_\pi(x))
		\geq h (x, \pi(x), \nabla v_\pi(x))
		= \alpha \cdot v_{\pi}(x) \qquad \forall x \in \mathcal{X}
	\]
where we substituted the differential BE~\eqref{eq:differential BE}. Therefore, the application of Lemma~\ref{lemma:policy improvement lemma} directly proves the theorem. 
\end{Proof*}

\begin{Proof*}{Theorem~\ref{thm:nc for optimality} (\secref{subsection:HJBE})}
	By optimality and Lemma~\ref{lemma:VF range}, 
	$v \leqslant v_* \leqslant \widebar v$ for any $v \in \mathcal{V}_\mathsf{a}$, implying $v_* \in \mathcal{V}_\mathsf{a}$.
	Since $v_*$ is the VF for the policy $\pi_*$, $v_* = v_{\pi_*}$ and $\pi_* \in \Pi_\mathsf{a}$. 
	Moreover, the maximal policy $\pi'_*$ over $\pi_* \in \Pi_\mathsf{a}$ is also optimal since $\pi_* \preccurlyeq \pi_*'$ by Theorem~\ref{thm:policy improvement thm} and $\pi_*' \preccurlyeq \pi_*$ by the optimality of $\pi_*$, resulting in $v_* = v_{\pi_*} = v_{\pi_*'} \in \mathcal{V}_\mathsf{a}$.  Therefore, the differential BE~\eqref{eq:differential BE} w.r.t. the policy~$\pi = \pi_*'$ and the policy improvement \eqref{eq:assumption:policy improvement} for $\pi' = \pi_*'$, both with $v_{\pi_*} = v_{\pi_*'} = v_*$, result in the HJBE~\eqref{eq:HJBE}. Comparing the HJBE~\eqref{eq:HJBE} with the differential BE~\eqref{eq:differential BE} for $\pi = \pi_*$ and $v_\pi = v_*$, we have \eqref{eq:optimal policy}. 
\end{Proof*}

\subsection{Proofs in \secref{section:fundamental properties of PI} Fundamental Properties of PIs}
\label{appendix:C}

\begin{Proof*}{Theorem~\ref{thm:fundamental properties}}
	$\pi_0$ is admissible by initialization. Suppose for some $i \in \mathbb{N}$ that $\pi_{i-1}$ is admissible. 
	Then, $v_{i} = v_{\pi_{i-1}}$ holds by Theorem~\ref{thm:uniqueness condition} and the boundary condition~\eqref{eq:boundary condition for PI}; $\pi_i$ is also admissible and $\pi_{i-1} \preccurlyeq \pi_{i}$ by Theorem~\ref{thm:policy improvement thm}. 
	Therefore, the mathematical induction completes the proof.
\end{Proof*}

\begin{Proof*}{Theorem~\ref{thm:convergence of PI}}
By Theorem~\ref{thm:fundamental properties} and Lemma~\ref{lemma:VF range}, we have
\[
	v_{1}(x) \leq \cdots \leq v_{i}(x) \leq v_{i+1}(x) \leq \cdots \leq \widebar v < \infty
	\text{ for each fixed $x \in \mathcal{X}$.}
\]
That is, the sequence $\langle v_{i}(x) \rangle$ in $\mathbb{R}$ is monotonically increasing and upper bounded by a constant $\widebar v \in \mathbb{R}$. Hence, $v_i(x)$ converges to ${\hat v}_*(x) \doteq \sup_{i \in \mathbb{N}} v_i(x)$ by monotone convergence theorem \citep[Theorem~2.28]{Thomson2001}, implying the pointwise convergence $v_i \to {\hat v}_*$. 

Next, since every admissible VF is assumed $\mathrm{C}^1$ (see \eqref{eq:assumption:every admissible VF is C1}) and $v_i = v_{\pi_{i-1}}$ is admissible by Theorem~\ref{thm:fundamental properties}, $v_i$ is continuous for each $i \in \mathbb{N}$. Hence, ${\hat v}_*$ is lower semicontinuous \citep[Proposition~7.11c]{Folland1999} and the monotone sequence $\langle v_i \rangle$ converges to ${\hat v}_*$ \emph{uniformly on} $\Omega$ if $\Omega$ is compact and ${\hat v}_*$ is \emph{continuous over} $\Omega$ by Dini's theorem \citep[Theorem~7.13]{Rudin1964}. Finally, $v_i \to {\hat v}_*$ uniformly on any compact $\Omega \subset \mathcal{X}$ if ${\hat v}_*$ is continuous, hence the last statement is obvious by Lemma~\ref{lemma:equivalence b.t.w. locally uniform and compact convergence}.
\end{Proof*}

\begin{Proof*}{Proposition~\ref{prop:a fixed point is a solution to HJBE} (\secref{subsection:convergence})}
Since $v^*$ is a fixed point of $\mathcal{T}$, $\mathcal{T} v^* = v^* \in \mathcal{V}_\mathsf{a}$. Let $\pi^*$ be a policy s.t. 
\begin{equation}
	\pi^*(x) \in \Argmax_{u \in \mathcal{U}}\, h(x, u, \nabla v^*(x)) \qquad \forall x \in \mathcal{X}.
	\label{eq:policy pi^*}
\end{equation}
Then, we have $v_{\pi^*} = \mathcal{T} v^* = v^* \in \mathcal{V}_\mathsf{a}$, hence $\pi^*$ is admissible. Since any admissible policy~$\pi$ satisfies the differential BE~\eqref{eq:differential BE}, it is true for $\pi = \pi^*$, that is, 
$\alpha \cdot v_{\pi^*}(x)  = h(x, \pi^*(x), \nabla v_{\pi^*}(x))$ for all $x \in \mathcal{X}$, 
from which and $v_{\pi^*} = v^*$ we finally obtain
\[
		\alpha \cdot v^*(x) 
		= h(x, \pi^*(x), \nabla v^*(x)) \qquad \forall x \in \mathcal{X}.
\]
Therefore, the substitution of \eqref{eq:policy pi^*} concludes that a fixed point $v^*$ of $\mathcal{T}$ is a solution $v_*$ to the HJBE~\eqref{eq:HJBE}.	
\end{Proof*}

\begin{Proof*}{Theorem~\ref{thm:true convergence in metric} (\secref{subsection:convergence})}
By Lemma~\ref{lemma:uniqueness of fixed point} below and Assumption~\ref{assumption:uniqueness of fixed point}, $v^*$ is a unique fixed point of $\mathcal{T}^N$ for all $N \in \mathbb{N}$. Hence, \citet{Bessaga1959}'s converse of the \citet{Banach1922}'s fixed point theorem ensures that there exists a metric $d$ on $\mathcal{V}_{\mathsf{a}}$ such that $(\mathcal{V}_{\mathsf{a}}, d)$ is a complete metric space and $\mathcal{T}$ is a contraction under~$d$. Then, as $v^*$ is the unique fixed point of $\mathcal{T}$, the \citet{Banach1922}'s fixed point theorem (e.g., see \citealp[Theorem~2.2]{Kirk2013}; or \citealp[Lemma~13.73]{Thomson2001}) shows
	\[
		\forall v_1 \in \mathcal{V}_\mathsf{a}: \lim_{i \to \infty} v_i = \lim_{i \to \infty} \mathcal{T}^{i - 1} v_1 = v^* \textrm{ in the metric $d$,}	
	\]
	implying the convergence $v_i \to v^*$ in the metric $d$.
\end{Proof*}

\begin{lemma} [\secref{subsection:convergence}]
	\label{lemma:uniqueness of fixed point}
		If $v^*$ is a unique fixed point of $\mathcal{T}$, then it is a unique fixed point of $\mathcal{T}^N$ for any $N \in \mathbb{N}$.
\end{lemma}

\begin{Proof}
	Suppose $v^*$ is the unique fixed point of $\mathcal{T}$. Then, it is also a fixed point of $\mathcal{T}^N$ for any $N \in \mathbb{N}$ since 
  	\[
		\mathcal{T}^N v^* = \mathcal{T}^{N-1} ( \mathcal{T} v^* ) = \mathcal{T}^{N-1} v^* = \cdots = \mathcal{T}v^* = v^*. 
	\]
	To show that $v^*$ is the \emph{unique} fixed point of $\mathcal{T}^N$ \emph{for all} $N \in \mathbb{N}$ by contradiction, 
	suppose that there exist $M \in \mathbb{N}$ and $v \in \mathcal{V}_{\mathsf{a}}$ s.t. $\mathcal{T}^M v = v \neq v^*$. Then, the repetitive applications of Theorem~\ref{thm:policy improvement thm} result in 
	 \[
	 	v \leqslant \mathcal{T} v \leqslant \mathcal{T}^2 v \leqslant \cdots \leqslant \mathcal{T}^M v = v
	 \]
	 and thus $\mathcal{T} v = v$. Since $v^*$ is the \emph{unique} fixed point of~$\mathcal{T}$, we have a contradiction, $v = v^*$. Therefore, $v^*$ is the unique fixed point of $\mathcal{T}^N$ for all $N \in \mathbb{N}$, and the proof is completed.
\end{Proof}

\begin{Proof*}{Theorem~\ref{thm:true uniform convergence:1} (\secref{subsection:convergence})}
${\hat v}_* \in \mathcal{V}_\mathsf{a}$ and \eqref{eq:assumption:every admissible VF is C1} imply ${\hat v}_* \in \mathrm{C}^1$ and thus continuity of ${\hat v}_*$. Hence, $v_i$ locally uniformly converges to $\hat v_*$ by Theorem~\ref{thm:convergence of PI}c. This and Lemma~\ref{lemma:equivalence b.t.w. locally uniform and compact convergence} imply that for each compact subset $\Omega$ of $\mathcal{X}$, $v_i \to {\hat v}_*$ in the uniform pseudometric $d_\Omega$. By this and continuity of $\mathcal{T}$ under $d_\Omega$, we have 
\[
 	{\hat v}_* = \lim_{i \to \infty} v_{i+1} = \lim_{i \to \infty} \mathcal{T} v_i = \mathcal{T} \Big  (\lim_{i \to \infty} v_i \Big ) = \mathcal{T}{\hat v}_* \; \textrm{ in the pseudometric } d_\Omega,
\]
implying $d_\Omega({\hat v}_*, \mathcal{T}{\hat v}_*) = 0$ for every compact subset~$\Omega \subset \mathcal{X}$, hence ${\hat v}_* = \mathcal{T} {\hat v}_*$. Therefore, we finally have ${\hat v}_* = v^*$ by Assumption~\ref{assumption:uniqueness of fixed point}, and the proof is completed.
\end{Proof*}

\begin{Proof*}{Theorem~\ref{thm:true uniform convergence:2} (\secref{subsection:convergence})}
$\nabla v_i$ converges locally uniformly by Assumption~\ref{assumption:convergence of nabla v_i and pi_i}a. Hence, for each $x \in \mathcal{X}$, there is a neighborhood $\mathcal{N}_x$ of $x$ on which $\nabla v_i$ converges uniformly. Since a neighborhood $\mathcal{N}_x$ of $x$ contains an open ball $\mathcal{B}_x \doteq \{ y \in \mathcal{X} : \|x - y\| < r\}$ centered at $x$, for some $r > 0$, and every open ball in $\mathcal{X}$ is convex, Lemma~\ref{lemma:uniform convergence of nabla v_i} below ensures that for every $x \in \mathcal{X}$, ${\hat v}_*$ is $\mathrm{C}^1$ over $\mathcal{B}_x$ and $\nabla v_i \to \nabla {\hat v}_*$ uniformly on $\mathcal{B}_x$. This and \smash{$\mathcal{X} = \bigcup_{x \in \mathcal{X}} \mathcal{B}_x$} establish that 	${\hat v}_*$ is $\mathrm{C}^1$ and 
	\begin{equation}
		\nabla v_i \to \nabla {\hat v}_* \textrm{ locally uniformly.}
		\label{eq:appendix:locally uniform limit of the gradients}
	\end{equation}
	Since ${\hat v}_*$ is continuous ($\because$ it is $\mathrm{C}^1$), Theorem~\ref{thm:convergence of PI}c implies that
	$v_i \to {\hat v}_*$ locally uniformly. Let ${\hat \pi}_*: \mathcal{X} \to \mathcal{U}$ be the function to which $\langle \pi_{i} \rangle$ converges pointwise. Such a function ${\hat \pi}_*$ exists by Assumption~\ref{assumption:convergence of nabla v_i and pi_i}b. Then, since each of the $i$th policy~$\pi_i$ satisfies 
	\[
		\pi_i (x) \in \Argmax_{u \in \mathcal{U}} \, h(x, u, \nabla v_i(x)) \qquad \forall x \in \mathcal{X},
	\]
	Assumption~\ref{assumption:closed graph} and \eqref{eq:appendix:locally uniform limit of the gradients} imply that the limit function~${\hat \pi}_*$ holds
	\begin{equation}
		{\hat \pi}_* (x) \in \Argmax_{u \in \mathcal{U}} \, h(x, u, \nabla {\hat v}_*(x)) \qquad \forall x \in \mathcal{X}. 
		\label{eq:appendix:the limit policy}		
	\end{equation}
	Note that for each $i \in \mathbb{N}$, $v_i =v_{\pi_{i-1}} \!\! \in \mathcal{V}_\mathsf{a}$ by Theorem~\ref{thm:fundamental properties}, hence $\pi_{i-1}$ satisfies the differential BE~\eqref{eq:differential BE} for $\pi = \pi_{i-1}$. That is,
	\begin{align*}
		\alpha \cdot  v_i(x) = h(x, \pi_{i-1}(x), \nabla v_i(x)) \qquad \forall x \in \mathcal{X} \quad \forall i \in \mathbb{N}.
	\end{align*}
	Then, taking the pointwise limit $i \to \infty$ on both sides and using continuity of $h$ and \eqref{eq:appendix:the limit policy} results in 
	\begin{align}
		\alpha \cdot {\hat v}_*(x) &= h(x, {\hat \pi}_*(x), \nabla {\hat v}_*(x)) 
		   = \max_{u \in \mathcal{U}} \; h(x, u, \nabla {\hat v}_*(x)) 
		   \qquad \forall x \in \mathcal{X}.
		   \label{eq:appendix:HJBE for v hat}
	\end{align} 
	Here, \eqref{eq:appendix:HJBE for v hat} and \eqref{eq:appendix:the limit policy} are exactly the HJBE~\eqref{eq:HJBE} and \eqref{eq:optimal policy}, respectively, for $v_* = {\hat v}_*$ and $\pi_* = {\hat \pi}_*$, completing the proof. 
\end{Proof*}

\begin{lemma}
	\label{lemma:uniform convergence of nabla v_i}
		If $\nabla v_i$ uniformly converges on an open convex subset $\mathcal{S} \subset \mathcal{X}$, then 
		\[
			\text{${\hat v}_*$ is $\mathrm{C}^1$ over $\mathcal{S}$ and $\nabla v_i \to \nabla {\hat v}_*$ uniformly on $\mathcal{S}$.}
		\]
\end{lemma}

\begin{Proof}
		Let $x \in \mathcal{S}$ and $e_j$ be the unit vector in $\mathcal{X}$ ($= \mathbb{R}^l$) whose $j$th element is 1 (and all the other ones are 0's). Since $\mathcal{S}$ is open, there exists $\theta > 0$ s.t. for each $j$, 
		both 
		\begin{align*}
			x_{j}^+ \doteq x + \dfrac{\theta}{2} \cdot e_j \text{ and }
			x_{j}^- \doteq x - \dfrac{\theta}{2} \cdot e_j			
		\end{align*}
		belong to $\mathcal{S}$. Define a function $g_j: [0, 1] \to \mathcal{S}$ as
		\[
			g_j(\beta) \doteq \beta x_j^+ + (1 - \beta) x_j^-  \text{ for } \beta \in [0, 1],
		\]
		where the dependencies on $x$ and $\theta$ are implicit; by convexity of $\mathcal{S}$, $g_j(\beta) \in \mathcal{S}$ for all $\beta \in [0, 1]$. Then, the composition $v_i \circ g_j$ pointwise converges to ${\hat v}_* \circ g_j$ by Theorem~\ref{thm:convergence of PI}a. Moreover, the derivative $(v_i \circ g_j)'$ (w.r.t. $\beta$) can be expressed by chain rule as 
		\[
			(v_i \circ g_j)'(\beta) = \theta \cdot \! \nabla v_i ( g_j(\beta) ) \, e_j = 
			\theta \cdot \frac{\partial v_i (z) }{\partial z_j}\bigg|_{z = g_j(\beta)}
		\]
		which reveals that $(v_i \circ g_j)'$ is continuous and converges uniformly on $[0, 1]$ since so is $\nabla v_i$ on $\mathcal{S}$ (note that $v_i = v_{\pi_{i-1}} \!\! \in \mathrm{C}^1$ by Theorem~\ref{thm:fundamental properties} and the regularity Assumption~\eqref{eq:assumption:every admissible VF is C1}). Hence, the application of \citep[Theorem~9.34]{Thomson2001} shows that ${\hat v}_* \circ g_j$ is differentiable (w.r.t. $\beta$) and 
		\begin{equation}
			(v_i \circ g_j)' \to ({\hat v}_* \circ g_j)' \textrm{ uniformly on } [0, 1].
			\label{eq:uniform convergence of zeta_i o g}
		\end{equation}
		By definition, the derivative $({\hat v}_* \circ g_j)'(\beta)$ at $\beta = 1/2$ satisfies 
		\begin{align*}
			({\hat v}_* \circ g_j)'(1/2) 
			= \lim_{\epsilon \to 0} \frac{{\hat v}_*(g_j(1/2 + \epsilon)) - {\hat v}_*(g_j(1/2))}{\epsilon}
			= \lim_{\epsilon \to 0} \frac{{\hat v}_*(x + \epsilon \theta \! \cdot \! e_j) - {\hat v}_*(x)}{\epsilon}
			= \theta \cdot \frac{\partial {\hat v}_*(x)}{\partial x_j}.
		\end{align*}
		Since this is true for any $j = \{1,2,\cdots, l\}$ and any $x \in \mathcal{S}$, the gradient $\nabla {\hat v}_*$ exists over $\mathcal{S}$. Moreover, \eqref{eq:uniform convergence of zeta_i o g} at $\beta = 1/2$ implies
		\begin{align*}
			\frac{\partial v_i(x)}{\partial x_j} \to \frac{\partial {\hat v}_*(x)}{\partial x_j} \qquad \forall x \in \mathcal{S} \quad \forall j \in \{1,2,\cdots, l\}, 			
		\end{align*}
		hence $\nabla v_i$ uniformly converges to $\nabla {\hat v}_*$ on $\mathcal{S}$. This also implies that the convergent point $\nabla {\hat v}_*$ is continuous over $\mathcal{S}$ \citep[Theorem~7.12]{Rudin1964}. That is, ${\hat v}_*$ is $\mathrm{C}^1$ over $\mathcal{S}$. 
 \end{Proof}

\subsection{Proofs in \secref{section:case studies} Case Studies}
\label{appendix:proofs:case studies}

Here, we prove all the mathematical statements, including Theorems, Lemmas, Propositions, and Corollaries, w.r.t. each case study presented in the main paper \citep[\secref{section:case studies}]{Lee2020b}. For some proofs in \secref{subsection:RL under u-AC setting}, the following lemmas are required.

\begin{lemma} 
	\label{lemma:continuity of the inverse}
	For any two action spaces $\mathcal{U}$ and $\mathcal{A}$, the inverse of a continuous bijection $g: \Int{\mathcal{U}} \to \Int{\mathcal{A}}$ is continuous.
\end{lemma}
	
\begin{Proof}
	By definitions, $\Int{\mathcal{U}}$ and $\Int{\mathcal{A}}$ are open sets in $\mathbb{R}^m$. Hence, 
	\citet{Brouwer1911}'s invariance of domain theorem implies that $g(\mathcal{O})$ for every open subset $\mathcal{O} \subseteq \Int{\mathcal{U}}$ is open. Hence, the inverse $g^{-1}$ is continuous \citep[Theorem~4.8]{Rudin1964}.
\end{Proof}

\begin{lemma} 
	\label{lemma:uniform convergence of composition}
	Let $\psi: \mathcal{X}^2 \to \mathcal{U}$ be a continuous function. If a sequence $\langle g_i \rangle$ of continuous functions $g_i: \mathcal{X} \to \mathcal{X}$ converges to $g$ locally uniformly, then so does $x \mapsto \psi(x, g_i(x))$ to $x \mapsto \psi(x, g(x))$.
\end{lemma}

\begin{Proof}
Let $\Omega \subset \mathcal{X}$ be compact. Then, by locally uniform convergence $g_i \to g$ and Lemma~\ref{lemma:equivalence b.t.w. locally uniform and compact convergence}, we have the followings:
	\begin{enumerate}
		\item $\langle g_i \rangle$ is uniformly equicontinuous over any compact subset $S \subset \mathcal{X}$ \citep[Theorem~7.24]{Rudin1964}, that is, 
			\begin{equation}
				\text{given $\delta > 0$, there exists $\delta' > 0$ s.t. } \big ( \|x - x'\| < \delta' \; \Longrightarrow \; \|g_i (x) - g_i(x')\| < \delta, \quad \forall x, x' \in S \quad \forall i \in \mathbb{N} \big );
				\label{appendix:proof:lemmaC6:1}		
			\end{equation}
		\item $\langle g_i \rangle$ is uniformly bounded on $\Omega$ (e.g., \citealp[Theorem~7.25]{Rudin1964});
		\item $g$ is continuous over $\Omega$ \citep[Theorem~7.12]{Rudin1964}, hence the image $g(\Omega)$ is compact \citep[Theorem~4.14]{Rudin1964}.
	\end{enumerate}
	In short, we have uniform equicontinuity \eqref{appendix:proof:lemmaC6:1}	over any compact subset $S \subset \mathcal{X}$ and a uniform bound $M > 0$ over $\Omega$, that is,
	\begin{equation}
		\|g_i(x)\| \leq M \text{ and }
		\|g(x)\| \leq M \qquad \forall x \in \Omega \quad \forall i \in \mathbb{N}.
		\label{appendix:proof:lemmaC6:2}		
	\end{equation}	
	
	Next, let $\varepsilon > 0$ and $S_0 \subset \mathcal{X}$ be a compact subset defined as $S_0 \doteq \{ y \in \mathcal{X} : \|y\| \leq M \}$. Then, the function $\psi$ is uniformly continuous over the compact subset $\Omega \times S_0 \subset \mathcal{X}^2$ \citep[Theorem~4.14]{Rudin1964}, hence there exists $\delta > 0$ such that 
	\[
		\|x - x'\| < \delta \text{ and } \|y - y'\| < \delta \quad \Longrightarrow \quad \big \| \psi(x, y) - \psi(x', y') \big \| < \varepsilon,
		\qquad \forall x, x' \in \Omega \quad \forall y, y' \in S_0
	\]
	Since $g_i(x), g(x) \in S_0$ for each $x \in \Omega$ and $i \in \mathbb{N}$ by \eqref{appendix:proof:lemmaC6:2}, the uniform equicontinuity \eqref{appendix:proof:lemmaC6:1} over the compact set $S = S_0$ finally results in: for any $x, x' \in \Omega$ and any $i \in \mathbb{N}$, 
	\[
		\|x - x'\| < \delta^*  \quad \Longrightarrow \quad \big \| \psi(x, g_i(x)) - \psi(x, g_i(x)) \big \| < \varepsilon
	\]
	where $\delta^* \doteq \min \{ \delta, \delta'\}$. That is, the sequence of functions $x \mapsto \psi(x, g_i(x))$ is uniformly equicontinuous over $\Omega$. Moreover, for each $x \in \mathcal{X}$, $y \mapsto \psi(x, y)$ is continuous and $g_i(x)$ converges to $g(x)$, hence $\psi(x, g_i(x))$ converges to $\psi(x, g(x))$. Therefore, $x \mapsto \psi(x, g_i(x))$ converges to $x \mapsto \psi(x, g(x))$ uniformly on $\Omega$ (\citealp[Lemma~39 in Chapter~7]{Royden1988}; or see \citealp[Exercise~16 in Chapter~7]{Rudin1964}); since the compact set $\Omega$ is arbitrary, the proof is completed by Lemma~\ref{lemma:equivalence b.t.w. locally uniform and compact convergence}.
\end{Proof}
	
\begin{Proof*}{Properties of $\mathfrak{c}$ (\secref{subsubsection:case I})} 
The target properties are w.r.t. the gradient and the gradient inverse of $\mathfrak{c}$ shown below:
\begin{enumerate}
	\item $\nabla \mathfrak{c}^\mathsf{T}$ is \emph{bijective}, so that its inverse $\sigma \doteq (\nabla \mathfrak{c}^\mathsf{T})^{-1}$ exists;
	\item $\nabla \mathfrak{c}^\mathsf{T}$ and $\sigma$ are \emph{strictly monotone} and \emph{continuous}.
\end{enumerate}
where $\mathfrak{c}: \mathcal{U} \to \mathbb{R}$ is a function given in \eqref{eq:strictly concave reward}. Here, we prove those properties of $\mathfrak{c}$.

To begin with, recall that $\mathfrak{c}$ is assumed \emph{strictly convex}, $\mathrm{C}^1$, and its gradient $\nabla \mathfrak{c}$ is \emph{surjective}, i.e.,  $\nabla \mathfrak{c}^\mathsf{T}(\Int{\mathcal{U}}) = \mathbb{R}^{m}$. First, we focus on $\nabla \mathfrak{c}^\mathsf{T}$. As $\mathfrak{c}$ is assumed $\mathrm{C}^1$, continuity of $\nabla \mathfrak{c}^\mathsf{T}$ is obvious. To prove strict monotonicity, note that $\mathfrak{c}$ satisfies Lemma~\ref{lemma:strict inequality for strictly convex S} below; by adding the two strict inequalities in Lemma~\ref{lemma:strict inequality for strictly convex S} and rearranging it, we obtain 
	\begin{equation}
			(\nabla \mathfrak{c} (u) - \nabla \mathfrak{c}(u')) (u - u') > 0
			\qquad \forall u \neq u'.
			\label{eq:monotonicity of nabla S}		
	\end{equation}
	Hence, $\nabla \mathfrak{c}^\mathsf{T}$ is strictly monotone.\footnote{The converse (i.e., $\nabla \mathfrak{c}^\mathsf{T}$ is strictly monotone $\Longrightarrow$ $\mathfrak{c}$ is strictly convex) is also true. This equivalence between convexity of $\mathfrak{c}$ and monotonicity of $\nabla \mathfrak{c}^\mathsf{T}$ is known as \citet{Kachurovskii1960}'s theorem.} Moreover, $\nabla \mathfrak{c}^\mathsf{T}$ is injective --- if not, $\exists u$, $u' \in \Int{\mathcal{U}}$ s.t. $u \neq u'$ but $\nabla \mathfrak{c}(u) = \nabla \mathfrak{c}(u')$, which and strict monotonicity of $\nabla \mathfrak{c}$ directly lead us a contradiction ``$0 > 0$'': 
	\[
    	0 = 0^\mathsf{T} (u - u') = (\nabla \mathfrak{c}(u) - \nabla \mathfrak{c}(u')) (u - u') > 0.
	\]
	Therefore, the surjective mapping $\nabla \mathfrak{c}^\mathsf{T}$ is also injective and thus bijective. This ensures the existence of the inverse $\sigma$; since $\nabla \mathfrak{c}^\mathsf{T}$ is continuous, so is its inverse $\sigma$ by Lemma~\ref{lemma:continuity of the inverse} above with $\mathcal{A} =\mathbb{R}^m$. 
	
	To prove strict monotonicity of $\sigma$, let $u \doteq \sigma(\mathfrak{u})$ and $u' \doteq \sigma(\mathfrak{u}')$ for arbitrary $\mathfrak{u}$, $\mathfrak{u}' \in \mathbb{R}^m$. Then, we obviously have $\mathfrak{u} = \nabla \mathfrak{c}^\mathsf{T}(u)$ and $\mathfrak{u}' = \nabla \mathfrak{c}^\mathsf{T}(u')$ and thus, by \eqref{eq:monotonicity of nabla S}, we conclude that 
	$
		(\mathfrak{u} - \mathfrak{u}')^\mathsf{T} (\sigma(\mathfrak{u}) - \sigma(\mathfrak{u}')) > 0 
	$ whenever $\mathfrak{u} \neq \mathfrak{u}'$, hence $\sigma$ is also strictly monotone. This completes the proof.
\end{Proof*}

\begin{lemma} 
\label{lemma:strict inequality for strictly convex S}
For a strictly convex $\mathrm{C}^1$ function $\mathfrak{c}: \mathcal{U} \to \mathbb{R}$ and for any $u$, $u' \in \Int{\mathcal{U}}$ such that $u \neq u'$, 
\[
	\begin{cases}
		\mathfrak{c}(u) > \mathfrak{c}(u') + \nabla \mathfrak{c}(u')(u - u')
		\\[2.5pt]
		\mathfrak{c}(u') > \mathfrak{c}(u) + \nabla \mathfrak{c}(u)(u' - u),
	\end{cases}
\]
where the second inequality is due to the interchange of $u$ and $u'$ of the first one.
\end{lemma}

\begin{Proof}
	Let $g(\beta) \doteq \mathfrak{c} \big ( \beta \cdot u + (1 - \beta) \cdot u' \big )$ for $\beta \in [0, 1]$. Then, $g$ is strictly convex and $\mathrm{C}^1$ since so is $\mathfrak{c}$. By the mean value theorem, there exists $\bar \beta \in (0, 1)$ such that 
	\begin{align*}
		g(1) - g(0) = g'(\bar \beta) > \smash{\lim_{\beta \to 0^+}} g'(\beta)
		= \nabla \mathfrak{c}(u') (u - u'),
	\end{align*}
	where the strict inequality comes from the fact that the derivative $g'$ of a strictly convex $\mathrm{C}^1$ function $g$ is strictly increasing.\footnote{Consider the inequalities for $0 \leq x_1 < x_1' < x_2 < x_2' \leq 1$:
	\[
		\smash{\frac{g(x_1') - g(x_1)}{x_1' - x_1} < \frac{g(x_2) - g(x_1')}{x_2 - x_1'} < \frac{g(x_2') - g(x_2)}{x_2' - x_2}}
	\]  
	(e.g., see \citealp[Theorem~7.5]{Sundaram1996}) and take the limits $x_1' \to x_1$ and $x_2' \to x_2$, resulting in $g'(x_1) < g'(x_2)$ for $x_1 < x_2$.} Then, the proof is completed by substituting the definition of $g$ into the strict inequality.
\end{Proof}

\begin{Proof*}{Theorem~\ref{thm:true convergence} (\secref{subsubsection:case I})}
Combine Lemma~\ref{lemma:concave h} below with Theorem~\ref{thm:true uniform convergence:2}.
\end{Proof*}

\begin{lemma} [\secref{subsubsection:case I}]
	\label{lemma:concave h}
	Under \eqref{eq:input-affine dynamics} and \eqref{eq:strictly concave reward},
	\begin{enumerate}
		\item [\textbf{\emph{a}}.] Assumption~\ref{assumption:closed graph} is true;
				\\[-5pt]
		\item [\textbf{\emph{b}}.] if $\langle \nabla v_i \rangle$ locally uniformly converges to a function $\xi$, then $\langle \pi_i \rangle$ locally uniformly converges to $\pi_\xi$, where 
		\[
			\pi_\xi(x) \doteq \sigma ( F_\mathsf{c}^\mathsf{T} (x) \! \cdot \! \xi^\mathsf{T}(x)).
		\]
	\end{enumerate}
\end{lemma}

\begin{Proof}
	\textbf{a.} Under \eqref{eq:input-affine dynamics} and \eqref{eq:strictly concave reward}, 
	the maximal function~$u_*$ in \eqref{eq:assumption:general condition for policy improvement} is uniquely determined by \eqref{eq:u bar without max} and thus continuous. This also implies that the $\mathrm{argmax}$-set in \eqref{eq:assumption:general condition for policy improvement} is a singleton, hence Assumption~\ref{assumption:closed graph} is equivalent to the continuity of $p \mapsto u_*(x, p)$ (see Remark~\ref{remark:argmax assumption}) which is obviously true by continuity of $u_*$. 
	
	\textbf{b. } For each $i \in \mathbb{N}$, $v_i \in \mathrm{C}^1$ (i.e., $\nabla v_i$ is continuous) since $v_i \in \mathcal{V}_\mathsf{a}$ by Theorem~\ref{thm:fundamental properties} and $\mathcal{V}_\mathsf{a} \subset \mathrm{C}^1$ by \eqref{eq:assumption:every admissible VF is C1}. The function $(x, y) \mapsto \sigma (F_\mathsf{c}^\mathsf{T}(x)y)$ is also continuous since so are $F_\mathsf{c}$ and $\sigma$. Therefore, applying Lemma~\ref{lemma:uniform convergence of composition} with $\psi(x, y) = \sigma(F_\mathsf{c}^\mathsf{T} (x) y )$, $g_i = \nabla v_i$, and $g = \xi$ completes the proof.
\end{Proof}

\begin{Proof*}{Theorem~\ref{thm:true convergence:nonaffine case} (\secref{subsubsection:case II})}
Denoting $a \doteq \varphi(u)$ and considering $a \in \Int{\mathcal{A}}$ as the action transformed from $u \in \Int{\mathcal{U}}$, we can formulate the input-affine dynamics $\bar f$ and the reward function $\bar r$ from \eqref{eq:a class of non-affine dynamics} and \eqref{eq:a class of reward fnc for non-affine dynamics} as 
		\begin{equation}
			\begin{cases}
				{\bar f}(x, a) 
				\doteq f(x,  \varphi^{-1}(a)) 
				= f_{\mathsf{d}}(x) + F_\mathsf{c}(x) \cdot a
			\\[5pt]
				{\bar r}(x, a) 
				\doteq 
				r(x, \varphi^{-1}(a)) 
				= \mathfrak{r}(x) - \mathfrak{c}(a),
			\end{cases}
			\label{eq:proof:thm5.4:transformed RL problem}	
		\end{equation}
	both of which are defined for all $(x, a) \in \mathcal{X} \times \Int{\mathcal{A}}$. The associated Hamiltonian~$\bar h: \mathcal{X} \times \Int{\mathcal{A}} \times \mathcal{X}^\mathsf{T} \to \mathbb{R}$ is given by 
	\begin{align}
		\bar h(x, a, p) 
		= \underbrace{\mathfrak{r}(x) - \mathfrak{c}(a)}_{{\bar r}(x, a)} + \, p \cdot (\underbrace{f_\mathsf{d}(x) + F_\mathsf{c}(x) \cdot a}_{{\bar f}(x, a)} )		
		= h(x, \varphi^{-1}(a), p).
		\label{eq:transformed hamiltonian}
	\end{align}
	Here, both $a \mapsto \bar r(x, a)$ and $a \mapsto \bar h(x, a, p)$ are strictly concave and $\mathrm{C}^1$ for each $x \in \mathcal{X}$. Thus, similarly to the maximal function $u_*$ in \secref{subsubsection:case I}, a maximal function~$a_*: \mathcal{X} \times \mathcal{X}^\mathsf{T} \to \Int{\mathcal{A}}$ such that
	\[
		a_*(x, p) \in \Argmax_{a \in \Int{\mathcal{A}}} {\bar h}(x, a, p)
		\qquad \forall (x, p) \in \mathcal{X} \times \mathcal{X}^\mathsf{T}
	\]
	exists and is continuous as it can be uniquely represented as (see \secref{subsubsection:case I})
	\begin{equation}
		a_* (x, p)
		= \sigma \big ( F_{\mathsf{c}}^\mathsf{T}(x) \, p^\mathsf{T} \big )
		\qquad \forall (x, p) \in \mathcal{X} \times \mathcal{X}^\mathsf{T}.
		\label{eq:a_*}
	\end{equation}
	\vspace{-1.5em}
	\begin{claim} [\secref{subsubsection:case II}]
		\label{claim:maximum lemma}
		$\varphi^{-1} \big [ a_*(x, p) \big ] \in {\Argmax_{u \in \mathcal{U}}} \, h(x, u, p)$
		$\quad \forall (x, p) \in \mathcal{X} \times \mathcal{X}^\mathsf{T}$.
	\end{claim}
	
	By \eqref{eq:a_*} and Claim~\ref{claim:maximum lemma}, a maximal function $u_*$ satisfying \eqref{eq:assumption:general condition for policy improvement} for the RL problem \eqref{eq:a class of non-affine dynamics} and \eqref{eq:a class of reward fnc for non-affine dynamics} is given by 
	\begin{align}
		u_*(x, p)
		= {\tilde \sigma} \big ( F_\mathsf{c}^\mathsf{T}(x) \, p^\mathsf{T} \, \big )
		\qquad \forall (x, p) \in \mathcal{X} \times \mathcal{X}^\mathsf{T},		
		\label{eq:u_* with transformation}
	\end{align}
	where $\tilde \sigma (\mathfrak{u}) \doteq \varphi^{-1} [ \sigma(\mathfrak{u})]$.
	Moreover, $u_*$ is continuous since so are both $a_*$ and the inverse $\varphi^{-1}$ by \eqref{eq:a_*} and Lemma~\ref{lemma:continuity of the inverse}, respectively (note that $u_*(x, p) = \varphi^{-1} \big [ a_*(x, p) \big ]$ by Claim~\ref{claim:maximum lemma} above). Therefore, substituting \eqref{eq:u_* with transformation} into \eqref{eq:closed-form expression of the maximal policy under uniqueness} and \eqref{eq:closed-form expression of optimal policy} result in the following respective closed-form expressions of a maximal policy $\pi'$ over $\pi \in \Pi_\mathsf{a}$ and a HJB policy $\pi_*$:
	\[
		\pi' (x) = \tilde \sigma ( F_{\mathsf{c}}^\mathsf{T} (x) \nabla v_\pi^\mathsf{T} (x) )
		\text{ and }
		\pi_* (x)= \tilde \sigma ( F_{\mathsf{c}}^\mathsf{T} (x) \nabla v_*^\mathsf{T} (x) ).
	\]
	
	Next, substituting \eqref{eq:transformed hamiltonian} and $\pi_*(x) = \varphi^{-1} [ a_*(x, \nabla v_*(x)) ]$ into the HJBE~\eqref{eq:HJBE}, we obtain the HJBE w.r.t. $\bar h$, for the same $v_*$:
	\begin{align*}
		\alpha \cdot v_*(x) 
			= h (x, \pi_*(x), \nabla v_*(x) ) 
			= {\bar h}\big (x, a_*(x, \nabla v_*(x)), \nabla v_*(x) \big ) 
			= \max_{a \in \Int{\mathcal{A}}} {\bar h}(x, a, \nabla v_*(x)) \qquad \forall x \in \mathcal{X}.
	\end{align*}
	In addition, the PI running on the original RL problem \eqref{eq:a class of non-affine dynamics} and \eqref{eq:a class of reward fnc for non-affine dynamics}, with its policy improvement $\pi_{i}(x) =
		\tilde \sigma \big ( F_{\mathsf{c}}^\mathsf{T}(x) \nabla v_i^\mathsf{T}(x) \big )$, results in the same VFs as the PI running on the transformed one~\eqref{eq:proof:thm5.4:transformed RL problem}. This is because once a policy~$\pi$ is admissible, 
	\[
		\alpha \cdot v_\pi(x) = h(x, \pi(x), \nabla v_\pi(x))  = \bar h(x, \bar \pi(x), \nabla v_{\pi}(x)) \qquad \forall x \in \mathcal{X}	
	\]
	for the policy $\bar \pi$ given by $\bar \pi (x) \doteq \varphi(\pi(x))$, by the differential BE~\eqref{eq:differential BE} and \eqref{eq:transformed hamiltonian}. This implies that applied to the transformed RL problem, the policy $\bar \pi$ is admissible and its VF is equal to the original VF $v_\pi$ by Theorem~\ref{thm:uniqueness condition}. 
	
	Therefore, the application of Theorem~\ref{thm:true convergence} to the transformed RL problem shows that for both cases, the limit function ${\hat v}_*$ is a solution $v_* \in \mathrm{C}^1$ to the HJBE s.t. $v_i \to v_*$ and $\nabla v_i \to \nabla v_*$ both locally uniformly. For locally uniform convergence of $\langle \pi_i \rangle$ towards $\pi_*$, apply Lemma~\ref{lemma:uniform convergence of composition} with $\psi(x, y) = \tilde \sigma (F_{\mathsf{c}}^\mathsf{T} (x) y)$, $g_i = \nabla v_i$, and $g = \nabla v_*$.

\textbf{(Proof of Claim~\ref{claim:maximum lemma}).} Fix $(x, p) \in \mathcal{X} \times \mathcal{X}^\mathsf{T}$ and note that the associated Hamiltonian $\bar h$ satisfies (see \eqref{eq:transformed hamiltonian})
		\begin{equation}
			\bar h(x, a, p) = h(x, \varphi^{-1}(a), p)
			\quad \forall a \in \Int{\mathcal{A}}.
			\label{eq:transformed hamiltonian in proof}	
		\end{equation}
		Since $\varphi$ is a bijection between the interior spaces $\Int{\mathcal{U}}$ and $\Int{\mathcal{A}}$, we have $\varphi(\Int{\mathcal{U}}) = \Int{\mathcal{A}}$, which and \eqref{eq:transformed hamiltonian in proof} imply
				\begin{align}
					\max_{a \in \Int{\mathcal{A}}} \bar h(x, a, p) 
					= 
					\max_{u \in \Int{\mathcal{U}}} \bar h(x, \varphi(u), p)
					=
					\max_{u \in \Int{\mathcal{U}}} h(x, u, p).
					\label{eq:maximum conversion eq}
				\end{align}
		For simplicity, denote $a_*(x, p)$ by $a_*$ and $u_* \doteq  \varphi^{-1} (a_*)$. Here, $a_*$ and $u_*$ belong to the interiors $\Int{\mathcal{A}}$ and $\Int{\mathcal{U}}$, respectively. So, 
			\begin{align*}
				\max_{a \in \Int{\mathcal{A}}} \bar h(x, a, p)
				&=
				\bar h(x, a_*, p) 
				=
				h(x, \varphi^{-1}(a_*), p)	
				=
				h(x, u_*, p)
				\qquad \text{by \eqref{eq:transformed hamiltonian in proof}.}
			\end{align*}
			This and \eqref{eq:maximum conversion eq} imply that $u_*$ satisfies 
			\[
				u_* (x, p) \in \Argmax_{u \in \Int{\mathcal{U}}} h(x, u, p).				
			\]
			This proves the statement if $\Bdy{\mathcal{U}} = \varnothing$. If not, suppose that there exists a $\tilde u \in \Bdy{\mathcal{U}}$ on the boundary $\Bdy{\mathcal{U}}$ s.t. 
			\begin{equation}
				h(x, \tilde u, p) > h(x, u, p) \textrm{ for all } u \in \Int{\mathcal{U}}.
				\label{eq:maximum on the boundary}
			\end{equation}
			Then, by continuity of $h$ and the definition of a boundary, 
			for $\varepsilon = h(x, \tilde u, p) - h(x, u_*, p) > 0$, there exists ${\hat u}_*$ in the interior $\Int{\mathcal{U}}$ s.t. $h(x, {\tilde u}, p) - h(x, {\hat u}_*, p) < \varepsilon$, which implies 
			\[
				h(x, {\hat u}_*, p) > h(x, u_*, p),
			\]
			meaning that $u_* \in \Int{\mathcal{U}}$ is not a maximum of the mapping $u \mapsto h(x, u, p)$ over the interior $\Int{\mathcal{U}}$, a contradiction. Therefore, there is no $\tilde u \in \Bdy{\mathcal{U}}$ s.t. \eqref{eq:maximum on the boundary} holds; we conclude that $u_*$ is a maximum of the mapping $u \mapsto h(x, u, p)$ over $\Int{\mathcal{U}} \cup \Bdy{\mathcal{U}} = \mathcal{U}$. 
\end{Proof*}

\begin{Proof*}{Proposition~\ref{prop:continuity of VFs} (\secref{subsection:discounted RL with bounded v})}
	Fix the policy~$\pi$ and for simplicity, denote with slight abuse of notation 
	\[
		v \doteq v_\pi, \quad X_t(x) \doteq \mathbb{G}_\pi^x[X_t], \quad R_t(x) \doteq \mathbb{G}_\pi^x [ R_t ].
	\]
	Here, the dependencies on the policy~$\pi$ are implicit, and for each $x \in \mathcal{X}$, the state trajectory $t \mapsto X_t(x)$ is assumed to exist uniquely for all $t \geq 0$ (see~\secref{section:preliminaries}). Also note that $R_t(x) = r_\pi(X_t(x))$ since 
	\[
		R_t(x) = \mathbb{G}_\pi^x [ R_t ] = r_\pi(\mathbb{G}_\pi^x [ X_t ]) = r_\pi(X_t(x)).
	\]

	Suppose $v$ is bounded and fix $x_0 \in \mathcal{X}$. Then, by continuity of $r_\pi$ and continuous dependency of $(t, x) \mapsto X_t(x)$ on $x$ \citep[Theorem 3.5]{Khalil2002}, we have: for any $\eta > 0$ and any $\beta > 0$, there exists $\delta \equiv \delta(\beta, \eta) > 0$ such that
	\[
		\|x - x_0\| < \delta \;\Longrightarrow\; 
		\big | R_t (x) - R_t(x_0) \big | < \beta \;\;\,  
		\forall t \in [0, \eta],
	\]
	from which and the integral BE~\eqref{eq:Bellman eq}, we obtain that whenever $\|x - x_0 \| < \delta$, 
	\begin{align*}
		\big |v(x) - v(x_0) \big | \leq  
		\int_0^{\eta} \gamma^{t} \!\cdot\! \big | R_t(x) - R_t(x_0) \big |  \, dt  + \gamma^{\eta} \! \cdot \big |v\big (X_{\eta}(x) \big ) \big | +  \gamma^{\eta} \! \cdot \big | v\big (X_{\eta}(x_0) \big ) \big | < \beta \cdot \eta + 2 \cdot \gamma^\eta \cdot M,
	\end{align*}
	where $M > 0$ is a bound of $v$, i.e., a positive constant such that $\sup_{x \in \mathcal{X}} |v(x)| \leq M$. Since $\beta, \eta > 0$ are arbitrary, for given $\varepsilon > 0$, choose $\beta = \varepsilon / 2\eta$ and any $\eta > 0$ s.t. $\gamma^\eta < \varepsilon / (4 M)$. Then, we conclude that for any $\varepsilon > 0$, there exists $\delta  \equiv \delta(\varepsilon) > 0$ s.t.
    \[
    	\|x - x_0 \| < \delta \; \Longrightarrow \; \big |v(x) - v(x_0) \big | < \frac{\varepsilon}{2} + \frac{\varepsilon}{2} = \varepsilon,
    \] 
    the $\varepsilon$-$\delta$ statement of the continuity of $v$ ($=v_\pi$) at $x_0$, and the proof is completed as $x_0 \in \mathcal{X}$ is arbitrary. 
\end{Proof*}

\begin{Proof*}{Proposition~\ref{prop:boundary condition true when discounted and bounded} (\secref{subsection:discounted RL with bounded v})}
	If $v$ is bouned, then since $\mathbb{G}_\pi^x [ v(X_t) ] = v \big ( \mathbb{G}_\pi^x [ X_t] \big )$, $t \mapsto \mathbb{G}_\pi^x [ v(X_t) ]$ for any policy~$\pi$ is bounded over $\mathbb{T}$ (uniformly in $x \in \mathcal{X}$), hence by Lemma~\ref{lemma:boundary condition on v_* with discounting} below, $v$ satisfies the boundary condition~\eqref{eq:uniqueness condition for v}.
\end{Proof*}

\begin{lemma}
	\label{lemma:boundary condition on v_* with discounting}
	In the discounted case, the boundary condition \eqref{eq:uniqueness condition for v} is true if $t \mapsto \mathbb{G}_{\pi}^x  [ v(X_{t}) ]$ is bounded for each $x \in \mathcal{X}$.
\end{lemma}

\begin{Proof}
	For $x \in \mathcal{X}$, let $M_x > 0$ be a constant s.t. $\sup_{t \in \mathbb{T}} \mathbb{G}_{\pi}^x  \big [ |v(X_{t})| \big ] \leq M_x$. Then, since $\gamma \in (0, 1)$, we have
	\[
		0 \leq \lim_{t \to \infty} \mathbb{G}_{\pi}^x \big  [ \, \gamma^{t} \cdot | v(X_{t}) |  \, \big ] 
		\leq
		\lim_{t \to \infty} M_x \cdot \gamma^{t} = 0 \qquad \forall x \in \mathcal{X},
	\]
	implying the boundary condition \eqref{eq:uniqueness condition for v}.
\end{Proof}

\begin{Proof*}{Corollary~\ref{cor:fundamental policy eval and imp thm for bounded v} (\secref{subsection:discounted RL with bounded v})}
		For the first part, since $v$ is bounded, Proposition~\ref{prop:boundary condition true when discounted and bounded} and Theorem~\ref{thm:uniqueness condition} ensure $v = v_\pi$, hence $v_\pi$ is bounded. Next, if $v_\pi$ is bounded (hence admissible), then we have $v_{\pi} \leqslant v_{\pi'} \leqslant \widebar v$ by Theorem~\ref{thm:policy improvement thm} and Lemma~\ref{lemma:VF range}, and thus $v_{\pi'}$ is also bounded.
\end{Proof*}

\begin{Proof*}{Corollary~\ref{cor:properties when R is bounded} (\secref{subsection:discounted RL with bounded v})}
	Under Assumption~\ref{assumption:bounded R}, we can choose $\ushort \alpha$ and $\xi(x)$ in the lower bound \eqref{eq:exponential increasing condition for admissibility} of $\mathbb{G}_\pi^x [ R_t ]$ as $\ushort \alpha = 0$ and a constant function $\xi(x) \equiv \inf \{ r(y, u) : (y, u) \in \mathcal{X} \times \mathcal{U} \} \in \mathbb{R}$. Hence, $v_\pi$ is bounded by Proposition~\ref{prop:general admissible condition}, for any given policy~$\pi$. The remaining proof is now obvious by Proposition~\ref{prop:continuity of VFs} and Corollary~\ref{cor:fundamental policy eval and imp thm for bounded v}.
\end{Proof*}

\begin{Proof*}{Theorem~\ref{thm:policy improvement thm in locally Lipschitz case} (\secref{subsection:RL with local Lipschitzness})}
	Since the differential BE~\eqref{eq:differential BE} and the $\mathrm{argmax}$-formula \eqref{eq:assumption:policy improvement} are true for $\pi \in \Pi_\mathsf{a}$ and the maximal policy~$\pi'$ over it, they satisfy the inequality  \eqref{eq:policy improvement inequality} in Lemma~\ref{lemma:policy improvement lemma} for $v = v_\pi$. Hence,
	\begin{align*}
		{\dot v}_\pi (x, \pi'(x)) 
		\geq - r_{\pi'}(x) + \alpha \cdot v_\pi(x) 
		\geq - (r_\mathsf{max} - \alpha \cdot v_\pi(x)) = - \alpha \cdot (\widebar v - v_\pi(x))
		\qquad \forall x \in \mathcal{X},
	\end{align*}
	where the last equality comes from Lemma~\ref{lemma:VF range}. Let $J_\pi \doteq \widebar v - v_\pi$. Then, the inequality can be expressed as 
	\[
		{\dot J}_\pi (x, \pi'(x)) \leq \alpha J_\pi(x) \qquad \forall x \in \mathcal{X},
	\]
	by substituting ${\dot v}_\pi = - {\dot J}_\pi$ and rearranging it. By $\pi \in \Pi_\mathsf{a}$ and the Assumptions, we see that $u_*$ and $\nabla v_\pi$ are locally Lipschitz. Hence, $\pi'$ given by \eqref{eq:closed-form expression of the maximal policy under uniqueness} is locally Lipschitz (i.e., $\pi' \in \Pi_\mathsf{Lip}$); the application of Lemma~\ref{lemma:existence and uniqueness of state trj.} below results in $t_\mathsf{max}(x; \pi') = \infty$ for all $x \in \mathcal{X}$. Now that the state trajectories exist globally and uniquely, we conclude $\pi \preccurlyeq \pi' \in \Pi_\mathsf{a}$ by Theorem~\ref{thm:policy improvement thm}.
\end{Proof*}

\begin{lemma} [\secsref{subsection:RL with local Lipschitzness} and \ref{subsection:nonlinear optimal control}]
	\label{lemma:existence and uniqueness of state trj.}
	If there exist  a $\mathrm{C}^1$ function $J: \mathcal{X} \to \mathbb{R}$ and $\mathcal{K}_\infty$-functions $\rho_1$ and $\rho_2$ s.t. for all $x \in \mathcal{X}$, 
	\begin{align}
		&\rho_1(\|x\|_\Omega) \leq J(x) \leq \rho_2(\|x\|_\Omega)
		\label{eq:appendix:first-inequality-for-existence-of-unique-state-trj}
		\\[5pt]
		&{\dot J}(x, \pi(x)) \leq \lambda \cdot J(x) 
		\label{eq:appendix:second-inequality-for-existence-of-unique-state-trj}
	\end{align}
for a compact subset $\Omega \subset \mathcal{X}$, a constant $\lambda \in \mathbb{R}$, and a policy $\pi \in \Pi_\mathsf{Lip}$, then $t_\mathsf{max}(x; \pi) = \infty$ for all $x \in \mathcal{X}$.
\end{lemma}

\begin{Proof}
	First, $t_\mathsf{max}(x; \pi) \in (0, \infty]$ is well-defined for each $x \in \mathcal{X}$ since $\pi \in \Pi_\mathsf{Lip}$ and thus $f_{\pi}$ is locally Lipschitz. Applying the \citet{Gronwall1919}'s inequality to \eqref{eq:appendix:second-inequality-for-existence-of-unique-state-trj}, we obtain $\mathbb{G}_{\pi}^x [J(X_t)] \leq e^{\lambda t} \cdot J(x)$ and by \eqref{eq:appendix:first-inequality-for-existence-of-unique-state-trj},
	\[
		\mathbb{G}_{\pi}^x [\rho_1(\|X_t\|_\Omega)] \leq \mathbb{G}_{\pi}^x [J(X_t)] \leq e^{\lambda t} \cdot \rho_2(\|x\|_\Omega)
		\qquad \forall x \in \mathcal{X}.
	\]	 
	Therefore, the proof is completed by applying Lemma~\ref{lemma:existence and uniqueness of state trj. bounded exponentially} below for each $x \in \mathcal{X}$, with $\rho = \rho_1$ and $\bar \rho(x, t) = e^{\lambda t} \rho_2(\|x\|_\Omega)$.
\end{Proof}

\begin{lemma} [\secsref{subsection:RL with local Lipschitzness} and \ref{subsection:nonlinear optimal control}]
	\label{lemma:existence and uniqueness of state trj. bounded exponentially}
	Let $\Omega \subset \mathcal{X}$ be compact. Given a policy $\pi$ and $x\in \mathcal{X}$, 
	if there exist functions $\rho: [0, \infty) \to [0, \infty)$ and $\bar \rho: \mathcal{X} \times [0, \infty) \to [0, \infty)$ s.t. 
	\begin{enumerate}
		\item both $\rho$ and $t \mapsto \big ( \bar \rho(x, t) - \bar \rho(x, 0) \big ) $ are $\mathcal{K}_\infty$;
		\\[-5pt]
		\item $\mathbb{G}_\pi^x [\rho \big ( \|X_t\|_\Omega \big ) ] \leq {\bar \rho} (x, t)$ for all $t \in [0, t_\mathsf{max}(x; \pi))$,
	\end{enumerate}
	then $t_\mathsf{max}(x; \pi) = \infty$.
\end{lemma}

\begin{Proof}
	Since the inverse of a $\mathcal{K}_\infty$ function $\rho^{-1}$ exists and is also $\mathcal{K}_\infty$ \citep[Lemma~4.2]{Khalil2002}, we obtain
	\[
		\mathbb{G}_{\pi}^x [\|X_t\|_\Omega] \leq {\tilde \rho}(x, t) \doteq \rho^{-1}\big ( \bar \rho(x, t) \big )
		\qquad \forall t \in [0, t_\mathsf{max}(x; \pi)).
	\]
	For the proof, we suppose \emph{$t_\mathsf{max}(x; \pi)$ is finite} and show a contradiction $t_\mathsf{max}(x; \pi) = \infty$. First, $t \mapsto \mathbb{G}_{\pi}^x [\|X_t\|_\Omega]$ is bounded by 
	\[
		{\tilde \rho}_\mathsf{max}(x; \pi) \doteq \sup \big \{  {\tilde \rho}(x, t) : 0 \leq t \leq t_\mathsf{max}(x; \pi) \big \} = {\tilde \rho}(x, t_\mathsf{max}(x; \pi)).
	\] 
	Thus, the state trajectory $t \mapsto \mathbb{G}_{\pi}^x [X_t]$, defined for all $t \in [0, t_\mathsf{max}(x))$, remains within the compact\footnote{\smash{${\widebar \Omega}(x; \pi)$ is compact (i.e., closed and bounded) by its definition since (i) so is $\Omega$ and (ii) $t_\mathsf{max}(x; \pi)$ (hence, ${\tilde \rho}_\mathsf{max}(x; \pi)$) is finite.}} set ${\widebar \Omega}(x; \pi) \supset \Omega$ given by
	\[
		{\widebar \Omega}(x; \pi) \doteq \big \{ y \in \mathcal{X} : \| y \|_\Omega \leq {\tilde \rho}_\mathsf{max}(x; \pi) \big \}
	\]
	and thereby is uniquely defined \emph{for all $t \in \mathbb{T}$} by Proposition~\ref{prop:existence and uniqueness of state trj} in \secref{subsection:discounted RL with bounded state trj}, leading to a contradiction: $t_\mathsf{max}(x; \pi) = \infty$.
\end{Proof}

\begin{Proof*}{Lemma~\ref{lemma:positive definiteness of c_pi} (\secref{subsection:nonlinear optimal control})}
	The positive definiteness of $c_\pi$ is obvious by \eqref{eq:n.d. R in optimal control} and $c_\pi(0) = c(0, \pi(0)) = 0$ ($\because$ $\pi(0) = 0$).
\end{Proof*}

\begin{Proof*}{Lemma~\ref{lemma:Lyapunov v} (\secref{subsection:nonlinear optimal control})}
\textbf{a.} Since $\pi \in \Pi_\mathsf{a} \subseteq \Pi_0$, 
\begin{enumerate}
	\item $c_\pi(0) = 0$ by Lemma~\ref{lemma:positive definiteness of c_pi}; 	
	\item $x_\mathsf{e} = 0$ is an equilibrium point under~$\pi$ ($\because$ $f_\pi(0) = f(0, \pi(0)) = f(0, 0) = 0$), that is, $\mathbb{G}_\pi^0 [ X_t ] \equiv 0$.
\end{enumerate}
Hence, $\mathbb{G}_\pi^0 [C_t ] = c_\pi \big ( \mathbb{G}_\pi^0 [X _t ] \big ) = c_\pi(0) = 0$ for all $t \in \mathbb{T}$, implying $J_\pi(0) = 0$. 

By $\pi \in \Pi_\mathsf{a}$, we also have $t_\mathsf{max}(x; \pi) = \infty$ and 
\[
	J_\pi(x) = \mathbb{G}_\pi^x \bigg [ \int_0^\infty \gamma^t \cdot C_t \, dt  \bigg ] \in [0, \infty)
	\qquad \forall x \in \mathcal{X}.
\]
Since $c_\pi(0) = 0$ and $c_\pi(x) > 0$ for any $x \neq 0$ by Lemma~\ref{lemma:positive definiteness of c_pi}, and $t \mapsto \mathbb{G}_\pi^x [C_t]$ is continuous ($\because$ so are $c_\pi$ and $t \mapsto \mathbb{G}_\pi^x[X_t]$ --- see \secref{subsection:RL with local Lipschitzness}), we have that for each $x \neq 0$, there exists $\eta > 0$ s.t. $\inf_{0 \leq t < \eta} \mathbb{G}_\pi^x [C_t ] > 0$. Therefore, by the integral BE~\eqref{eq:Bellman eq}, 
\[
	J_\pi(x) \geq \bigg  ( \! \inf_{0 \leq t < \eta} \mathbb{G}_\pi^x [C_t ] \bigg ) \cdot \!\!\int_0^\eta \gamma^t \, dt + \gamma^\eta \! \cdot \mathbb{G}_\pi^x [ J_\pi(X_\eta) ] >\gamma^\eta \! \cdot \mathbb{G}_\pi^x [ J_\pi(X_\eta) ] \geq 0
	\qquad \forall x \neq 0,
\]
that is, $J_\pi(x) > 0$ for each $x \neq 0$. This and $J_\pi(0) = 0$ prove that $J_\pi$ is positive definite. 

\textbf{b.} and \textbf{c.} Since $\pi \in \Pi_\mathsf{a}$ satisfies the differential BE \eqref{eq:differential BE}, by Lemma~\ref{lemma:appendix:conversion b.t.w. BEs for arbitrary eta} with $J_\pi = - v_\pi$, $c_\pi = -r_\pi$, and ${\dot  J}_\pi = - {\dot v}_\pi$, we have 
 	\begin{align}
		{\dot J}_\pi(x, \pi(x)) 
		&= \alpha J_\pi(x) - c_\pi(x)
		\qquad \forall x \in \mathcal{X}.
		\label{eq:optimal control:nd of v dot}
	\end{align}
	Here, ${\dot J}_\pi(0, \pi(0)) = 0$ since both $J_\pi$ and $c_\pi$ are positive definite; the proof is now obvious by \eqref{eq:condition for stability}, \eqref{eq:condition for asymp stability}, and \eqref{eq:optimal control:nd of v dot}. 	
\end{Proof*}

\newpage
\begin{Proof*}{Theorem~\ref{thm:stability of optimal control} (\secref{subsection:nonlinear optimal control})}
Given $\pi \in \Pi_\mathsf{a}$, the inequality \eqref{eq:condition for asymp stability} in Lemma~\ref{lemma:Lyapunov v} is true whenever $\gamma = 1$ (i.e., $\alpha = 0$) since $c_\pi$ is positive definite by Lemma~\ref{lemma:positive definiteness of c_pi}. Therefore, the proof is obvious by Lemma~\ref{lemma:Lyapunov v} and Lyapunov's stability theorems \citep[Theorems 4.1 and 4.2]{Khalil2002}, except that the asymptotic stability is global when ``$\gamma = 1$ but $J_\pi$ \emph{is not radially unbounded (but radially nonvanishing)}''. To prove this case, fix $\pi \in \Pi_\mathsf{a}$ and let $x_\mathsf{e} = 0$ be asymptotically stable under $\pi$. $\mathcal{B}_\pi \subseteq \mathcal{X}$ denotes the basin of attraction under $\pi$, i.e., the set of all points $x \in \mathcal{X}$ s.t. $X_t(x) \to 0$ as $t \to \infty$, where $X_t(x) \doteq \mathbb{G}_\pi^x [X_t]$ denotes the state trajectory under $\pi$ starting at $X_0 = x \in \mathcal{X}$. Here, the dependency of $X_t(x)$ on $\pi$ is implicit. Also note that 
\begin{enumerate}
	\item since $\mathcal{B}_\pi$ is open \citep[Lemma~8.1]{Khalil2002} and contains the origin $x_\mathsf{e} = 0$, 
			there exists $r > 0$ such that 
			\begin{align}
				\|x\| < r \quad \Longrightarrow \quad x \in \mathcal{B}_\pi;
				\label{eq:proof:stability:1}
			\end{align}
	\item since $c_\pi \geqslant 0$ is positive definite by Lemma~\ref{lemma:positive definiteness of c_pi}, continuous ($\because$ so is $r$ and $\pi \in \Pi_\mathsf{a}$ by definitions), and assumed radially nonvanishing (i.e., $\lim_{r \to \infty} \inf \{c_\pi(x) : \|x\| \geq r\} \neq 0$ --- see \secref{appendices:notations:4}), we have
	\begin{equation}
		 \varphi_\pi(r) \doteq \inf \{ c_\pi(x): \| x \| \geq r \} > 0 \qquad \forall r > 0;
		\label{eq:proof:stability:2}
	\end{equation}
	\item by time-invariance $X_{\tau + t} (x) = X_t ( X_{\tau}(x))$ (and noting that $t_\mathsf{max}(x; \pi) = \infty$ for all $x \in \mathcal{X}$ by $\pi \in \Pi_\mathsf{a}$), we have
		\begin{equation}
			\label{eq:proof:stability:3}		
			x \not \in \mathcal{B}_\pi \quad \Longrightarrow \quad 
			X_t(x) \not \in \mathcal{B}_\pi \text{ for all } t \geq 0
		\end{equation}
		--- if $X_\tau(x) \in \mathcal{B}_\pi$ for some $\tau > 0$, then $x \in \mathcal{B}_\pi$ ($\because$ $\displaystyle \smash{\lim_{t \to \infty}} X_{\tau + t}(x) = \smash{\lim_{t \to \infty}} X_t(X_\tau(x)) = 0$).
\end{enumerate} 

The proof will be done by contradiction. Suppose $\mathcal{B}_\pi \neq \mathcal{X}$. 
Then, it implies that there exists $x \not \in \mathcal{B}_\pi$ in $\mathcal{X}$ and $r > 0$ such that
	\begin{equation}
		\|X_t(x)\| \geq r \text{ for all } t \in \mathbb{T}
		\label{eq:proof:stability:4}
	\end{equation}
by \eqref{eq:proof:stability:3} and then the contraposition of \eqref{eq:proof:stability:1}.
Finally, applying \eqref{eq:proof:stability:2} and \eqref{eq:proof:stability:4} to the cost VF for $\gamma = 1$ yields
	\begin{align*}
		J_\pi(x) 
		= \lim_{\eta \to \infty} \int_0^\eta c_\pi(X_t(x)) \, dt 
		\geq \varphi_\pi(r) \cdot \lim_{\eta \to \infty} \int_0^\eta 1 \, dt 
		= \infty,
	\end{align*}
	 a contradiction to $\pi \in \Pi_\mathsf{a}$. Therefore, $\mathcal{B}_\pi = \mathcal{X}$ and thus the asymptotic stability under $\pi \in \Pi_\mathsf{a}$ is global whenever $\gamma = 1$.
\end{Proof*}

\begin{Proof*}{Theorem~\ref{thm:uniqueness under gas} (\secref{subsection:nonlinear optimal control})}
	Since $x_\mathsf{e} = 0$ under $\pi$ is globally attractive, for each $x \in \mathcal{X}$, (i) the state trajectory $t \mapsto \mathbb{G}_\pi^x [X_t]$ exists uniquely and globally over $\mathbb{T}$, and (ii) $\mathbb{G}_\pi^x [X_t] \to 0$ as $t \to \infty$. Hence, continuity of $v$ at $0$ and $v(0) = 0$ imply that $\mathbb{G}_\pi^x [\gamma^{t} \, v(X_{t})] \to 0$ as $t \to \infty$, for all $x \in \mathcal{X}$; the proof is completed by Theorem~\ref{thm:uniqueness condition}.	
\end{Proof*}

\begin{Proof*}{Theorem~\ref{thm:uniqueness condition in optimal control} (\secref{subsection:nonlinear optimal control})}
	Since $J$ ($\doteq -v$) is positive definite, $v$ is upper-bounded by zero. $\kappa J \leqslant c_\pi$ is equivalent to $r_\pi \leqslant \kappa \cdot v$ by definitions. Therefore, if the state trajectories $t \mapsto \mathbb{G}_\pi^x [ X_t]$ are uniquely defined over $\mathbb{T}$ (i.e., $t_\mathsf{max}(x; \pi) = \infty$), for all $x \in \mathcal{X}$, then the inequality~\eqref{eq:optimal control:exponential increasing condition for admissibility} in the case \textbf{b} is equivalent to \eqref{eq:exponential increasing condition for admissibility} for $\xi = - \zeta$, and the application of Theorem~\ref{thm:boundary condition:sufficiency} concludes $\pi \in \Pi_\mathsf{a}$ and $v = v_\pi$ for both cases \textbf{a} and \textbf{b}. The followings are the proofs of $t_\mathsf{max}(x; \pi) = \infty$ for all $x \in \mathcal{X}$ for each case. Also note that $c_\pi$ is positive definite by Lemma~\ref{lemma:positive definiteness of c_pi} and $\pi \in \Pi_0$.
	
	\textbf{a}. First, $J$ satisfies 
	$
		\alpha J(x) = c_\pi(x) + {\dot J}(x, \pi(x)) \geq \kappa J(x) + {\dot J}(x, \pi(x)) 
	$ 
	for all $x \in \mathcal{X}$,
	by Lemma~\ref{lemma:appendix:conversion b.t.w. BEs for arbitrary eta} and $\kappa J \leqslant c_\pi$. That is, ${\dot J}(x, \pi(x)) \leq (\alpha - \kappa) \cdot J(x)$ for all $x \in \mathcal{X}$. Since $J$ is assumed $\mathrm{C}^1$ and radially unbounded, Lemma~\ref{lemma:bound of a positive definite function} below implies that there exist $\mathcal{K}_\infty$ functions $\rho_1$ and $\rho_2$ s.t. 
	$
		\rho_1(\|x\|) \leq J (x) \leq \rho_2(\|x\|)
	$ for all $x \in \mathcal{X}$. Therefore, the application of Lemma~\ref{lemma:existence and uniqueness of state trj.} with $\Omega = \{0\}$ (i.e., with $\|\cdot\|_\Omega = \|\cdot\|$) proves that $t_\mathsf{max}(x ; \pi) = \infty$ for all $x \in \mathcal{X}$.
	
	\textbf{b.} Since $c_\pi$ is positive definite, continuous by definitions, and radially unbounded by assumption, there exists a $\mathcal{K}_\infty$ function $\rho$ s.t. $\rho(\|x\|) \leq c_\pi(x)$ for all $x \in \mathcal{X}$ by Lemma~\ref{lemma:bound of a positive definite function}. Hence, we obtain that for each $x \in \mathcal{X}$, 
	\[
		\mathbb{G}_\pi^x [ \rho( \|X_t\|) ] \leq \mathbb{G}_\pi^x [ c_\pi(X_t) ] = \mathbb{G}_\pi^x [ C_t ] \leq \zeta(x) \cdot \exp(\ushort \alpha t) \quad \forall t \in [0, t_\mathsf{max}(x; \pi) );
	\]
	the application of Lemma~\ref{lemma:existence and uniqueness of state trj. bounded exponentially} for each $x \in \mathcal{X}$, with $\bar \rho(x, t) = \zeta(x) \exp(\ushort \alpha t)$, concludes $t_\mathsf{max}(x; \pi) = \infty$ for all $x\in \mathcal{X}$. 
\end{Proof*}

\begin{lemma} [{\citealp[Lemma~4.3]{Khalil2002}}]
	\label{lemma:bound of a positive definite function}
	If $g : \mathcal{X} \to \mathbb{R}$ is continuous, positive definite, and radially unbounded, then there exist $\mathcal{K}_\infty$ functions $\rho_1$ and $\rho_2$ s.t. 
		$
			\rho_1(\|x\|) \leq g(x) \leq \rho_2(\|x\|)
		$ for all $x \in \mathcal{X}$.
\end{lemma}

\begin{Proof*}{Theorem~\ref{thm:policy improvement thm for optimal control} (\secref{subsection:nonlinear optimal control})}
$J_\pi$ is (i) positive definite (by Lemma~\ref{lemma:Lyapunov v}a),  (ii) $\mathrm{C}^1_\mathsf{Lip}$ (by the regularity $\mathcal{V}_\mathsf{a} \subset \mathrm{C}^1_\mathsf{Lip}$), and (iii) radially unbounded. So, by Lemma~\ref{lemma:bound of a positive definite function} above, there exist $\mathcal{K}_\infty$ functions $\rho_1$ and $\rho_2$ s.t.  
$
		\rho_1(\|x\|) \leq J_\pi (x) \leq \rho_2(\|x\|)$ for all $x \in \mathcal{X}
$.
Since $c$ is positive definite by \eqref{eq:n.d. R in optimal control},  $r_\mathsf{max} = -\min \{ c(x, u) : (x, u) \in \mathcal{X} \times \mathcal{U}\} = 0$, hence $\widebar v = 0$ by Lemma~\ref{lemma:VF range}. Therefore, we conclude $\pi' \in \Pi_\mathsf{a}$ and $J_{\pi'} \leqslant J_{\pi}$ 
by Theorem~\ref{thm:policy improvement thm in locally Lipschitz case} with $\Omega = \{ 0 \}$ (i.e., $\| \cdot \|_\Omega = \| \cdot \|$) and Lemma~\ref{lemma:pi'(0)=0} below.
\end{Proof*} 

\begin{lemma} [\secref{subsection:nonlinear optimal control}]
	\label{lemma:pi'(0)=0}
	Let $v : \mathcal{X} \to \mathbb{R}$ be $\mathrm{C}^1_\mathsf{Lip}$ and negative definite.
	If a policy~$\pi$ is given by $\pi(x) = u_*(x, \nabla v(x))$, then  $\pi \in \Pi_0$.
\end{lemma}

\begin{Proof}
	Since $v$ is $\mathrm{C}^1$ and negative definite, we have $\nabla v(0) = 0$ ($\because$ $x = 0$ is the global maximum). Then, by the $\mathrm{argmax}$-formula \eqref{eq:assumption:general condition for policy improvement} of $u_*$ and the definition \eqref{eq:def:Hamiltonian} of the Hamiltonian $h$, we have at $x = 0$:
	\[
		\pi(0) = u_*(0, \nabla v(0)) = u_*(0, 0)\in \Argmax_{u \in \mathcal{U}} \, h(0, u, 0) = \Argmax_{u \in \mathcal{U}} \, r(0, u),
	\]
	which implies $\pi(0) = 0$ since $(x, u) = (0, 0)$ is the unique global maximum of $r$ by negative definiteness of $r$ ($= -c$)---see \eqref{eq:n.d. R in optimal control}. Moreover, $\pi$ is locally Lipschitz since so are both $u_*$ and $\nabla v$. Therefore, we conclude that $\pi \in \Pi_0$. 
\end{Proof}

\begin{Proof*}{Theorem~\ref{thm:fundamental properties in optimal control} (\secref{subsection:nonlinear optimal control})}
First, $\pi_0 \in \Pi_\mathsf{a}$ by \textsf{(A)}. Suppose $\pi_{i-1} \in \Pi_\mathsf{a}$ for some $i \in \mathbb{N}$. Then, for $\gamma = 1$, $x_\mathsf{e} = 0$ is globally asymptotically stable (hence globally attractive) under $\pi_{i-1}$ by Theorem~\ref{thm:stability of optimal control}, hence $J_i = J_{\pi_{i-1}}$ by Theorem~\ref{thm:uniqueness under gas} and $v_i \in \smash{\mathrm{C}^1_\mathsf{Lip}}$ in \textsf{(B)}. In the same way, we have $J_i = J_{\pi_{i-1}}$ for $\gamma \in (0, 1)$ by Theorem~\ref{thm:uniqueness under gas} if $x_\mathsf{e} = 0$ under $\pi_{i-1}$ is globally attractive. Otherwise, the condition~\textsf{(B)} and the inequality $\kappa_i \cdot J_i \leqslant c_{\pi_{i-1}}$ in \textsf{(C)} (for IPI, with \textsf{(D)} or \textsf{(E)}) result in $J_i = J_{\pi_{i-1}}$ by Theorem~\ref{thm:uniqueness condition in optimal control}. In short, we have shown that (i) $J_i = J_{\pi_{i-1}}$ by Theorems~\ref{thm:uniqueness under gas} and \ref{thm:uniqueness condition in optimal control} for any case, and (ii) if $\gamma = 1$, then $x_\mathsf{e} = 0$ under $\pi_{i-1}$ is globally asymptotically stable. Therefore, the followings are also obvious by the radial unboundedness of $J_i$ in \textsf{(B)} and Theorems~\ref{thm:stability of optimal control} and \ref{thm:policy improvement thm for optimal control}:
\begin{enumerate}
	\item $x_\mathsf{e} = 0$ under $\pi_{i-1}$ is globally asymptotically stable \emph{if \eqref{eq:gas condition in optimal control} is true (or if $\gamma = 1$ as proven above)},
	\item $\pi_i \in \Pi_\mathsf{a}$ and $J_{\pi_i} \leqslant J_{\pi_{i-1}}$.
\end{enumerate}
Now that $\pi_i \in \Pi_\mathsf{a}$, the mathematical induction completes the proof.
\end{Proof*} 

\subsection{Proofs of Some Facts in \secref{subsection:LQR} LQRs}
\label{appendix:proofs regarding LQR}

In this appendix, for completeness, 
we provide proofs of some of the facts used in \secref{subsection:LQR}.

\begin{Proof*}{Stabilizability and Observability of $(A^\alpha, B, S)$}
Note that $(A, B)$ is stabilizable iff 
$\rank{[ \, A - \lambda I \;\,  B\, ]} = l$ $\forall \lambda \in \mathbb{C}$ such that $\mathrm{Re}{\lambda} \geq 0$ \citep[Theorem~3.2]{Zhou1998}. Since $(A^0,B)$ is stabilizable, therefore, we obtain 
\begin{equation}
	\rank{[ \, A^\alpha - \lambda I \;\,  B\, ]} = \mathrm{rank}\big([ \, A^0 - (\lambda + \alpha/2) I \;\,  B\, ]\big ) = l
	\label{eq:appendix:rank equation for stabilizability}
\end{equation}
for all $\lambda \in \mathbb{C}$ such that $\textrm{Re} \lambda \geq - \alpha/2$ with $\alpha \geq 0$. Hence, \eqref{eq:appendix:rank equation for stabilizability} holds whenever $\textrm{Re} \lambda \geq 0$, i.e., $(A^\alpha, B)$ is stabilizable. Similarly, $(S, A)$ is observable iff $\rank{[ \, A^\mathsf{T} - \lambda I \;\,  S \, ]} = l$ for all $\lambda \in \mathbb{C}$ \citep[Theorem~3.3]{Zhou1998}. Since 
\[
	\mathrm{rank}\big ( [ \, ( A^\alpha )^{\!\mathsf{T}} \! \! - \lambda I \;\,  S\, ]\big ) = \mathrm{rank}\big([ \, ( A^0 )^{\!\mathsf{T}} - \bar \lambda I \;\,  S\, ]\big ) = l
\]
for all $\bar \lambda \doteq \lambda  + \alpha / 2 \in \mathbb{C}$ and thus for all $\lambda \in \mathbb{C}$, the observability of $(S, A^\alpha)$ is now obvious by that of $(S, A^0)$. 	
\end{Proof*}

\begin{Proof*}{Existence of $P$ s.t. $P_i \to P$} For $x, y \in \mathcal{X}$,  let 
$B_i: \mathcal{X}^2 \to \mathbb{R}$ be defined for $J_i(x) = x^\mathsf{T} P_i x$ ($= -v _i(x)$) as 
\[
	B_i(x, y) \doteq  J_i(x + y) - J_i(x - y)  = 4 \, x^\mathsf{T} P_i y,
\]
and denote ${\hat J}_* \doteq - {\hat v}_*$. Then, since $J_i \to {\hat J}_*$ pointwise by Theorem~\ref{thm:convergence of PI}a, $B_i$ pointwise converges to $B$ defined as 
\[
	B (x, y) \doteq {\hat J}_*(x + y) - {\hat J_*}(x - y).
\]
Since $B_i$ is bilinear and symmetric, we have the following claim.
\begin{claim}
	\label{claim:appendix:bilinearity}
	$B$ is \textbf{\emph{a.}} bilinear and \textbf{\emph{b.}} symmetric. 
\end{claim}
By Claim~\ref{claim:appendix:bilinearity}, there exists a symmetric matrix $P$ s.t. $B(x, y) = 4 \, x^\mathsf{T} P y$. Moreover, ${\hat J}_*(0) = 0$ ($\because$ $0 = J_i(0) \to {\hat J}_*(0)$). Therefore, ${\hat J}_*$ is quadratic (and thus continuous) as shown below: 
\[
	{\hat J}_*(x) = {\hat J}_*(x) - {\hat J}_*(0) = B(x/2,x/2) = x^\mathsf{T} P x.
\]		
Next, let $\Omega \doteq \{x \in \mathcal{X} \! : \|x\| \! = \! 1\}$. Then, $\Omega$ is obviously compact, hence $J_i \to {\hat J}_*$ uniformly  on $\Omega$ by Theorem~\ref{thm:convergence of PI}b. Moreover, since ${\hat J}_* \leqslant J_i$ for every $i \in \mathbb{N}$ by Theorem~\ref{thm:fundamental properties}, every $P_i - P$ is positive semidefinite and thus represented as  $P_i - P = N_i^\mathsf{T} N_i$ for some $N_i \in \mathbb{R}^{l \times l}$ \citep[Theorem~3.7.3]{Chen1998}. Therefore, by the definition of $d_\Omega$,
\begin{align*}
	d_\Omega(J_i, {\hat J}_*) \!
	= \sup_{x \in \Omega} \big | \, x^\mathsf{T} (P_i - P)x \, \big | 
	= \sup_{\|x \| = 1}  \| N_i x \|^2
	= \MatNorm{N_i}^2 = \MatNorm{\smash{N_i^\mathsf{T} N_i}} = \MatNorm{P_i - P} \geq 0,
\end{align*}
where $\MatNorm{N_i}^2 = \MatNorm{\smash{N_i^\mathsf{T} N_i}}$ holds since $\MatNorm{\;\cdot\;}$ is induced by the Euclidean norm $\|\cdot\|$. Finally, since $d_\Omega(J_i, {\hat J}_*) \to 0$ by the uniform convergence $J_i \to {\hat J}_*$ on $\Omega$, we conclude that $\MatNorm{P_i - P} \to 0$, i.e., $P_i \to P$.

\textbf{(Proof of Claim~\ref{claim:appendix:bilinearity}).} \textbf{a. Bilinearity.} Since $B_i$ is bilinear (i.e., $B_i(x, y) = 4\, x^\mathsf{T} P_i y$),
	\[
		B_i (x_1 + x_2, y) = B_i(x_1,y) + B_i(x_2,y) \qquad  \forall x_1, x_2, y \in \mathcal{X},
	\]
	where both sides converge to $B(x_1 + x_2,y)$ and $B(x_1, y) + B(x_2, y)$, respectively. This proves that for each $y \in \mathcal{X}$, $B(\cdot, y)$ preserves the vector addition. Similarly, we can prove that $B(\alpha x, y) = \alpha B(x, y)$ for all $x, y \in \mathcal{X}$ and $\alpha \in \mathbb{R}$. Therefore, $B(\cdot, y)$ is linear and in the same way, so is $B(x, \cdot)$, meaning that $B$ is bilinear.
	
	\textbf{b. Symmetry.} Since each $P_i$ is symmetric, so is each $B_i$, hence for all $x, y \in \mathcal{X}$, $B_i(x, y) = B_i (y, x)$; by the pointwise convergence $B_i \to B$, we have $B_i(x, y) \to B(x, y)$ and $B_i(y, x) \to B(y, x)$; by the uniqueness of  the limit point, $B(x, y) = B(y, x)$, for all $x \in \mathcal{X}$. Therefore, $B$ is symmetric.
\end{Proof*}

\begin{Proof*}{Quadratic Convergence $P_i \to P_*$}
Note that the matrix formula~\eqref{eq:matrix iter in policy eval of DPI for LQR} can be rewritten for $i \in \mathbb{N} \setminus \{1\}$ as 
\begin{equation}
	(A_{i-1}^\alpha)^\mathsf{T} P_i + P_i A_{i-1}^\alpha  
	= - \mathcal{S} - \mathcal{K}_{i-1}^\mathsf{T} \Gamma \mathcal{K}_{i-1},
	\label{eq:appendix:matrix iter in policy eval of DPI for LQR} 
\end{equation}
where $\mathcal{S} \doteq S - E \Gamma^{-1} E^\mathsf{T}$ is a Schur complement of $\mathcal{W}$ and thus positive semi-definite \citep{Horn1990}; $\mathcal{K}_{i-1} \doteq \Gamma^{-1} B^\mathsf{T} P_{i-1}$. Here, by the policy improvement and definitions in \secref{subsection:LQR}, $A_{i-1}^\alpha$ in \eqref{eq:appendix:matrix iter in policy eval of DPI for LQR} can be rewritten as
\begin{equation*}
	A_{i-1}^\alpha = \mathcal{A}^\alpha - B \mathcal{K}_{i-1}
	\textrm{ for }
	\mathcal{A}^\alpha \doteq A^0 - \alpha I  / 2 - B\Gamma^{-1} E^\mathsf{T},	
\end{equation*}
where $\mathcal{A}^\alpha$ is different from $A^\alpha$ ($\doteq A^0 - \alpha I  / 2$). Therefore, one can see that \eqref{eq:appendix:matrix iter in policy eval of DPI for LQR} is exactly same as the well-known matrix-form PI \citep{Kleinman1968} for the LQR \eqref{eq:LQR case} with $A^0$, $S$, and $E$ replaced by $\mathcal{A}^\alpha$, $\mathcal{S}$ and $0$, respectively.

For the corresponding simplified LQR:
$
	\begin{cases*}
		\textrm{a linear dynamics: } {\tilde f}(x,u) = \mathcal{A}^\alpha x + Bu,	\\[2.5pt]
		\textrm{the unconstrained action space: } \mathcal{U} = \mathbb{R}^m,	\\
		\textrm{a quadratic positive cost function: }
		{\tilde c}(x,u) = x^\mathsf{T} \mathcal{S} x + u^\mathsf{T} \Gamma u,
	\end{cases*}
$\\[10pt]
each policy ${\tilde \pi}_i(x) = - \mathcal{K}_i x$ is admissible by Theorem~\ref{thm:fundamental properties}, meaning that the states under ${\tilde \pi}_i$ converge to zero as $t \to \infty$ (as discussed in \secref{subsection:LQR}). This convergence happens for the linear system iff $A^\alpha_i$ ($= \mathcal{A}^\alpha - B\mathcal{K}_i$) is Hurwitz \citep{Chen1998,Khalil2002} and thus also proves that $(\mathcal{A}^\alpha, B)$ is stabilizable. Since $(S, A^\alpha)$ is observable, so is $(\mathcal{S}, \mathcal{A}^\alpha)$ by \citet[Lemma~16.2.7]{Lancaster1995} and nondegenerate $\mathcal{W}$. Therefore, the quadratic convergence $P_i \to P_*$ is directly proven by following \citet{Kleinman1968}'s proof (when $(\mathcal{A}^\alpha, B)$ is controllable) or generally, by \citet[Theorem~5 and Remark~4 with $\hslash \to \infty$]{Lee2014}. Additionally, this approach can provide an alternative proof of Theorem~\ref{thm:LQR IPI property}. 	
\end{Proof*}

\newlength{\bibitemsep}\setlength{\bibitemsep}{.4\baselineskip plus .05\baselineskip minus .05\baselineskip}
\newlength{\bibparskip}\setlength{\bibparskip}{0pt}
\let\oldthebibliography\thebibliography
\renewcommand\thebibliography[1]{%
  \oldthebibliography{#1}%
  \setlength{\parskip}{\bibitemsep}%
  \setlength{\itemsep}{\bibparskip}%
}

{\small
\bibliographystyle{myplainnat}

}
\end{document}